\documentclass[11pt]{article}
\usepackage{amsmath, amssymb, amsthm}
\usepackage{geometry}
\usepackage[hidelinks]{hyperref}
\usepackage{graphicx}

\usepackage{algorithm}
\usepackage{algorithmic}

\usepackage{tikz}
\usetikzlibrary{positioning, shapes.geometric, calc}

\usepackage{empheq}
\usepackage{tikz}

\usepackage{enumitem}

\newtheorem{theorem}{Theorem}
\newtheorem{corollary}{Corollary}

\newcommand{\dashedbox}[1]{%
  \tikz[baseline=(X.base)]%
    \node[draw,dashed,inner sep=6pt] (X) {$\displaystyle #1$};%
}

\usepackage{fancyhdr}
\title{Gaussian Joint Embeddings\\[0.2em] For Self-Supervised Representation Learning}
\author{Yongchao Huang\footnote{[Email: yongchao.huang@abdn.ac.uk] The author welcomes any follow-up work, extensions, and adaptations of these ideas. If this manuscript found useful in future research, appropriate citation would be appreciated. It was developed over many days and nights with the aim of providing a self-contained material for open knowledge sharing, although some (many) errors may still remain after careful review.}}
\date{02/03/2026}

\begin{document}

\maketitle

\begin{abstract}
Self-supervised representation learning often relies on deterministic predictive architectures to align context and target views in latent space. While effective in many settings, such methods are limited in genuinely multi-modal inverse problems, where squared-loss prediction collapses towards conditional averages, and they frequently depend on architectural asymmetries to prevent representation collapse. In this work, we propose a probabilistic alternative based on generative joint modeling. We introduce Gaussian Joint Embeddings (GJE) and its multi-modal extension, Gaussian Mixture Joint Embeddings (GMJE), which model the \textit{joint density} of context and target representations and replace black-box prediction with closed-form conditional inference under an explicit probabilistic model. This yields principled uncertainty estimates and a covariance-aware objective for controlling latent geometry. We further identify a failure mode of naive empirical batch optimization, which we term the \textit{Mahalanobis Trace Trap}, and develop several remedies spanning parametric, adaptive, and non-parametric settings, including prototype-based GMJE, conditional Mixture Density Networks (GMJE-MDN), topology-adaptive Growing Neural Gas (GMJE-GNG), and a Sequential Monte Carlo (SMC) memory bank. In addition, we show that standard contrastive learning can be interpreted as a degenerate non-parametric limiting case of the GMJE framework. Experiments on synthetic multi-modal alignment tasks and vision benchmarks show that GMJE recovers complex conditional structure, learns competitive discriminative representations, and defines latent densities that are better suited to unconditional sampling than deterministic or unimodal baselines.
\end{abstract}

\tableofcontents
\newpage

\section{Introduction} \label{sec:intro}

Recent progress in Self-Supervised Learning (SSL) for high-dimensional data has been strongly influenced by joint-embedding architectures, which have emerged as a major alternative to reconstruction-based objectives \cite{assran2023ijepa,vanassel2025jointembedding}. The central aim of these frameworks is to encode structurally complex inputs, such as augmented or masked views of an image, into a latent representation space in which semantically meaningful dependencies can be captured and predicted \cite{assran2023ijepa,lecun2022path}. Within this joint-embedding paradigm, modern methods are commonly divided into two broad families: \textit{contrastive approaches}, which learn by attracting positive pairs and repelling negative pairs, and \textit{non-contrastive} or \textit{predictive approaches}, which directly align representations across views without explicit negative samples \cite{Chen2020simCLR,grill2020byol,chen2021simsiam,bardes2021vicreg,balestriero2022contrastive}.

In the contrastive paradigm (e.g. InfoNCE \cite{oord2019InfoNCE}, SimCLR \cite{Chen2020simCLR}, MoCo \cite{He2020MoCo}), networks are trained to increase the similarity between paired augmented views while decreasing similarity to non-matching instances \cite{oord2019InfoNCE,Chen2020simCLR,He2020MoCo}. While highly effective at instance-level discrimination, these methods primarily learn augmentation-invariant representations and are commonly understood as balancing two geometric effects: alignment of positive pairs and uniformity of the representation distribution \cite{wang2020understanding}. Because they typically rely on $\ell_2$-normalized embeddings and similarity-based losses, contrastive methods are often analysed on the unit hypersphere, rather than through an explicit model of cross-view covariance structure \cite{wang2020understanding}. In practice, they can also be computationally demanding, since strong performance often depends on access to many negative samples, achieved either through very large batch sizes as in SimCLR or momentum-updated memory queues as in MoCo \cite{Chen2020simCLR,He2020MoCo,Chen2022LargeBatchCL}.

To reduce the computational and design burdens associated with negative sampling, recent SSL research has increasingly explored non-contrastive and predictive architectures, including BYOL \cite{grill2020byol} and JEPA-style methods \cite{assran2023ijepa}. Rather than relying on explicit negative pairs, these approaches learn by predicting one representation from another in latent space \cite{grill2020byol,assran2023ijepa}. However, deterministic prediction can be limiting when the conditional target distribution is genuinely multimodal. In partial-observation or masking settings, a single context representation may be compatible with multiple plausible target representations, especially when the context does not uniquely determine the target \cite{bishop1994mixture}. Under squared-error risk, the Bayes-optimal deterministic predictor is the conditional mean, and in multimodal problems this averaging effect can place predictions between modes, potentially in regions that do not correspond to typical data samples \cite{bishop1994mixture,bishop2006prml}. Moreover, unlike contrastive objectives, predictive regression losses do not explicitly enforce a global dispersion or uniformity constraint on the representation distribution, so additional mechanisms are often introduced to prevent degenerate solutions during training, including predictor asymmetry, stop-gradient operations, EMA target networks, or explicit variance/covariance regularization \cite{grill2020byol,tian2021understanding,bardes2021vicreg,halvagal2023implicit}.

In this work, we propose a shift from deterministic latent prediction towards \textit{generative joint modeling}. We introduce \textbf{Gaussian Joint Embeddings (GJE)}, a probabilistically grounded framework that models the concatenated context and target representations through a joint density, $p(z_c, z_t)$. Minimizing the full joint negative log-likelihood (NLL) yields the standard decomposition
$
-\log p(z_c, z_t) = -\log p(z_t \mid z_c) - \log p(z_c),
$
which separates a conditional matching term from a marginal density term \cite{bishop2006prml}. Under a Gaussian parameterization, the conditional component encourages agreement between matched context-target pairs, while the marginal component introduces covariance-dependent \textit{volume regularization} through the log-determinant term, promoting a non-degenerate latent geometry. In this sense, GJE is designed to mitigate both \textit{instance collapse} and \textit{dimensional collapse} by explicitly modelling variance and covariance structure, thereby reducing reliance on architectural asymmetries such as EMA target networks or stop-gradient mechanisms that are commonly used in non-contrastive SSL \cite{grill2020byol,bardes2021vicreg,tian2021understanding}.

Because a single Gaussian density is often too restrictive to represent multi-modal or highly non-convex latent structure, we further extend this GJE formulation and introduce \textbf{Gaussian Mixture Joint Embeddings (GMJE)}. By leveraging the expressive power and universal approximation properties of Gaussian mixture models \cite{li_mixture_1999,Huang2025GMA}, GMJE can represent highly irregular and multi-modal semantic structure more flexibly than a single-Gaussian joint model \cite{goodfellow2016deep,lindberg2021sparsegmm}. Importantly, we identify a failure mode that can arise when naively implementing probabilistic joint-embedding objectives, which we term the \textit{Mahalanobis Trace Trap}. Specifically, we show that when joint likelihoods are computed using empirical batch covariances, the resulting Mahalanobis coupling term can degenerate so that the intended cross-view attractive signal becomes constant, thereby weakening cross-view learning.
To address this failure mode, we introduce three GMJE architectural remedies spanning both parametric and non-parametric settings. Ultimately, we show that GMJE can learn strong discriminative representations while simultaneously defining a continuous probabilistic density over the latent space, which can be used for subsequent discriminative or generative tasks.

In summary, our core contributions are as follows:
\begin{itemize}[label=-]
    \item \textbf{A new probabilistic paradigm:} we formalize the GJE framework and show that replacing conditional MSE prediction with joint generative modeling yields closed-form conditional predictions, uncertainty estimates, and an objective with built-in variance-covariance control that can mitigate representation collapse while reducing reliance on architectural asymmetries.
    
    \item \textbf{Multi-modal extension via GMJE:} we introduce Gaussian Mixture Joint Embeddings (GMJE) to address the unimodal limitations of single-Gaussian GJE. To optimize this richer joint distribution while avoiding the \textit{Mahalanobis Trace Trap}, we develop several GMJE variants spanning parametric, adaptive, and non-parametric settings, including learnable global prototypes, conditionally dynamic \textit{Mixture Density Networks} (GMJE-MDN), topology-adaptive \textit{Growing Neural Gas} (GMJE-GNG), and the non-parametric \textit{SMC-GMJE} formulation.
    
    \item \textbf{The InfoNCE bridge and SMC memory banks:} we establish a theoretical connection showing that standard contrastive learning can be interpreted as a degenerate, non-parametric limiting case of our GMJE framework. Building on this view, we introduce \textit{Sequential Monte Carlo} (SMC) particle filtering, which replaces unweighted contrastive FIFO queues, as an alternative memory-bank mechanism, yielding dynamically weighted samples that prioritize more informative samples.
    
    \item \textbf{Generative latent densities:} we empirically show that GMJE learns a continuous latent density that is more amenable to generative sampling than standard predictive or contrastive latent spaces. In our latent-sampling experiments, this supports unconditional image synthesis through direct sampling from the learned latent distribution.
\end{itemize}

\section{Background: From Deterministic Prediction to Generative Joint Modeling} \label{sec:background}

In Joint Embedding Predictive Architectures (JEPA), the learning objective is to predict a target representation from a context representation in latent space, typically using a parameterized predictor network. In image-based JEPA, this is implemented by predicting the representations of masked or spatially separated target blocks from an informative context block within the same image \cite{assran2023ijepa}.
Despite its empirical success, classic JEPA still exhibits several conceptual and architectural limitations:
\begin{itemize}[label=-]
    \item \textit{Deterministic without Uncertainty:} classic JEPA is typically trained with a deterministic regression-style objective (e.g. MSE) in latent space. Under squared-error risk, the Bayes-optimal predictor is the conditional mean $\mathbb{E}[z_t \mid z_c]$; hence, in genuinely multi-modal settings, deterministic prediction can average across modes and produce representations that do not correspond to typical targets. Because no conditional distribution is modeled explicitly, this formulation does not natively provide predictive uncertainty estimates \cite{bishop1994mixture,bishop2006prml}.
    
    \item \textit{Risk of Representation Collapse:} standard JEPA does not explicitly optimize a joint density over $(z_c,z_t)$ and therefore does not include a separate marginal density term such as $-\log p(z_c)$. As a result, unlike objectives that explicitly control variance or covariance structure, it does not directly impose a global dispersion constraint on the representation distribution, which can make collapse avoidance more delicate in practice \cite{bardes2021vicreg,tian2021understanding,halvagal2023implicit}.
    
    \item \textit{Training Stability Mechanisms:} as in other non-contrastive joint-embedding methods, stable training often relies on architectural asymmetries or auxiliary regularization mechanisms. In related predictive SSL frameworks, these include stop-gradient operations, predictor asymmetry, and slow-moving target networks updated by exponential moving averages (EMA) \cite{grill2020byol,chen2021simsiam,tian2021understanding}.
    
    \item \textit{Directional Constraints:} standard JEPA is typically formulated as a uni-directional prediction problem ($z_c \to z_t$). Recent work has explored bi-directional extensions, such as BiJEPA \cite{huang2026BiJEPA}, to improve symmetry and generalization. However, it suggests that symmetric prediction can amplify optimization instability and may require explicit norm control, such as $L^2$-normalization, to prevent representation explosion \cite{huang2026BiJEPA}.
\end{itemize}

A comparison of the classic JEPA and our proposed GJE framework is shown in Fig.~\ref{fig:comparison_classicJEPA_GJE}. GJE replaces deterministic latent prediction with generative joint modeling: instead of learning only a predictor $g(z_c)$, we model the concatenated representations $[z_c, z_t]^T$ through a joint distribution. Under this view, prediction is derived from the learned dependency structure rather than solely from a black-box predictor. By optimizing a full joint objective, GJE explicitly retains variance-covariance structure that is absent from purely deterministic prediction losses, and is therefore designed to promote a more diverse and non-degenerate embedding geometry.

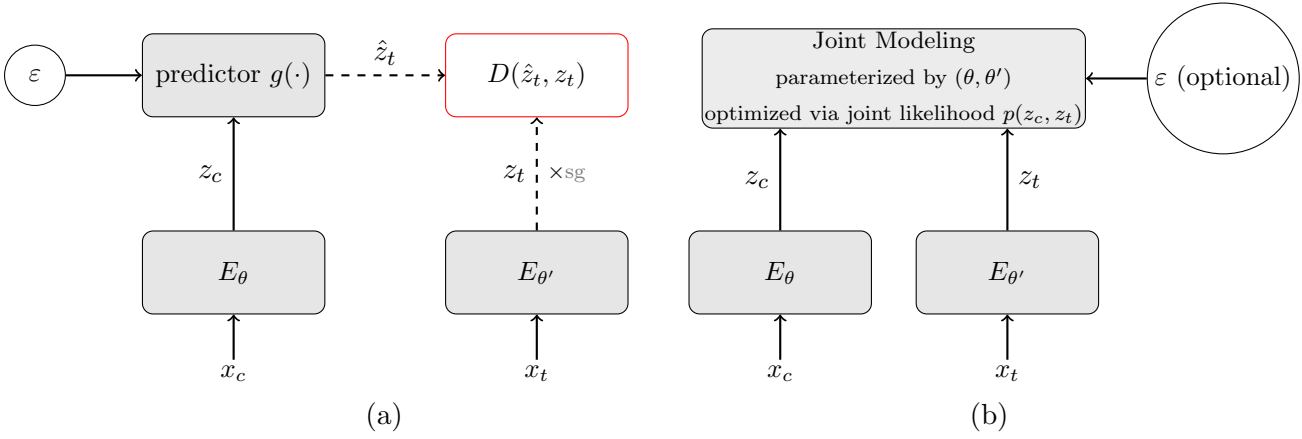
\begin{figure}[t]
\centering
\begin{minipage}[b]{0.45\textwidth}
\centering
\begin{tikzpicture}[
    node distance=1.5cm and 1cm,
    input_var/.style={inner sep=2pt},
    varnode/.style={draw, circle, inner sep=2pt, minimum size=0.8cm},
    component/.style={draw, rounded corners, fill=gray!20, minimum width=2.4cm, minimum height=1.1cm},
    distance_box/.style={draw=red, rounded corners, fill=white, minimum width=2.4cm, minimum height=1.1cm, inner sep=3pt}
]
    % Common bounding box so rows align across the two minipages
    \path[use as bounding box] (-1.6,-0.9) rectangle (6.3,4.5);

    % JEPA Nodes
    \node[input_var] (xc) {\small $x_c$};
    \node[input_var, right=3.5cm of xc] (xt) {\small $x_t$};
    \node[component, above=0.6cm of xc] (xc-enc) {\small $E_\theta$};
    \node[component, above=0.6cm of xt] (xt-enc) {\small $E_{\theta'}$};
    
    \node[component, above=of xc-enc] (pred) {\small predictor $g(\cdot)$};
    \node[varnode, left=1cm of pred] (eps) {\small $\varepsilon$};

    \node[distance_box, above=of xt-enc] (dist) {\small $D(\hat{z}_t, z_t)$};

    % Arrows 
    \draw[->, thick] (xc) -- (xc-enc);
    \draw[->, thick] (xc-enc) -- node[left] {$z_c$} (pred);
    \draw[->, thick] (eps) -- (pred);
    
    \draw[->, thick] (xt) -- (xt-enc);
    \draw[->, thick, dashed] (xt-enc) -- node[left] {$z_t$} node[right] {\scriptsize $\times$\textcolor{gray}{sg}} (dist);
    
    \draw[->, thick, dashed] (pred) -- node[above] {$\hat{z}_t$} (dist);
\end{tikzpicture}

\vspace{-0.7cm}
\centerline{(a)}
\label{fig:jepa}
\end{minipage}
\hfill
\begin{minipage}[b]{0.45\textwidth}
\centering
\begin{tikzpicture}[
    node distance=1.5cm and 1cm,
    input_var/.style={inner sep=2pt},
    varnode/.style={draw, circle, inner sep=2pt, minimum size=0.8cm},
    component/.style={draw, rounded corners, fill=gray!20, minimum width=2.4cm, minimum height=1.1cm},
    joint_modeling_box/.style={
        draw,
        rounded corners,
        fill=gray!15,
        minimum width=4.1cm,
        minimum height=1.1cm,
        inner sep=1pt,
        align=center
    }
]
    % Common bounding box so rows align across the two minipages
    \path[use as bounding box] (-0.8,-0.9) rectangle (6.3,4.5);

    % GJE Nodes
    \node[input_var] (xc) {\small $x_c$};
    \node[input_var, right=2.5cm of xc] (xt) {\small $x_t$};

    \node[component, above=0.6cm of xc] (Etheta) {\small $E_\theta$}; 
    \node[component, above=0.6cm of xt] (Etheta_prime) {\small $E_{\theta'}$}; 

    \node[joint_modeling_box, above=1.35cm of $(Etheta.north)!0.5!(Etheta_prime.north)$, anchor=south] (joint_model) {
        {\footnotesize Joint Modeling} \\
        {\scriptsize parameterized by $(\theta,\theta')$} \\
        {\scriptsize optimized via joint likelihood $p(z_c,z_t)$}
    };

    \node[varnode, right=0.8cm of joint_model] (eps_gje) {\small $\varepsilon\ \text{(optional)}$};
    \draw[->, thick] (eps_gje) -- (joint_model);

    % Arrows for flows into the joint modeling block 
    \draw[->, thick] (xc) -- (Etheta);
    \draw[->, thick] (Etheta.north) -- node[left] {$z_c$} (joint_model.south -| Etheta.north);

    \draw[->, thick] (xt) -- (Etheta_prime);
    \draw[->, thick] (Etheta_prime.north) -- node[right] {$z_t$} (joint_model.south -| Etheta_prime.north);
\end{tikzpicture}

\vspace{-0.7cm}
\centerline{(b)}
\label{fig:gje}
\end{minipage}
\hfill
\caption{A comparison of (a) the classic JEPA framework, based on separate encoding and alignment via deterministic latent prediction, and (b) our proposed GJE framework, which models the joint distribution of context and target representations and derives predictions probabilistically. The component with a red outline in (a) represents the deterministic distance loss specific to classic JEPA. `sg' denotes \textit{stop gradient}, and $\varepsilon$ represents injected side information (e.g. physics, action conditioning or noise).}
\label{fig:comparison_classicJEPA_GJE}
\end{figure}

\paragraph{Notation and Terminology.} Throughout this work, we use the following mathematical notations.
\begin{itemize}[label=-, leftmargin=*]
    \item \textit{Data and Embeddings:} lowercase letters ($x_c, x_t$) denote raw individual context and target views (e.g. augmented images). $z_c, z_t \in \mathbb{R}^d$ denote their respective $d$-dimensional latent representations\footnote{Note: in the brief literature review of Mixture Density Networks in Section.\ref{subsubsec:MDN}, we temporarily preserve Bishop's original use of $c$ to denote dimensionality to maintain historical fidelity.}, and $Z = [z_c^T, z_t^T]^T$ denotes the concatenated joint embedding. Uppercase letters ($Z_c, Z_t$) represent matrices containing finite batches of these embeddings.
    \item \textit{Probabilistic Parameters:} $\Sigma$ denotes the general, theoretical population covariance matrix of a multivariate distribution, while $\mu$ denotes the mean vector. For mixture models, $\pi_k$ represents the mixture weight for component $k$, and $\gamma_k(\cdot)$ represents the data-dependent posterior probability (soft routing responsibility).
    \item \textit{Structural Capacities ($N, M, K$):} to strictly distinguish between optimization batches, global memory, and model capacity, we define:
    \begin{itemize}[label=$\circ$]
        \item $N$: the \textit{batch size}, i.e. the number of current representations processed in a single forward pass.
        \item $M$: the \textit{memory bank size}, i.e. the massive queue of cached target representations, used in contrastive learning and SMC.
        \item $K$: the \textit{number of components/prototypes} in a Gaussian Mixture Model (GMM).
    \end{itemize}
    \item \textit{The Dual Meaning of $K$ and $\Sigma$:} in Section.\ref{sec:GJE}, when discussing dual-GJE, $K$ is exceptionally used to denote the empirical \textit{block Gram matrix} computed over a batch via a positive-definite kernel function (e.g. $K_{cc}, K_{tt}$). Thus, in Section.\ref{sec:GJE}, $\Sigma$ and $K$ represent the same structural concept, distinguishing between pure probabilistic theory ($\Sigma$) and empirical kernel-based computation ($K$). Everywhere else in the manuscript (Section.\ref{sec:GMJE} onwards), scalar $K$ strictly refers to the number of mixture components.
\end{itemize}

\section{Gaussian Joint Embeddings (GJE)} \label{sec:GJE}

To model the alignment between a context view $x_c$ and a target view $x_t$, we introduce the \textit{Gaussian Joint Embedding} (GJE) framework. As shown in Fig.\ref{fig:comparison_classicJEPA_GJE}, the input views are passed through a context encoder $E_\theta$ and a target encoder $E_{\theta'}$ to produce latent embeddings $z_c = E_\theta(x_c) \in \mathbb{R}^{d_c}$ and $z_t = E_{\theta'}(x_t) \in \mathbb{R}^{d_t}$.

Although the general GJE framework can naturally incorporate optional side information $\varepsilon$ (such as physical constraints, action variables in reinforcement learning, or noise parameters) by conditioning the joint distribution, in the following we focus on unconditional representation alignment. Let $z = [z_c^T, z_t^T]^T \in \mathbb{R}^d$, where $d = d_c + d_t$, denote the concatenated joint embedding pair. We assume these representations are drawn from an underlying joint probability distribution $p(z; \theta, \theta')$ parameterized by the neural encoders. The self-supervised objective is then to optimize the encoder parameters by maximizing the expected joint log-likelihood over matching pairs $(x_c, x_t)$ drawn from the dataset $\mathcal{D}$:
\begin{equation} \label{eq:general_gje_objective}
    \max_{\theta, \theta'} \mathbb{E}_{(x_c, x_t) \sim \mathcal{D}} \left[ \log p(E_\theta(x_c), E_{\theta'}(x_t)) \right]
\end{equation}

By the chain rule of probability, the joint objective factorizes as
\[
\log p(z_c, z_t) = \log p(z_t \mid z_c) + \log p(z_c),
\]
which may be interpreted as a predictive conditional term together with a marginal regularization term. To make this optimization computationally tractable and geometrically meaningful, we impose a Gaussian assumption on the embedding space. Depending on how this Gaussian structure is instantiated, the framework can be developed from two complementary perspectives: the \textit{dual (sample) view}, which focuses on the predictive conditional alignment $p(z_t \mid z_c)$ over a batch of $N$ points, and the \textit{primal (feature) view}, which models the full joint density $p(z_c, z_t)$ natively over the $d$ latent dimensions.

% ==========================================
% DUAL SPACE
% ==========================================
\subsection{The Dual Formulation: Sample-Space Gaussian Process Joint Embeddings (GPJE)} \label{subsec:dual_GJE}

In classic JEPA, a deterministic neural predictor maps the context embedding to the target embedding: $\hat{z}_t = g(z_c)$. To generalize this probabilistically, we first formulate the problem in the \textit{dual sample space} by replacing this parametric MLP with a non-parametric probabilistic regression model, which treats the target embedding $z_t$ as a realization of an underlying latent function $f$ evaluated at the context embedding $z_c$:
\begin{equation}
    z_t = f(z_c) + \epsilon, \quad \epsilon \sim \mathcal{N}(\textbf{0}, \sigma^2 I)
\end{equation}
Here, $f \sim \mathcal{GP}(\textbf{0}, k(z_c, z_c'))$ is a Gaussian Process defined by a positive-definite kernel $k$, which dictates that $\text{cov}(f(z_c), f(z_c')) = k(z_c, z_c')$. Rather than modeling the covariance between the $d$ feature channels, this sample-space GP models the similarities between the \textit{individual} data points.

\paragraph{1. Training (The Joint Likelihood).}
During training, we extract a batch of $N$ context and target embedding pairs, $Z_c \in \mathbb{R}^{N \times d_c}$ and $Z_t \in \mathbb{R}^{N \times d_t}$. We compute the Gram matrix $K_{cc} \in \mathbb{R}^{N \times N}$ over the context embeddings, where each element $K_{cc}[i,j] = k_\phi(z_{c,i}, z_{c,j})$, capturing the non-linear topological distances between every sample in the batch.

Treating the target embeddings $Z_t \in \mathbb{R}^{N \times d_t}$ as the observed function values, we derive the training objective by assuming the $d_t$ latent feature channels are independent Gaussian Processes sharing the same context kernel $K_{cc}$. For a single feature channel (represented as a column vector $z_{t,j} \in \mathbb{R}^{N \times 1}$), the standard GP negative marginal log-likelihood, assuming negligible observation noise and dropping constant scaling terms, is\footnote{See Eq.(2.29) and Eq.(2.30) in \cite{rasmussen2006gaussian}. We dropped the constant term.} $\frac{1}{2} z_{t,j}^T K_{cc}^{-1} z_{t,j} + \frac{1}{2} \log |K_{cc}|$. Summing this independent objective across all $d_t$ target dimensions yields the total joint negative log-likelihood. In linear algebra, the summation of these data-fit quadratic forms across all column vectors mathematically condenses into a matrix trace operator ($\sum_{j=1}^{d_t} z_{t,j}^T K_{cc}^{-1} z_{t,j} = \text{Tr}(Z_t^T K_{cc}^{-1} Z_t)$). Meanwhile, the data-independent log-determinant penalty simply accumulates $d_t$ times. Thus, the final Dual-GJE objective used to optimize the encoder weights is:
\begin{equation} \label{eq:dual_gje_loss1}
    \mathcal{L}_{\text{Dual-GJE}}(\theta, \theta', \phi) = \underbrace{\frac{1}{2} \text{Tr}\left(Z_t^T K_{cc}^{-1} Z_t\right)}_{\text{Data-Fit}} + \underbrace{\frac{d_t}{2} \log |K_{cc}|}_{\text{Complexity Penalty}}
\end{equation}
This Dual-GJE optimisation objective is a \textit{conditional} NLL\footnote{See Eq.(2.29) or Eq.(2.30) in \cite{rasmussen2006gaussian}.}, $p(Z_t \mid Z_c)$.
As this GP marginal likelihood is inherently asymmetric - it penalizes the volume of the context space ($\log |K_{cc}|$) but applies no such expansive regularizer to the target space, optimizing Eq.\ref{eq:dual_gje_loss1} symmetrically would result in catastrophic representation collapse. The target encoder would trivially map $Z_t \to \mathbf{0}$ to minimize the quadratic trace penalty, subsequently causing the log-determinant to collapse the context space\footnote{A detailed explanation of this learning collapse can be found in Appendix.\ref{app:dual_collapse_proof}.}. 

To effectively optimize this dual objective, the target encoder $E_{\theta'}$ must be maintained as an Exponential Moving Average (EMA) \cite{grill2020byol,chen2021simsiam,assran2023ijepa,huang2026vjepa} of the context encoder\footnote{In standard Joint Embedding Architectures, the target encoder weights $\theta'$ are typically updated as an Exponential Moving Average (EMA) of the context encoder weights $\theta$, such that $\theta' \leftarrow \tau \theta' + (1 - \tau) \theta$ where $\tau \in [0, 1)$ is a momentum parameter. This asymmetric weight update, combined with a \textit{stop-gradient} operation on the target branch, is a crucial architectural heuristic in models such as BYOL \cite{grill2020byol} and standard JEPA \cite{assran2023ijepa} to prevent representation collapse - a failure mode where the network maps all inputs to a trivial constant vector to perfectly minimize prediction error.}, with a strict stop-gradient applied to $Z_t$. By anchoring the targets, $Z_t$ acts as a diverse, fixed topological reference. Only under this asymmetric EMA framework does the trace term correctly function as an expansionary force, compelling the context embeddings to spread out and strictly match the pairwise diversity of the targets, while the log-determinant term acts as an Occam's razor, regularizing the context manifold to remain smooth.

If explicit homoscedastic observation noise $\sigma^2 I$ is modeled, the objective Eq.\ref{eq:dual_gje_loss1} expands to:
\begin{equation} \label{eq:dual_gje_loss2}
    \mathcal{L}_{\text{JEGP}}(\theta, \theta') = \frac{1}{2} \text{Tr}\left( (K_{cc} + \sigma^2 I)^{-1} Z_t Z_t^T \right) + \frac{d_t}{2} \log |K_{cc} + \sigma^2 I|
\end{equation}
By the cyclic property of the trace operator ($\text{Tr}(AB) = \text{Tr}(BA)$), the data-fit terms in Eq.\ref{eq:dual_gje_loss1} and Eq.\ref{eq:dual_gje_loss2} are mathematically equivalent. Both formulations drive the exact same geometric forces, optimizing the encoders symmetrically. 

The Dual-GJE algorithm is presented in Algo.\ref{alg:exact_dual_gje} in Appendix.\ref{app:training_dual_gje}; a scalable Dual-GJE via Random Fourier Features (RFF) is presented in the same appendix.

\paragraph{2. Inference (The Conditional Prediction).}
Once \textit{the encoders and kernel are trained}, we freeze the training memory bank $Z_t \in \mathbb{R}^{N \times d_t}$ and the training Gram matrix $K_{cc} \in \mathbb{R}^{N \times N}$. 

Given a brand new, unobserved test context image $x_c^*$, we map it to a single embedding $z_c^* \in \mathbb{R}^{1 \times d_c}$. We compute its cross-similarity vector against all $N$ training samples, $k_{*c} \in \mathbb{R}^{1 \times N}$, and its self-similarity scalar $k_{**} \in \mathbb{R}^{1 \times 1}$. Using the Schur complement to condition the joint GP distribution, the predictive distribution for the target $z_t^* \in \mathbb{R}^{1 \times d_t}$ is exactly Gaussian, $\mathcal{N}(\mu_{t|c}, \Sigma_{t|c})$, defined by\footnote{See Eq.(2.25) and Eq.(2.26) in \cite{rasmussen2006gaussian}. If explicit homoscedastic observation noise is modeled (as in Eq.\ref{eq:dual_gje_loss2}), the predictive distribution seamlessly adapts by replacing the inverse Gram matrix $K_{cc}^{-1}$ in both equations with the noise-regularized term $(K_{cc} + \sigma^2 I)^{-1}$.}:
\begin{equation} \label{eq:dual_GJE_mean_and_cov}
\begin{aligned} 
    \mu_{t|c} &= k_{*c} K_{cc}^{-1} Z_t \\
    \Sigma_{t|c} &= k_{**} - k_{*c} K_{cc}^{-1} k_{c*}
\end{aligned}
\end{equation}
Notice the dimensions of Eq.\ref{eq:dual_GJE_mean_and_cov}: $[1 \times N] \times [N \times N] \times [N \times d_t] = [1 \times d_t]$. The prediction is a non-parametric, weighted average over (i.e. linear combination of) the entire training memory bank $Z_t$. \textit{(Note: Because exact inversion requires $\mathcal{O}(N^3)$ compute, GJE supports a highly scalable Random Fourier Feature (RFF) approximation, detailed in Appendix.\ref{app:rff_theory}).}

% ==========================================
% PRIMAL SPACE
% ==========================================
\subsection{The Primal Formulation: Feature-Space Covariance} \label{subsec:primal_GJE}

While the dual formulation evaluates sample similarity, the primal formulation evaluates the geometric shape of the data manifold itself by measuring how the $d$ latent features co-vary. 

\paragraph{1. Training and the Mahalanobis Trace Trap.}
We concatenate the context and target batches into a single joint batch $Z = [Z_c, Z_t] \in \mathbb{R}^{N \times d}$. We assume a joint, zero-mean Multivariate Gaussian generative model $p(z) = \mathcal{N}(\textbf{0}, C_{joint})$, i.e. each concatenated embedding vector $z_i \in \mathbb{R}^d$ is drawn from this Gaussian distribution:
\begin{equation} \label{eq:primal_GJE_joint_Gaussian_pdf}
p(z_i) \sim  \mathcal{N}(z|\textbf{0}, C_{joint}) = \frac{1}{\sqrt{(2\pi)^d |C_{joint}|}} \exp\left(-\frac{1}{2} z^T C_{joint}^{-1} z \right)
\end{equation}
In the context of deep representation learning, this strict zero-mean assumption serves a critical regularizing function: it anchors the latent space at the origin and forces the encoders to delegate all structural learning entirely to the covariance matrix. This prevents the network from trivially minimizing the loss via arbitrary translational mean shifts.

Consequently, the feature covariance matrix $C_{joint}$, with dimension $d \times d$, can be empirically\footnote{For simplicity, we use the same notation $C_{joint}$ for the population and empirically estimated covariance matrix.} computed directly from the origin for the $N$ training data points\footnote{By treating the data matrix $Z \in \mathbb{R}^{N \times d}$ as a stack of $N$ row vectors $Z = [z_1^T, z_2^T, \dots, z_N^T]^T$, its transpose $Z^T \in \mathbb{R}^{d \times N}$ becomes a matrix of side-by-side column vectors $Z^T = [z_1, z_2, \dots, z_N]$. Using block matrix multiplication, multiplying the row of blocks by the column of blocks yields the sum of their outer products: $Z^T Z = [z_1, \dots, z_N] [z_1^T, \dots, z_N^T]^T = \sum_{i=1}^N z_i z_i^T$. Taking the average with $N$ or $N-1$ is trivial.}:
\begin{equation} \label{eq:primal_GJE_empirical_cov}
    C_{joint} = \frac{1}{N} Z^T Z \quad \in \mathbb{R}^{d \times d}
\end{equation}
and it can be partitioned into four blocks:
\begin{equation}
    C_{joint} = \begin{bmatrix} C_{cc} & C_{ct} \\ C_{tc} & C_{tt} \end{bmatrix}
\end{equation}
where $C_{cc} \in \mathbb{R}^{d_c \times d_c}$ and $C_{tt} \in \mathbb{R}^{d_t \times d_t}$ are the auto-covariances, and $C_{ct} \in \mathbb{R}^{d_c \times d_t}$ is the cross-covariance capturing mutual information. Inverting $C_{joint}$ can be done using block matrix inversion and the Schur complement (see Eq.\ref{eq:primal_GJE_cov_inverse} in Appendix.\ref{app:primal_GJE_joint_vs_conditional_objectives}).

Assuming the $N$ samples in the mini-batch are independent and identically distributed, the likelihood of the entire empirical batch is the product of the individual densities, $p(Z) = \prod_{i=1}^N p(z_i)$. To optimize the encoder parameters, we formulate the empirical average of the NLL, defined as $-\frac{1}{N} \log p(Z)$:
\begin{equation*}
\begin{aligned}
-\frac{1}{N} \log \prod_{i=1}^N p(z_i) &= -\frac{1}{N} \sum_{i=1}^N \log p(z_i) \
&= \frac{1}{N} \sum_{i=1}^N \left( \frac{1}{2} z_i^T C_{joint}^{-1} z_i + \frac{1}{2} \log |C_{joint}| + \frac{d}{2} \log(2\pi) \right)
\end{aligned}
\end{equation*}
Dropping the constant $\frac{d}{2} \log(2\pi)$ term, which has no effect on the gradient with respect to the network weights, the strict generative objective minimizing the exact NLL over this $d \times d$ space becomes:
\begin{equation} \label{eq:primal_gje_loss}
    \mathcal{L}_{\text{Primal-GJE}}(\theta, \theta') = \frac{1}{2N} \sum_{i=1}^N z_i^T C_{joint}^{-1} z_i + \frac{1}{2} \log |C_{joint}|
\end{equation}
From an information-theoretic perspective, minimizing this joint objective natively \textit{maximizes the Mutual Information} between the context and target spaces while regularizing their differential entropy. Further, factoring this joint likelihood shows why conditional-only architectures like classic JEPA require asymmetric heuristics to survive (for formal proofs, see Appendix.\ref{app:primal_GJE_joint_vs_conditional_objectives}).

It seems from the primal objective Eq.\ref{eq:primal_gje_loss} that, minimizing the volume penalty $\log |C_{joint}|$ natively enforces compression, while the inverted Mahalanobis distance acts as a repulsive force; however, directly optimizing this exact NLL in the empirical primal space reveals a fatal optimization trap. Because the scalar quadratic form is equal to its own trace, we can apply the cyclic property ($\text{Tr}(AB) = \text{Tr}(BA)$) to the data-fit term:
\begin{equation*}
    \frac{1}{2N} \sum_{i=1}^N \text{Tr}(z_i^T C_{joint}^{-1} z_i) = \frac{1}{2} \text{Tr}\left( C_{joint}^{-1} \left[ \frac{1}{N} \sum_{i=1}^N z_i z_i^T \right] \right) = \frac{1}{2} \text{Tr}(C_{joint}^{-1} C_{joint}) = \frac{d}{2}
\end{equation*}
The inverted Mahalanobis distance algebraically cancels against the empirical covariance, collapsing the entire data-fit term into a dead constant scalar $\frac{d}{2}$. With a constant data-fit term, the expansive gradient is exactly zero. The optimizer is left solely with the volume penalty ($+\frac{1}{2} \log |C_{joint}|$), which it trivially minimizes by shrinking all representations to the origin, causing catastrophic dimensional collapse (details about The “\textit{Mahalanobis Trace Trap}” see Appendix.\ref{app:mahalanobis_trap}).

Therefore, to maintain stable representations in the primal space, we cannot rely on the native Mahalanobis distance for repulsion. We must structurally invert the regularizer, i.e. \textit{maximizing} the differential entropy $\frac{1}{2} \log |C_{joint}|$ to force the features to span the available volume \cite{bardes2021vicreg}, or alternatively, optimize the cross-covariance directly via the Hilbert-Schmidt Independence Criterion (see Appendix.\ref{app:hsic_theory}).

The Primal-GJE algorithm is presented in Algo.\ref{alg:primal_gje} in Appendix.\ref{app:training_primal_gje}.

\paragraph{2. Inference (The Closed-Form Linear Predictor).}
After training, we freeze the four $d \times d$ covariance blocks. Given a new test context vector $z_c^* \in \mathbb{R}^{d_c \times 1}$ (formatted as a column vector), the conditional distribution $p(z_t^* \mid z_c^*) = \mathcal{N}(z_t^*|\mu_{t|c}(z_c^*),\Sigma_{t|c})$ is analytically derived via block matrix inversion (derivations in Appendix.\ref{app:block_matrix}):
\begin{equation} \label{eq:primal_GJE_mean_and_cov}
\begin{aligned} 
    \mu_{t|c} &= C_{tc} C_{cc}^{-1} z_c^* \\
    \Sigma_{t|c} &= C_{tt} - C_{tc} C_{cc}^{-1} C_{ct}
\end{aligned}
\end{equation}
In Eq.\ref{eq:primal_GJE_mean_and_cov}, the term $(C_{tc} C_{cc}^{-1})$ multiplies a $[d_t \times d_c]$ matrix by a $[d_c \times d_c]$ matrix, resulting in a fixed $[d_t \times d_c]$ weight matrix. Therefore, the primal GJE formulation mathematically derives the \textit{optimal linear projection layer} directly from the global feature covariances, requiring no training data in memory during inference\footnote{Training data info is summarised in the matrices $[C_{cc},C_{ct},C_{tc},C_{tt}]$ via Eq.\ref{eq:primal_GJE_empirical_cov}}. It is observed that, the predictive variance is a constant decided by the training set, and it's independent of the test context $z_c^*$.

% ==========================================
% LIMITATIONS -> GMJE
% ==========================================
\subsection{Sample Space vs. Feature Space}

The Dual GJE (Section.\ref{subsec:dual_GJE}) and the Primal GJE (Section.\ref{subsec:primal_GJE}) both rely on the fundamental Gaussian assumption, but they approach the $z_c \rightarrow z_t$ mapping problem from perpendicular geometric perspectives. 

\textit{Dual GJE (GPJE)} operates in the sample space, asking the geometric question: ``how similar is Image A to Image B?'' By evaluating the $N \times N$ Gram matrix computed via inner products between samples (e.g. $K_{cc} \propto Z_c Z_c^T$), this view does not model individual features, but rather collapses feature vectors into holistic similarity scores. It formulates the predictor as a non-parametric Gaussian Process, assuming the target embeddings $Z_t$ are jointly Gaussian with a covariance structure determined by these kernel similarities of their contexts. The resulting predictive distribution (Eq.\ref{eq:dual_GJE_mean_and_cov}) is dynamic; its conditional mean and variance depend on both the stored training instances and the topological distance of the specific test context $z_c^*$. However, as an instance-based memory model, evaluating the GPJE objective (Eq.\ref{eq:dual_gje_loss1}) requires inverting an $N \times N$ matrix, imposing a severe $\mathcal{O}(N^3)$ computational bottleneck. Further, as the GP marginal likelihood is strictly a \textit{conditional} objective\footnote{see Eq.(2.29) or Eq.(2.30) in \cite{rasmussen2006gaussian}.} $p(Z_t \mid Z_c)$, it lacks a native volume regularizer for the context space, thereby mandating asymmetric heuristics like EMA target networks to prevent representation collapse.

\textit{Primal GJE}, conversely, operates in the feature space, asking the perpendicular question: ``how similar are the individual latent features to each other across the dataset?'' It acts as a parametric model. It assumes a single, unimodal joint Gaussian distribution (Eq.\ref{eq:primal_GJE_joint_Gaussian_pdf}) over the concatenated feature space. By computing the inner products between the features themselves, its empirical covariance matrix $C_{joint} \propto Z^T Z$ has dimension $d \times d$, which, for typical modern embedding sizes where $d \ll N$, is massively smaller and computationally cheaper to invert ($\mathcal{O}(d^3)$) than the dual Gram matrix. In this primal view, the resulting conditional prediction (Eq.\ref{eq:primal_GJE_mean_and_cov}) yields a mean that linearly depends on the test context $z_c^*$, but a predictive covariance that is strictly constant, frozen entirely by the global distribution of the training data. Crucially, because it optimizes the full \textit{joint} likelihood\footnote{See Appendix.\ref{app:primal_GJE_joint_vs_conditional_objectives} a comparison of the joint and conditional NLL opjectives.} $p(z_c, z_t)$, it provides native, symmetric regularizers for both the context and target spaces, eliminating the strict requirement for asymmetric stop-gradients.

\textit{The Primal-Dual Equivalence.} While they approach the geometry from different dimensions, the two formulations are fundamentally linked. In fact, if the non-parametric Dual-GJE is implemented using a standard linear kernel, the exact $\mathcal{O}(N^3)$ Gaussian Process NLL algebraically collapses entirely into the $\mathcal{O}(d^3)$ Primal-GJE covariance matrix framework (proof see Appendix.\ref{app:linear_kernel_equivalence}). This strict primal-dual equivalence mathematically proves that, without the use of computationally expensive non-linear kernel approximations (such as RBF or RFF), the dual sample-space optimization locks the latent geometry to the exact same single, rigid density as the primal feature space\footnote{The surprise comes as: with a linear kernel, Dual is exactly equal to Primal. Therefore, they both suffer from the exact same Unimodal Gaussian Trap.}.

\paragraph{Unimodality and Homoscedasticity Issues}
Regardless of whether GJE is formulated in the non-parametric dual space (Eq.\ref{eq:dual_GJE_mean_and_cov}) or the parametric primal space (Eq.\ref{eq:primal_GJE_mean_and_cov}), both architectures suffer from a shared limitation: \textit{unimodality}. Because both formulations rely on a single joint Gaussian structure, they can only issue a unimodal predictive distribution. Consequently, when faced with diverging semantic branches, i.e. when there are multiple options for $z_{t}^*$ for a given $z_c^*$, the predictive mean $\mu_{t|c}$ is mathematically forced to predict the literal average of the modes (slicing through empty space, see proof in Appendix.\ref{app:mse_derivation}). Neither model can physically route predictions to distinct, isolated target branches (a failure case we will see later in a synthetic example in Section \ref{sec:experiments}).

Further, the parametric primal-GJE is constrained by \textit{homoscedasticity}. While the dual-GJE yields a dynamic, distance-aware predictive covariance\footnote{Standard Gaussian Process regression is heteroscedastic with respect to distance (the variance grows as the test point moves further from the training data, because $k_{*c}$ goes to zero and we are left with the prior variance $k_{**}$).} that changes based on the test context $z_c^*$, the primal-GJE predictive covariance ($\Sigma_{t|c}$ in Eq.\ref{eq:primal_GJE_mean_and_cov}) is a rigid, constant matrix determined entirely by the global training distribution. It contains no reference to the specific input $z_c^*$. Thus, to blanket the multi-modal spread of ambiguous inputs, the primal-GJE model must globally inflate its static variance across the entire latent space, degrading predictive certainty even for unambiguous samples. 

These structural limitations, i.e. the inability to model multiple target branches and the rigid covariance of the parametric space, motivate our transition to Gaussian Mixture Joint Embeddings (GMJE) in the primal feature space.

\section{(Primal Space) Gaussian Mixture Joint Embeddings (GMJE)} \label{sec:GMJE}
GJE is strictly unimodal regardless of the kernel used; either classic JEPA or GJE can solve multi-modal alignment\footnote{We present an ambiguous multi-modal alignment task in Section.\ref{subsec:exp_synthetic} to illustrate this.}. In this section, we present the feature (primal) space \textit{Gaussian Mixture Joint Embeddings} (GMJE), a fully generalized, probabilistically grounded framework for multi-modal representation learning. We first motivate the transition from unimodal to multi-modal joint modeling and establish the universal approximation capabilities of Gaussian mixtures. Next, we define the general GMJE formulation (Section.\ref{subsec:GMJE_formulation}), demonstrating how it uses a joint mixture model to natively capture complex dependencies and multi-modal alignments. The remainder of the section introduces a suite of principled methodologies to optimize this joint distribution, naturally divided into parametric and non-parametric branches. In the parametric branch, we explore a highly scalable approach using \textit{relaxed Expectation-Maximization} with fixed global prototypes (Section.\ref{subsec:GMJE_reduced_EM}), extend it to instance-conditioned parameter generation via \textit{Mixture Density Networks} (Section.\ref{subsec:GMJE_MDN}), and resolve the need for a predefined number of components using dynamic topology mapping via \textit{Growing Neural Gas} (Section.\ref{subsec:GMJE_GNG}). Finally, in the \textit{non-parametric} branch, we provide a theoretical proof linking GMJE to Contrastive Learning (Section.\ref{sec:GMJE_Contrastive}), which motivates our solution to memory bank optimization using a dynamic \textit{Sequential Monte Carlo} (SMC) particle filter (Section.\ref{sec:SMC_Memory_Bank}).

\subsection{From Unimodal to Multi-modal Joint Modeling}

While the single Gaussian assumption in GJE provides an elegant closed-form predictor (\ref{eq:GJE_conditional_mean_cov_final_replicated}), its operation in the dual sample space structurally restricts the conditional distribution $p(z_t | z_c)$ to a single, unimodal projection. In standard self-supervised tasks, the true distribution of valid target views given a context view is inherently \textit{multi-modal}. For example, a heavily masked image of a dog's face ($x_c$) could validly be completed into several distinct postures, breeds, or backgrounds ($x_t$). A unimodal projection over-smoothes these relationships, forcing the predicted mean to lie somewhere in the empty space between valid modes. 

To capture the complex, multi-modal alignment of representations, we extend GJE to \textit{Gaussian Mixture Joint Embeddings} (GMJE), switching from dual sample space view to the primal feature space view. However, note that GJE (sample space) and GMJE (feature space) are not directly comparable in their capacities.
Before detailing our specific joint formulation, we first establish the foundational power of Gaussian Mixture Models (GMMs). A general GMM with $K$ components is defined as \cite{reynolds_gaussian_2009}:
\begin{equation}
    q_{\boldsymbol{\theta}}(\mathbf{z}) = \sum_{i=1}^K w_i \mathcal{N}(\mathbf{z} \mid \boldsymbol{\mu}_i, \boldsymbol{\Sigma}_i)
\end{equation}
where the parameters \(\boldsymbol{\theta} = \{w_i, \boldsymbol{\mu}_i, \boldsymbol{\Sigma}_i\}_{i=1}^K\) consist of mixture weights ($w_i \geq 0$, \(\sum_{i=1}^K w_i = 1\)), means (\(\boldsymbol{\mu}_i \in \mathbb{R}^d\)), and positive definite covariance matrices (\(\boldsymbol{\Sigma}_i \in \mathbb{R}^{d \times d}\)).

The transition to a mixture-based approach is mathematically justified by the profound representational capacity of GMMs, formally stated in the following universal density approximation property \cite{li_mixture_1999}:

\begin{theorem}[Universal approximation property of GMMs \cite{li_mixture_1999, Huang2025GMA}] \label{thm:univ_approx_GMMs}
Any sufficiently smooth probability density function $p(\mathbf{z})$ on \(\mathbb{R}^d\) can be approximated arbitrarily closely in $L^1$ distance by a Gaussian Mixture Model (GMM) with a finite, sufficiently large number of components covering its whole support.
\end{theorem}

\begin{proof} (sketch)
According to results in mixture density estimation theory (e.g. Li \& Barron \cite{li_mixture_1999}, Norets \cite{norets_approximation_2010}), the set of Gaussian mixtures is dense in the space of continuous densities with respect to the $L^1$ norm, provided $p(\mathbf{z})$ has bounded support or decays sufficiently fast at infinity (e.g. faster than any Gaussian tail). 

Let $\epsilon > 0$ be the desired approximation error. The $L^1$ distance\footnote{In density estimation, the Kullback-Leibler (KL) divergence between two densities is often used due to its intrinsic connection with maximum likelihood. However, the $L^1$ norm is utilized in this theoretical analysis for two primary reasons \cite{Huang2025GMA}: its practical probabilistic meaning and its convenient mathematical properties for error decomposition. 
\textbf{(1) Probabilistic interpretation:} the $L^1$ norm has a direct and intuitive connection to how distinguishable two probability distributions are - it is equal to twice the \textit{Total Variation} (TV) distance: $\|p - q\|_{L^1} = 2 \cdot D_{TV}(p, q)$. The TV distance represents the largest possible difference in probability that the two distributions can assign to any single event. For example, if $\|p - q\|_{L^1} = 0.1$, then the TV distance is $0.05$. This guarantees that for any possible event, the probabilities calculated by $p$ and $q$ will differ by at most 5\%. 
\textbf{(2) Mathematical properties for analysis:} unlike KL divergence, the $L^1$ norm is a true \textit{metric}, satisfying symmetry and the triangle inequality. This allows total error to be directly decomposed into manageable parts: $\|p - p_{\text{samples}}\|_{L^1} \le \|p - q_{w_{opt}}\|_{L^1} + \|q_{w_{opt}} - q_{\mathbf{w}_K}\|_{L^1} + \|q_{\mathbf{w}_K} - p_{\text{samples}}\|_{L^1}$. Further, standard bounds for Monte Carlo sampling error are readily available for the $L^1$ distance \cite{taylor_nonparametric_1985,nussbaum_devroye_1988}, allowing them to be plugged directly into analytical proofs.} between the true density $p(\mathbf{z})$ and the approximation $q_{\boldsymbol{\theta}}(\mathbf{z})$ is defined as\footnote{The $L^1$ norm of a function $f(\mathbf{z})$ is defined as $\|f\|_{L^1} = \int |f(\mathbf{z})| d\mathbf{z}$; for any vector $\mathbf{v}=[v_1, v_2, \ldots, v_d]$, the $L^1$ norm is the sum of the absolute values of its components: $\|\mathbf{v}\|_{L^1} = \sum_{i=1}^d |v_i|$.}:
\begin{equation*}
    \| p - q_{\boldsymbol{\theta}} \|_{L^1} = \int_{\mathcal{Z}} |p(\mathbf{z}) - q_{\boldsymbol{\theta}}(\mathbf{z})| d\mathbf{z}
\end{equation*}

By the definition of this denseness property in $L^1$ space, there must exist a finite $K$ and a parameter set \(\boldsymbol{\theta}\) such that $\| p - q_{\boldsymbol{\theta}} \|_{L^1} < \epsilon$. The required number of components $K$ depends purely on the target's geometric complexity (e.g. its curvature and number of modes). Therefore, a GMM with finite $K$ can approximate any continuous target density within an arbitrary $\epsilon$-error bound, mathematically justifying the use of Gaussian mixtures as a flexible, universal representation capable of mapping the highly irregular distributions found in self-supervised learning.
\end{proof}

\begin{corollary}[Universal approximation with isotropic Gaussian mixtures \cite{mclachlan1988mixture, bishop1994mixture}] \label{cor:isotropic_approx}
The universal approximation property holds even if the mixture components are strictly restricted to be isotropic, i.e. $\boldsymbol{\Sigma}_i = \sigma_i^2 I$. A Gaussian mixture model utilizing purely isotropic kernels can approximate any given continuous density function to arbitrary accuracy, provided the mixing coefficients and Gaussian parameters are correctly chosen.
\end{corollary}

\begin{proof} (sketch)
While a single isotropic Gaussian assumes that all dimensions are statistically independent, a mixture of such Gaussians does not inherit this limitation globally \cite{bishop1994mixture}. Global cross-dimensional dependencies (correlations) are natively captured through the spatial arrangement of the multiple mixture centers $\boldsymbol{\mu}_i$. For example, a highly correlated, diagonal data manifold can be approximated to arbitrary precision by tiling the space with a sufficiently large sequence of small, independent spherical Gaussians arranged along the diagonal. Thus, while introducing full covariance matrices allows for fewer components, it is theoretically unnecessary for universal approximation.
\end{proof}

Having established the universal representational power of GMMs, we now formally define how this mixture-based approach is mathematically integrated into the symmetric joint embedding architecture.

\subsection{The GMM Formulation for Joint Embeddings} \label{subsec:GMJE_formulation}
When introducing mixture models, one could natively define a \textit{Conditional GMM}, e.g. a \textit{Mixture Density Network} \cite{bishop1994mixture,huang2025LLMPrior} where $p(z_t|z_c)$ is directly modeled as a Gaussian mixture with weights predicted by a neural net. However, doing so breaks the theoretical symmetry of GJE: we would lose the explicit modeling of the marginal distribution $p(z_c)$, which, as established in Section.\ref{sec:GJE} (Eq.\ref{eq:GJE_joint_pdf2} and Eq.\ref{eq:GJE_joint_pdf3}), provides the native geometric regularization preventing representation collapse (and forcing us back into the realm of heuristic stop-gradients\footnote{The architectural heuristic of applying a stop-gradient to one branch of a \textit{Siamese network} to prevent representation collapse was popularized by \textit{BYOL} \cite{grill2020byol}, which utilized an asymmetric exponential moving average (EMA) target network. The fundamental mathematical role of this stop-gradient operation in preventing trivial constant collapse, even without momentum or negative pairs, was subsequently isolated and proven in \textit{SimSiam} \cite{chen2021simsiam}. This asymmetric stop-gradient paradigm was later adapted specifically in I-JEPA \cite{assran2023ijepa}.}).

Instead, we maintain the symmetric paradigm by defining a \textit{Joint GMM} with $K$ components:
\begin{equation} \label{eq:GMJE_joint}
    p(z_c, z_t) = \sum_{k=1}^K \pi_k \mathcal{N} \left( \begin{bmatrix} z_c \\ z_t \end{bmatrix} \middle| \begin{bmatrix} \mu_{c,k} \\ \mu_{t,k} \end{bmatrix}, \begin{bmatrix} \Sigma_{cc,k} & \Sigma_{ct,k} \\ \Sigma_{tc,k} & \Sigma_{tt,k} \end{bmatrix} \right)
\end{equation}
where $\pi_k$ are the mixture weights ($\sum \pi_k = 1$). 

From this joint formulation, integrating out the context variable $z_c$ yields the marginal distribution of the target representation space $p(z_t)$, which elegantly remains a closed-form Gaussian mixture (and symmetrically so for $p(z_c)$):
\begin{equation} \label{eq:GMJE_marginal_target_replicated1} \tag{cc.Eq.\ref{eq:GMJE_marginal_target}}
    p(z_t) = \int p(z_c, z_t) dz_c = \sum_{k=1}^K \pi_k \mathcal{N}(z_t \mid \mu_{t,k}, \Sigma_{tt,k})
\end{equation}

Further, by the properties of joint Gaussians, the conditional distribution $p(z_t|z_c)$ derived from this joint mixture is exactly a GMM with closed-form parameters (see Appendix.\ref{app:GMJE_conditional_derivation} for a derivation):
\begin{equation} \label{eq:GMJE_conditional_replicated} \tag{cc.Eq.\ref{eq:GMJE_conditional}}  
    p(z_t | z_c) = \sum_{k=1}^K \gamma_k(z_c) \mathcal{N}(z_t | \mu_{t|c, k}, \Sigma_{t|c, k})
\end{equation}
where the conditional means $\mu_{t|c, k}$ and covariances $\Sigma_{t|c, k}$ are calculated for each component $k$ exactly as in GJE (\ref{eq:GJE_conditional_mean_cov_final_replicated}):

\begin{equation} \label{eq:GMJE_component_mean_cov_replicated} \tag{cc.Eq.\ref{eq:GMJE_component_mean_cov}}
\begin{aligned}
\mu_{t|c, k} &= \mu_{t,k} + \Sigma_{tc,k}\Sigma_{cc,k}^{-1}(z_c - \mu_{c,k}) \\
\Sigma_{t|c, k} &= \Sigma_{tt,k} - \Sigma_{tc,k}\Sigma_{cc,k}^{-1}\Sigma_{ct,k}
\end{aligned}
\end{equation}

And the new data-dependent mixing weights $\gamma_k(z_c)$ are given by:
\begin{equation} \label{eq:GMJE_cond_step3_replicated} \tag{cc.Eq.\ref{eq:GMJE_cond_step3}}
    \gamma_k(z_c) = \frac{\pi_k \mathcal{N}(z_c | \mu_{c,k}, \Sigma_{cc,k})}{\sum_{j=1}^K \pi_j \mathcal{N}(z_c | \mu_{c,j}, \Sigma_{cc,j})}
\end{equation}
This formulation is theoretically profound: the context embedding $z_c$ natively acts as a ``router'' that smoothly evaluates its own marginal likelihood under each mode's context distribution, dynamically selecting which mixture component (predictor) is most appropriate to generate the target. 

An schematic of the GMJE architecture is shown in Fig.\ref{fig:gmje_general}. Note the difference between GMJE and classic JEPA (Fig.\ref{fig:comparison_classicJEPA_GJE}(a)) in that, while standard JEPA relies on a rigid asymmetric design, e.g. utilizing stop-gradients and a lagging Exponential Moving Average (EMA) target network to artificially prevent representation collapse, GMJE operates within a fully symmetric architectural framework. Because the joint probabilistic objective natively penalizes dimensional collapse, gradients can flow freely and simultaneously through both the context and target encoders without requiring any such architectural heuristics.

\begin{figure}[H]
    \centering
    \begin{tikzpicture}[
        node distance=1.8cm,
        box/.style={draw, rounded corners, minimum width=2.8cm, minimum height=1.8cm, align=center, thick},
        mini box/.style={draw, rounded corners, minimum width=1.4cm, minimum height=1.8cm, align=center, thick},
        wide arrow/.style={-stealth, line width=3pt, draw=black!40},
        font=\sffamily
    ]
    
    % Neural Network Group (Invisible container for routing arrows)
    \node (nn) [minimum width=3.2cm, minimum height=1.8cm] {};
    
    % Dual Encoders
    \node (enc_c) [mini box] at ([xshift=-0.8cm]nn.center) {};
    \node (enc_t) [mini box] at ([xshift=0.8cm]nn.center) {};
    
    % Inputs
    \node (input_c) [below of=enc_c, yshift=-0.2cm, align=center] {\textbf{context} \\ $x_c$};
    \node (input_t) [below of=enc_t, yshift=-0.2cm, align=center] {\textbf{target} \\ $x_t$};
    
    % Draw connected NN nodes inside Context Encoder (enc_c)
    \begin{scope}[shift={(enc_c.center)}]
        \foreach \i/\y in {1/-0.5, 2/0, 3/0.5} {
            \node[draw, circle, inner sep=1.2pt] (c1\i) at (-0.35, \y) {};
            \node[draw, circle, inner sep=1.2pt] (c2\i) at (0.35, \y) {};
        }
        \foreach \i in {1,2,3} {
            \foreach \j in {1,2,3} {
                \draw[thin, gray] (c1\i) -- (c2\j);
            }
        }
    \end{scope}
    
    % Draw connected NN nodes inside Target Encoder (enc_t)
    \begin{scope}[shift={(enc_t.center)}]
        \foreach \i/\y in {1/-0.5, 2/0, 3/0.5} {
            \node[draw, circle, inner sep=1.2pt] (t1\i) at (-0.35, \y) {};
            \node[draw, circle, inner sep=1.2pt] (t2\i) at (0.35, \y) {};
        }
        \foreach \i in {1,2,3} {
            \foreach \j in {1,2,3} {
                \draw[thin, gray] (t1\i) -- (t2\j);
            }
        }
    \end{scope}
    
    \node [right=1.5cm of nn, align=center] {\textbf{dual} \\ \textbf{encoders}};
    
    % Joint Embedding Output
    \node (param1) [above of=nn, yshift=0.3cm, align=center] {\textbf{joint embedding} \\ $Z = [z_c^T, z_t^T]^T$};

    % Mixture Model Box (Expanded slightly to fit the labels and K components)
    \node (gmm) [box, above of=param1, yshift=1.0cm, minimum height=3.4cm, minimum width=3.2cm] {};
    
    % Draw shifted and scaled Gaussians inside the box
    \begin{scope}[shift={(gmm.center)}]
        % Top bell (narrow variance, shifted left)
        \draw[thick] (-1.4, 0.9) -- (-0.7, 0.9) .. controls (-0.5, 0.9) and (-0.5, 1.5) .. (-0.3, 1.5) .. controls (-0.1, 1.5) and (-0.1, 0.9) .. (0.1, 0.9) -- (1.4, 0.9);
        \node[font=\scriptsize] at (0.6, 1.2) {$(\mu_1, \Sigma_1)$};

        % Middle bell (wide variance, shifted right)
        \draw[thick] (-1.4, 0.1) -- (-0.4, 0.1) .. controls (-0.1, 0.1) and (-0.1, 0.5) .. (0.3, 0.5) .. controls (0.7, 0.5) and (0.7, 0.1) .. (1.0, 0.1) -- (1.4, 0.1);
        \node[font=\scriptsize] at (-0.6, 0.3) {$(\mu_2, \Sigma_2)$};

        % Bottom bell (medium variance, slightly left)
        \draw[thick] (-1.4, -0.7) -- (-0.6, -0.7) .. controls (-0.35, -0.7) and (-0.35, -0.2) .. (-0.1, -0.2) .. controls (0.15, -0.2) and (0.15, -0.7) .. (0.4, -0.7) -- (1.4, -0.7);
        \node[font=\scriptsize] at (0.7, -0.4) {$(\mu_3, \Sigma_3)$};
        
        % K components indicator
        \node[font=\normalsize] at (0, -1.1) {$\vdots$};
        \node[font=\scriptsize] at (0, -1.4) {$K$ components};
    \end{scope}
    
    % Label indicating global parameters
    \node [right=1.5cm of gmm, align=center] {\textbf{joint mixture} \\ \textbf{model} \\ (global $\mu_k, \Sigma_k, \pi_k$)};
    
    % Output
    \node (output) [above of=gmm, yshift=1.0cm, align=center] {\textbf{joint probability density} \\ $\mathbf{p(z_c, z_t)} = \sum_{k=1}^K \pi_k \mathcal{N}(Z \mid \mu_k, \Sigma_k)$};
    
    % Arrows
    \draw [wide arrow] (input_c) -- (enc_c);
    \draw [wide arrow] (input_t) -- (enc_t);
    \draw [wide arrow] (nn) -- (param1);
    \draw [wide arrow] (param1) -- (gmm);
    \draw [wide arrow] (gmm) -- (output);
    
    \end{tikzpicture}
    \caption{The general Gaussian Mixture Joint Embeddings (GMJE) framework. Dual encoders map the context and target views to a joint embedding $Z$. This embedding is evaluated against a set of $K$ learnable global mixture parameters ($\mu_k, \Sigma_k, \pi_k$) to natively model the joint probability density $p(z_c, z_t)$ within a symmetric architectural framework.}
    \label{fig:gmje_general}
\end{figure}
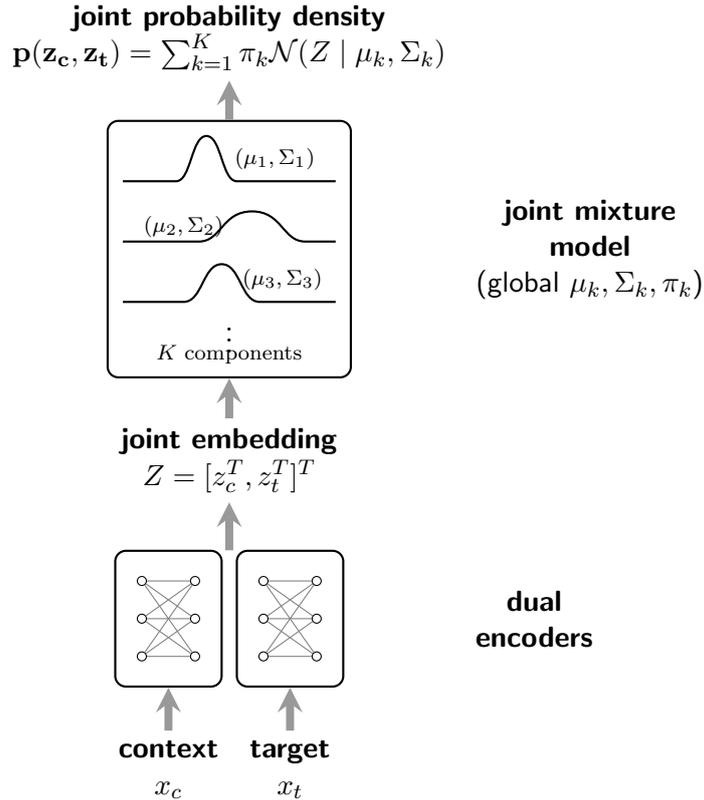

Having established the general GMJE architecture and its corresponding joint probability density, the critical subsequent challenge lies in \textit{effectively optimizing the $K$ mixture components ($\mu_k, \Sigma_k, \pi_k$) to accurately map the underlying data manifold}. Because the true representation space in self-supervised learning is highly complex and dynamically changing during training, we cannot rely on trivial clustering initializations. In the following subsections, we introduce diverse, theoretically grounded methods to learn these parameters. We first discuss the \textbf{parametric GMJE approach} with $K$ components.

\subsection{Scalable Optimization via Learnable Prototypes} \label{subsec:GMJE_reduced_EM}

To train GMJE efficiently on massive datasets, we introduce a parametric, prototypical approach. We maintain a set of $K$ learnable joint mixture means $\mu_k = [\mu_{c,k}^T, \mu_{t,k}^T]^T$, acting as global cluster prototypes, analogous to those used\footnote{Specifically, DeepCluster groups representations offline via $k$-means to generate discrete pseudo-labels and trains the network to predict them via cross-entropy \cite{caron2019deepcluster}. SwAV operates online, computing soft assignments (codes) for one augmented view against learnable prototypes using optimal transport, and predicting this code from a differently augmented view \cite{caron2021swav}.} in \textit{SwAV} \cite{caron2021swav} or \textit{DeepCluster} \cite{caron2019deepcluster}, alongside a shared, parameterized covariance structure $\Sigma$.

Rather than relying on a traditional, often unstable Expectation-Maximization (EM) loop, we perform a \textit{differentiable, relaxed EM directly through backpropagation}. We define our objective as minimizing the Negative Log-Likelihood (NLL) of the joint mixture model for an observed concatenated embedding $Z = [z_c^T, z_t^T]^T$. 

Starting from the joint distribution $p(Z)$ (Eq.\ref{eq:GMJE_joint}) and assuming a shared\footnote{This assumption can be loosened, though keeping it shared drastically improves tractability and numerical stability.} covariance $\Sigma$ across all components for tractability, the NLL expands as:
\begin{align} \label{eq:GMJE_loss_expansion}
    \mathcal{L}_{\text{GMJE-Proto}} &= - \log p(Z) \nonumber \\
    &= - \log \left( \sum_{k=1}^K \pi_k \mathcal{N}(Z \mid \mu_k, \Sigma) \right) \nonumber \\
    &= - \log \left( \sum_{k=1}^K \pi_k \frac{1}{\sqrt{(2\pi)^{2d}|\Sigma|}} \exp \left( - \frac{1}{2} (Z - \mu_k)^T \Sigma^{-1} (Z - \mu_k) \right) \right)
\end{align}

To compute this loss without numerical overflow (a common failure mode when summing small likelihoods), we absorb the mixing weights $\pi_k$ and the Gaussian normalization constants directly into the exponential using the identity $x = \exp(\log x)$. Ignoring the fixed $2\pi$ constant, which does not affect the gradients, the objective simplifies to:
\begin{equation} \label{eq:GMJE_prototypical_loss}
    \mathcal{L}_{\text{GMJE-Proto}} = - \log \left( \sum_{k=1}^K \exp \left( \log \pi_k - \frac{1}{2} (Z - \mu_k)^T \Sigma^{-1} (Z - \mu_k) - \frac{1}{2} \log |\Sigma| \right) \right)
\end{equation}

By formatting the loss purely as the logarithm of a sum of exponentials, we can optimize it using the \textit{Log-Sum-Exp (LSE) trick}\footnote{The Log-Sum-Exp (LSE) trick is a standard technique used to prevent numerical overflow when computing the logarithm of a sum of exponentials. Evaluating $\exp(x_k)$ directly for large values of $x_k$ easily exceeds machine precision, resulting in \texttt{NaN} or infinity. By defining $c = \max_k x_k$, we can leverage the algebraic identity $\log \sum_{k=1}^K \exp(x_k) = c + \log \sum_{k=1}^K \exp(x_k - c)$. This ensures the largest exponent evaluated is exactly $\exp(0) = 1$, safely bounding the sum to prevent overflow while harmlessly allowing negligible values to underflow to zero.} \cite{blanchard2019lse,Huang2025GMA} for numerical stability. Because the covariance $\Sigma$ is shared across all components, its determinant can be elegantly factored out of the LSE operator:

\begin{equation} \label{eq:optimize_GMM_components_with_uniform_covariance}
\begin{aligned}
    \pi_k, \mu_k, \Sigma &= \arg\min_{\pi_k, \mu_k, \Sigma} \mathbb{E}_{Z} \Bigg[ \frac{1}{2} \log |\Sigma| - c(Z) - \log \sum_{k=1}^K \exp \bigg( \log \pi_k - \frac{1}{2} (Z - \mu_k)^T \Sigma^{-1} (Z - \mu_k) - c(Z) \bigg) \Bigg] \\
    \text{where} \quad c(Z) &= \max_{j \in \{1 \dots K\}} \left( \log \pi_j - \frac{1}{2} (Z - \mu_j)^T \Sigma^{-1} (Z - \mu_j) \right)
\end{aligned}
\end{equation}

During the forward pass, the network computes the Mahalanobis distance of the current joint embedding $Z$ to all $K$ prototypes. The optimization of this objective function governs the latent geometry through two distinct, opposing forces:

\textbf{1. The \textit{Pulling Force} (Data-Fit via Soft-Routing):} 
the Log-Sum-Exp (LSE) function mathematically behaves as a smooth, differentiable approximation of the maximum operator ($\max$). When computing the gradients of the LSE term during backpropagation, the derivative yields the exact softmax probabilities\footnote{To see this mathematically, consider the Log-Sum-Exp function $f(\mathbf{x}) = \log \sum_{j=1}^K \exp(x_j)$. By applying the chain rule, the partial derivative with respect to a specific logit $x_k$ is evaluated as $\frac{\partial f}{\partial x_k} = \frac{1}{\sum_{j=1}^K \exp(x_j)} \frac{\partial}{\partial x_k} \sum_{j=1}^K \exp(x_j) = \frac{\exp(x_k)}{\sum_{j=1}^K \exp(x_j)}$, which is precisely the definition of the standard softmax function.}, which perfectly correspond to the posterior mixture responsibilities $\gamma_k(Z)$. Consequently, the majority of the gradient signal is routed exclusively to the dominant prototype, i.e. the one with the smallest Mahalanobis distance to $Z$. This disproportionate gradient update actively pulls the embedding $Z$ and the winning prototype $\mu_k$ closer together, effectively performing a probabilistic, soft-clustering alignment.

\textbf{2. The \textit{Pushing Force} (Geometric Regularization):} 
simultaneously, the model must prevent the trivial collapse of all prototypes to a single point \footnote{If we do not minimize $\tfrac{1}{2}\log|\Sigma|$, the network would become lazy: it would collapse all prototypes $\mu_k$ into the center of the space and inflate $\Sigma$ to be extremely wide, covering all data points at once (a single giant overlapping blob). By compressing the local variance, the discrete means $\mu_k$ are forced to carry the burden of the dataset's global variance.}. This defense is natively governed by the shared log-determinant complexity penalty, $\frac{1}{2} \log |\Sigma|$. This term acts as an \textit{information-theoretic regularizer} that strictly penalizes the internal volume of the cluster covariances. To successfully encode a highly diverse, full-rank dataset without inflating $\Sigma$ (which would incur a massive NLL penalty), the optimization process is mathematically forced to distribute the $K$ prototypes widely across the latent space to cover the data manifold. 

\textbf{The Local Compression vs. Global Spread Trade-off:} 
similar to the GJE joint objective case (Eq.\ref{eq:primal_gje_loss}), ultimately, balance of these two forces creates a profound geometric trade-off. On a local level, minimizing the complexity penalty $\frac{1}{2} \log |\Sigma|$ explicitly \textit{discourages} diversity within \textit{individual} clusters by compressing their covariance volume. However, because the inverse matrix $\Sigma^{-1}$ dictates the Mahalanobis distance in the data-fit term, this extreme local compression causes penalty for misaligned embeddings to explode. To survive this massive data-fit term without lazily inflating $\Sigma$, the network is forced to \textit{encourage} diversity on a \textit{global} level, i.e. to distribute the clusters \textit{widely} in the latent space. It cannot rely on a few massive, overlapping clusters to absorb the dataset; instead, it must spread the discrete prototypes $\mu_k$ extensively across the space to cover all training data points. By aggressively compressing the local variance, the objective natively transfers the burden of representing the dataset's total variance onto the global, discrete prototypes, dynamically learning a rich representation space without requiring heuristic uniformity losses.

\paragraph{Extension to Non-Uniform Covariances.} 
While sharing a single covariance matrix $\Sigma$ across all components drastically improves computational efficiency and numerical stability without losing versatile approximation power (an even stronger realisation of Corollary.\ref{cor:isotropic_approx}), the GMJE framework natively extends to the fully general case where each mixture component $k$ maintains its own distinct covariance matrix $\Sigma_k$. In this unconstrained formulation, the local complexity penalty $\frac{1}{2} \log |\Sigma_k|$ is dependent on the component index $k$ and therefore cannot be factored out of the LSE operator. Following the same derivation, the fully general prototypical loss expands to:
\begin{equation} \label{eq:GMJE_prototypical_loss_nonuniform}
\begin{aligned}
    \mathcal{L}_{\text{GMJE-Proto}} &= - \log \left( \sum_{k=1}^K \exp \left( \log \pi_k - \frac{1}{2} (Z - \mu_k)^T \Sigma_k^{-1} (Z - \mu_k) - \frac{1}{2} \log |\Sigma_k| \right) \right) \\
    &= - c(Z) - \log \sum_{k=1}^K \exp \bigg( \log \pi_k - \frac{1}{2} (Z - \mu_k)^T \Sigma_k^{-1} (Z - \mu_k) - \frac{1}{2} \log |\Sigma_k| - c(Z) \bigg) \\
    \text{where} \quad c(Z) &= \max_{j \in \{1 \dots K\}} \left( \log \pi_j - \frac{1}{2} (Z - \mu_j)^T \Sigma_j^{-1} (Z - \mu_j) - \frac{1}{2} \log |\Sigma_j| \right)
\end{aligned}
\end{equation}
Geometrically, this allows the latent space to capture \textit{heterogeneous} data distributions. Certain prototypes can learn to encode tight, highly specific concepts (driving their individual $|\Sigma_k|$ to be small), while others can encode broader, more variable concepts (tolerating a larger $|\Sigma_k|$ to minimize the data-fit term for diverse samples). However, this increased expressiveness comes at an increased computational cost, requiring $K$ independent matrix inversions ($\Sigma_k^{-1}$) and log-determinants ($\log |\Sigma_k|$) per optimization step, and may require careful regularization (such as diagonal jitter) to prevent any individual $\Sigma_k$ from collapsing to singularity \cite{Huang2025GMA}.

While optimizing fixed global prototypes via relaxed EM successfully captures macroscopic semantic clusters, it relies on static parameters ($\mu_k, \Sigma_k, \pi_k$) that do not adapt to the specific epistemic ambiguity of individual inputs. To achieve true instance-level predictive flexibility while maintaining a fixed number of components $K$, we explore integrating Mixture Density Networks into the GMJE framework.

\subsection{GMJE-MDN: Mixture Density Network (MDN) for GMJE Learning} \label{subsec:GMJE_MDN}

Standard neural networks trained with a sum-of-squares (MSE) or cross-entropy errors are fundamentally limited: they converge to approximate the conditional average of the target data, conditioned on the input (see a proof in Appendix.\ref{app:mse_derivation}). For applications involving multi-valued mappings or complex inverse problems, this conditional average represents a very limited statistical description. Particularly, in a multi-modal space, the average of several correct target values is not necessarily a correct value itself, frequently leading to completely erroneous results.

\subsubsection{MDN} \label{subsubsec:MDN}

To obtain a complete statistical description of the data, Bishop introduced the Mixture Density Network (MDN) \cite{bishop1994mixture}, a framework obtained by combining a conventional neural network with a mixture density model. An MDN can, in principle, represent arbitrary conditional\footnote{MDNs model conditional distributions: in an MDN, those parameters are not fixed constants; instead, they depend on, and are dynamic functions of, an input vector $x$. We feed an input $x$ into a feed-forward neural network, and the network's final layer outputs the specific means $\mu_i(x)$, variances $\sigma_i(x)^2$, and mixing weights $\alpha_i(x)$ to be used for the target variable $t$. The GMM therefore models the conditional density function $p(t|x)$, i.e. the probability of the target $t$, conditioned on the fact that we just observed input $x$ \cite{bishop1994mixture}.} probability distributions in the same way that standard networks represent arbitrary functions. Specifically, the conditional probability density of the target data $t$ is represented as a linear combination of $K$ kernel functions \cite{bishop1994mixture}:
\begin{equation}
    p(t|x) = \sum_{i=1}^K \alpha_i(x) \phi_i(t|x)
\end{equation}
where the mixing coefficients $\alpha_i(x)$ act as prior probabilities conditioned on $x$. For continuous variables, the kernel functions $\phi_i(t|x)$ are typically chosen to be Gaussians:
\begin{equation}
    \phi_i(t|x) = \frac{1}{(2\pi)^{c/2}\sigma_i(x)^c} \exp\left\{-\frac{||t-\mu_i(x)||^2}{2\sigma_i(x)^2}\right\}
\end{equation}
where $c$ denotes the dimensionality of the target vector $t \in \mathbb{R}^c$, $\mu_i(x)$ represents the center of the $i$-th kernel, and $\sigma_i(x)^2$ represents a common, scalar variance\footnote{Note on notation in THIS section: throughout this brief background review, we preserve the original mathematical symbols used by Bishop \cite{bishop1994mixture}. For conceptual clarity, Bishop's dynamic mixing weights $\alpha_i(x)$ are directly analogous to the mixture weights ($\pi_k$ or $\gamma_k$) in our GMJE framework, his Gaussian kernel functions $\phi_i(t|x)$ correspond to our component probability densities $\mathcal{N}(\cdot)$, and the generic input-target variables $(x, t)$ map exactly to our context and target embeddings $(z_c, z_t)$. He used $c$ as a notation of dimension, while we used $d$.}. 
Note that, this formulation inherently assumes the components of the target vector are statistically independent \textit{locally within} each individual kernel (i.e. an isotropic covariance structure). While this assumption can be relaxed by introducing full covariance matrices, such a complication is theoretically unnecessary \cite{bishop1994mixture}; a Gaussian mixture model utilizing purely \textit{isotropic} kernels can approximate any given density function to arbitrary accuracy, provided the mixing coefficients and Gaussian parameters are correctly chosen (see Corollary.\ref{cor:isotropic_approx}) \cite{mclachlan1988mixture}. Thus, the overall representation remains completely general. In particular, even when constructed purely from isotropic components, the full mixture $p(t|x)$ natively captures global cross-dimensional dependencies through the spatial arrangement of its multiple centers $\mu_i(x)$, avoiding the strict global independence assumed by a single isotropic Gaussian\footnote{As illustrated in the proof of Corollary.\ref{cor:isotropic_approx}, here is an intuitive example: imagine a 2D space where we want to model a diagonal line (a strong correlation between dimension $x$ and dimension $y$). A single isotropic Gaussian is just a single perfect circle. It assumes $x$ and $y$ are totally independent globally. It cannot stretch diagonally to fit the line. However, if we take 10 isotropic Gaussians (10 small perfect circles) and line them up diagonally corner-to-corner, the overall shape they create is a diagonal line.}.

In the standard MDN architecture, a feed-forward neural network takes the input $x$ and produces raw outputs (logits), denoted here as $h(x)$, that directly parameterize this mixture model. Specifically, the component centers $\mu_i(x)$ are derived directly from the linear network outputs\footnote{See Eq.(27) in \cite{bishop1994mixture}.}:
\begin{equation}
    \mu_i(x) = h_i^\mu(x)
\end{equation}
The mixing coefficients $\alpha_i(x)$ are constrained by a softmax function applied to their corresponding network outputs to ensure they lie in the range $(0,1)$ and sum to unity\footnote{See Eq.(25) in \cite{bishop1994mixture}.}:
\begin{equation}
    \alpha_i(x) = \frac{\exp(h_i^\alpha(x))}{\sum_{j=1}^K \exp(h_j^\alpha(x))}
\end{equation}
Importantly, the variances $\sigma_i(x)$ are represented using an exponential function of the network outputs\footnote{See Eq.(26) in \cite{bishop1994mixture}.}:
\begin{equation}
    \sigma_i(x) = \exp(h_i^\sigma(x))
\end{equation}
This exponential formulation serves two critical purposes. First, it prevents pathological zero states (since $\exp(\cdot)$ is strictly positive). Second, and more importantly, it allows the variance to be directly and non-linearly parameterized by the input $x$. By doing so, the MDN breaks free from the \textit{homoscedasticity trap} (constant variance) inherent to classic joint Gaussians. For example, when the neural network observes an input that maps to a noisy or ambiguous region (e.g. $x=0$), it can dynamically output a large variance ($\sigma=0.5$). Conversely, when it observes a highly certain input (e.g. $x=1$), it can tighten the distribution to a very small variance ($\sigma=0.05$). 

The network is then optimized end-to-end by minimizing the negative logarithm of the mixture's likelihood\footnote{Same as Eq.(29) in \cite{bishop1994mixture}.} for an observed pattern $(x, t)$:
\begin{equation} \label{eq:MDN_objective}
    E = - \log \left( \sum_{i=1}^K \alpha_i(x) \phi_i(t|x) \right)
\end{equation}

\subsubsection{GMJE-MDN}

The GMJE framework integrates the core probabilistic principles of MDN, but fundamentally \textit{re-architects} them for self-supervised representation learning. A standard MDN is designed specifically to model the conditional probability density $p(t|x)$. As discussed in Section.\ref{subsec:GMJE_formulation}, mapping this directly to our embeddings (i.e. modeling $p(z_t|z_c)$ via a neural network) would break the representational symmetry between the context and target views, losing the explicit modeling of the marginal distribution $p(z_c)$ that provides the native geometric regularization preventing representation collapse.

However, as illustrated in Appendix.\ref{app:identity_collapse} (which was noted later in executing one of our vision benchmarks), if we parameterize the network directly from the full joint embedding $Z = [z_c^T, z_t^T]^T$ introduces a fatal loophole: \textit{Identity Collapse or Information Leakage}. Because the network can observe $z_t$, it trivially learns the identity function ($\mu_k(Z) = Z$), shrinking the predicted variance to zero and ignoring the semantic manifold entirely.

To resolve this conflict, i.e. preventing both identity collapse and representation collapse, we design a GMJE-MDN architecture which explicitly factorizes the joint objective: $-\log p(z_c, z_t) = -\log p(z_c) - \log p(z_t|z_c)$. As shown in Fig.\ref{fig:gmje_mdn}, we enforce a strict \textit{Information Bottleneck}: a parameter network is permitted to observe \textit{only} the context embedding $z_c$, utilizing it to dynamically predict the conditional mixture parameters ($\mu_k(z_c), \Sigma_k(z_c), \pi_k(z_c)$) required to evaluate the target $z_t$. This design separates prediction and evaluation. To maintain the symmetric joint framework and prevent encoder collapse, this conditional MDN objective ($-\log p(z_t|z_c)$) is coupled with a marginal regularization loss on $z_c$ (e.g. via a global parametric covariance constraint). Consequently, GMJE-MDN learns a dynamic, multi-modal alignment generator while preserving the probabilistic guards necessary for robust self-supervised learning.

Formally, the complete end-to-end training objective for the GMJE-MDN architecture evaluates to:
\begin{equation} \label{eq:GMJE_MDN_objective}
\begin{aligned}
    \mathcal{L}_{\text{GMJE-MDN}} &= - \log p(z_c, z_t) \\
    &= \underbrace{- \log p(z_c)}_{\text{Marginal Regularization}} \underbrace{- \log p(z_t \mid z_c)}_{\text{Conditional MDN}} \\
    &= \left( \frac{1}{2} z_c^T \Sigma_c^{-1} z_c + \frac{1}{2} \log |\Sigma_c| \right) - \log \left( \sum_{k=1}^K \pi_k(z_c) \mathcal{N}\Big(z_t \;\Big|\; \mu_k(z_c), \Sigma_k(z_c)\Big) \right)
\end{aligned}
\end{equation}
where $\Sigma_c$ is a global covariance matrix\footnote{This covariance matrix can be tracked via an empirical batch covariance or a moving average during training. If we track $\Sigma_c$ via EMA, then $\Sigma_{\text{EMA}}$ trails behind the current batch. Because it is a historical average, $\nabla_{z} (z^T \Sigma_{\text{EMA}}^{-1} z) = 2 \Sigma_{\text{EMA}}^{-1} z$. The gradients do not cancel. The expansive Mahalanobis barrier works exactly as intended.} governing the marginal distribution of the context space, and the conditional parameters $\mu_k(z_c)$, $\Sigma_k(z_c)$, and $\pi_k(z_c)$ are dynamic functions of the context embedding $z_c$, predicted directly by the parameter network.
During training, \textit{the dual encoders and the parameter network are optimized simultaneously}. By minimizing $\mathcal{L}_{\text{GMJE-MDN}}$ via standard gradient descent, gradients flow backward from the conditional mixture evaluation into the target encoder, and through the parameter network into the context encoder, requiring \textit{no alternating optimization phases or heuristic stop-gradients}.

This explicit factorization in Eq.\ref{eq:GMJE_MDN_objective} is the mathematical key to the architecture's stability. The right-hand term (the MDN objective) acts as a highly expressive, multi-modal \textit{data-fit} routing mechanism. Simultaneously, the left-hand term (the full marginal negative log-likelihood) acts as a compound geometric guard against representation collapse. Similar to GJE (joint loss, Eq.\ref{eq:primal_gje_loss}) and GMJE-Proto (Eq.\ref{eq:optimize_GMM_components_with_uniform_covariance}), minimizing the log-determinant alone would trivially collapse the space to zero volume. However, because a collapsed $\Sigma_c$ forces its inverse $\Sigma_c^{-1}$ to explode, the Mahalanobis distance term ($\frac{1}{2} z_c^T \Sigma_c^{-1} z_c$) fiercely penalizes any degenerate shrinkage. This elegant tension continuously maintains a well-conditioned, full-rank context space, satisfying the symmetric constraints required by self-supervised learning without restricting the target predictions to a rigid linear covariance structure.

\begin{figure}[H]
    \centering
    \begin{tikzpicture}[
        node distance=1.8cm,
        box/.style={draw, rounded corners, minimum width=2.8cm, minimum height=1.8cm, align=center, thick},
        mini box/.style={draw, rounded corners, minimum width=1.4cm, minimum height=1.8cm, align=center, thick},
        wide arrow/.style={-stealth, line width=3pt, draw=black!40},
        font=\sffamily
    ]
    
    % Neural Network Group (Invisible container for routing arrows)
    \node (nn) [minimum width=3.2cm, minimum height=1.8cm] {};
    
    % Inputs
    \node (input_c) [align=center, xshift=-1.5cm] {\textbf{context} \\ $x_c$};
    \node (input_t) [align=center, xshift=1.5cm] {\textbf{target} \\ $x_t$};
    
    % Dual Encoders
    \node (enc_c) [mini box, above of=input_c, yshift=0.2cm] {};
    \node (enc_t) [mini box, above of=input_t, yshift=0.2cm] {};
    
    % Draw connected NN nodes inside Context Encoder (enc_c)
    \begin{scope}[shift={(enc_c.center)}]
        \foreach \i/\y in {1/-0.5, 2/0, 3/0.5} {
            \node[draw, circle, inner sep=1.2pt] (c1\i) at (-0.35, \y) {};
            \node[draw, circle, inner sep=1.2pt] (c2\i) at (0.35, \y) {};
        }
        \foreach \i in {1,2,3} {
            \foreach \j in {1,2,3} {
                \draw[thin, gray] (c1\i) -- (c2\j);
            }
        }
    \end{scope}
    
    % Draw connected NN nodes inside Target Encoder (enc_t)
    \begin{scope}[shift={(enc_t.center)}]
        \foreach \i/\y in {1/-0.5, 2/0, 3/0.5} {
            \node[draw, circle, inner sep=1.2pt] (t1\i) at (-0.35, \y) {};
            \node[draw, circle, inner sep=1.2pt] (t2\i) at (0.35, \y) {};
        }
        \foreach \i in {1,2,3} {
            \foreach \j in {1,2,3} {
                \draw[thin, gray] (t1\i) -- (t2\j);
            }
        }
    \end{scope}
    
    \node [left=1cm of enc_c, align=center] {\textbf{dual} \\ \textbf{encoders}};
    
    % Embeddings
    \node (zc) [above of=enc_c, yshift=0.2cm, align=center] {\textbf{context embed} \\ $z_c$};
    \node (zt) [above of=enc_t, yshift=3.4cm, align=center] {\textbf{target embed} \\ $z_t$};
    
    % Third Neural Network (Parameter Network)
    \node (nn3) [box, above of=zc, yshift=0.6cm, minimum width=3.2cm] {};
    \begin{scope}[shift={(nn3.center)}]
        \foreach \i/\y in {1/-0.5, 2/0, 3/0.5} {
            \node[draw, circle, inner sep=1.2pt] (p1\i) at (-0.6, \y) {};
            \node[draw, circle, inner sep=1.2pt] (p2\i) at (0.6, \y) {};
        }
        \foreach \i in {1,2,3} {
            \foreach \j in {1,2,3} {
                \draw[thin, gray] (p1\i) -- (p2\j);
            }
        }
    \end{scope}
    \node [left=1.8cm of nn3, align=center] {\textbf{parameter} \\ \textbf{network} \\ \textit{(bottleneck)}};

    % MDN Parameters Output
    \node (param2) [above of=nn3, yshift=0.3cm, align=center] {\textbf{mixture parameters} \\ $\mu_k(z_c), \Sigma_k(z_c), \pi_k(z_c)$};

    % Mixture Model Box
    \coordinate (gmm_center) at ($(param2.north)!0.5!(zt.west)$);
    \node (gmm) [box, above of=param2, xshift=1.5cm, yshift=1.2cm, minimum height=3.4cm, minimum width=3.2cm] {};
    
    % Draw shifted and scaled Gaussians inside the box
    \begin{scope}[shift={(gmm.center)}]
        % Top bell (narrow variance, shifted left)
        \draw[thick] (-1.4, 0.9) -- (-0.7, 0.9) .. controls (-0.5, 0.9) and (-0.5, 1.5) .. (-0.3, 1.5) .. controls (-0.1, 1.5) and (-0.1, 0.9) .. (0.1, 0.9) -- (1.4, 0.9);
        \node[font=\scriptsize] at (0.6, 1.2) {$(\mu_1, \Sigma_1)$};

        % Middle bell (wide variance, shifted right)
        \draw[thick] (-1.4, 0.1) -- (-0.4, 0.1) .. controls (-0.1, 0.1) and (-0.1, 0.5) .. (0.3, 0.5) .. controls (0.7, 0.5) and (0.7, 0.1) .. (1.0, 0.1) -- (1.4, 0.1);
        \node[font=\scriptsize] at (-0.6, 0.3) {$(\mu_2, \Sigma_2)$};

        % Bottom bell (medium variance, slightly left)
        \draw[thick] (-1.4, -0.7) -- (-0.6, -0.7) .. controls (-0.35, -0.7) and (-0.35, -0.2) .. (-0.1, -0.2) .. controls (0.15, -0.2) and (0.15, -0.7) .. (0.4, -0.7) -- (1.4, -0.7);
        \node[font=\scriptsize] at (0.7, -0.4) {$(\mu_3, \Sigma_3)$};
        
        % K components indicator
        \node[font=\normalsize] at (0, -1.1) {$\vdots$};
        \node[font=\scriptsize] at (0, -1.4) {$K$ components};
    \end{scope}
    
    \node [right=1.8cm of gmm, align=center] {\textbf{conditional} \\ \textbf{mixture model}};
    
    % Output
    \node (output) [above of=gmm, yshift=1.0cm, align=center] {\textbf{conditional probability density} \\ $\mathbf{p(z_t|z_c)} = \sum_{k=1}^K \pi_k(z_c) \mathcal{N}(z_t \mid \mu_k(z_c), \Sigma_k(z_c))$};
    
    % Arrows
    \draw [wide arrow] (input_c) -- (enc_c);
    \draw [wide arrow] (input_t) -- (enc_t);
    \draw [wide arrow] (enc_c) -- (zc);
    \draw [wide arrow] (enc_t) -- (zt);
    \draw [wide arrow] (zc) -- (nn3);
    \draw [wide arrow] (nn3) -- (param2);
    \draw [wide arrow] (param2) |- ([yshift=-2.5cm, xshift=-0.5cm]gmm);
    \draw [wide arrow] (zt) |- ([yshift=-2.5cm, xshift=0.5cm]gmm);
    \draw [wide arrow] (gmm) -- (output);
    
    \end{tikzpicture}
    \caption{The Information Bottleneck in the GMJE-MDN architecture. To prevent identity collapse, the parameter network must take only the context embedding $z_c$ as input to predict the conditional mixture parameters. These generated parameters are then evaluated against the target embedding $z_t$ inside the conditional mixture model. To prevent representation collapse, this conditional objective is coupled with a marginal loss on $z_c$.}
    \label{fig:gmje_mdn}
\end{figure}
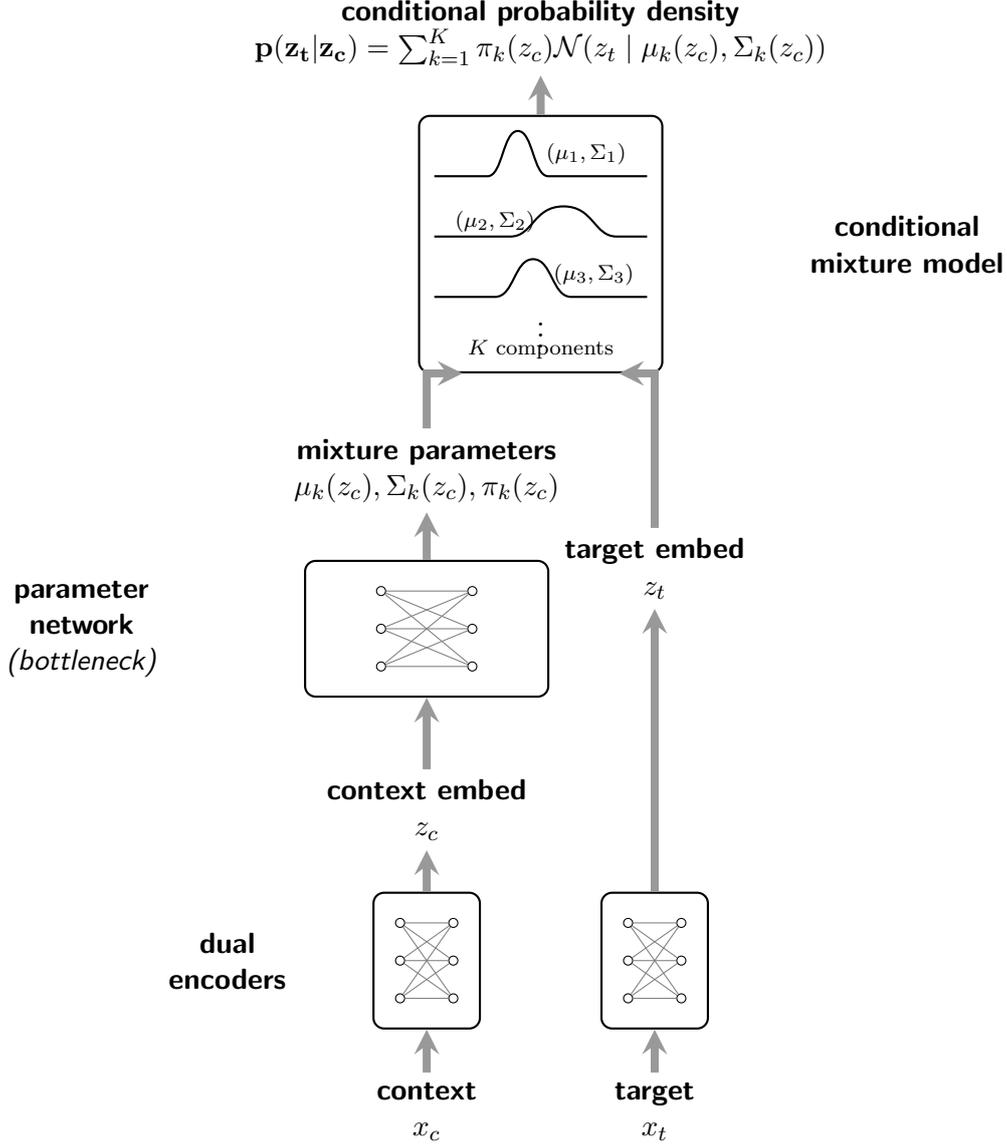

The GMJE-MDN architecture provides powerful, dynamic conditional routing. However, \textit{both the prototypical EM approach and the MDN approach inherently rely on a fixed, pre-defined $K$}. In many real-world scenarios, the true number of semantic modes is unknown, and the underlying data manifold features highly irregular, non-linear ridges. To resolve this structural rigidity and to dynamically discover both $K$ and the manifold's topology, we introduce a growing graph-based approach.

\subsection{GMJE-GNG: Dynamic Prototype Discovery via Growing Neural Gas} \label{subsec:GMJE_GNG}

While the mixture components in the parametric GMJE objective (i.e. the joint context-target embedding distribution, Eq.\ref{eq:GMJE_prototypical_loss}) can be optimized through relaxed EM (Eq.\ref{eq:optimize_GMM_components_with_uniform_covariance}), it relies on a critical structural assumption: \textit{the number of mixture modes, $K$, must be pre-defined}. Further, standard gradient-based optimization or $k$-means initializations are notoriously susceptible to local minima, particularly when the underlying joint distribution features highly overlapping or elongated geometric structures \cite{bishop2006pattern, jain2010data}.

To resolve the sensitivity to initialization and dynamically discover the optimal number of mixture components, we propose an topological initialization and routing mechanism for learning the GMM components based on \textit{Growing Neural Gas} (GNG) \cite{fritzke1995growing}. Unlike standard clustering, GNG is a \textit{topology-preserving}, \textit{density-matching} algorithm. It treats the GMJE global prototypes as nodes in a dynamic graph. Rather than forcing a predetermined $K$, GNG organically spawns new prototypes in regions of high latent error, mapping both the modes and the connecting "ridges" (manifold topology) of the joint representation space.

\subsubsection{Mathematical Formulation}
Let the global prototypes be represented as a set of vertices $\mathcal{V}$ with associated reference vectors (our joint means) $\mu_v \in \mathbb{R}^{2d}$, connected by a set of unweighted edges $\mathcal{E}$. The network is initialized with simply two nodes, $\mathcal{V} = \{v_a, v_b\}$, initialized at random data points. 

During the forward pass, we sample a joint embedding $Z = [z_c^T, z_t^T]^T$ from the data distribution. We locate the two nearest prototypes (the "winner" $s_1$ and "runner-up" $s_2$) measured by the squared Euclidean distance:
\begin{equation} \label{eq:GNG_search_neighbours}
\begin{aligned}
    s_1 &= \arg\min_{v \in \mathcal{V}} \| Z - \mu_v \|^2 \\
    s_2 &= \arg\min_{v \in \mathcal{V} \setminus \{s_1\}} \| Z - \mu_v \|^2
\end{aligned}
\end{equation}

Instead of a harsh winner-takes-all update (which causes local minima), GNG applies a soft topological update. Importantly, rather than relying on a rigid spatial radius to define this neighborhood, GNG determines proximity strictly through the active topological connections (the edge set $\mathcal{E}$). We directly move the winner $s_1$ and all its direct topological neighbors $\mathcal{N}(s_1) = \{n \in \mathcal{V} \mid (s_1, n) \in \mathcal{E}\}$ towards the observed embedding $Z$ using learning rates $\epsilon_b$ and $\epsilon_n$ (where $\epsilon_b \gg \epsilon_n$):
\begin{equation} \label{eq:GNG_prototype_update}
\begin{aligned}
    \mu_{s_1} &\leftarrow \mu_{s_1} + \epsilon_b (Z - \mu_{s_1}) \\
    \mu_n &\leftarrow \mu_n + \epsilon_n (Z - \mu_n), \quad \forall n \in \mathcal{N}(s_1)
\end{aligned}
\end{equation}

To dynamically grow the mixture model, the winner accumulates a local quantization error $E_{s_1} \leftarrow E_{s_1} + \|Z - \mu_{s_1}\|^2$. After every $\lambda$ iterations, the algorithm identifies the prototype $u$ with the maximum accumulated error, and its neighbor $v$ with the highest error. A new prototype mode $w$ is organically spawned precisely between them to alleviate the density mismatch:
\begin{equation}
    \mu_w = \frac{\mu_u + \mu_v}{2}
\end{equation}
The edges $(u,v)$ are removed, and new edges $(u,w)$ and $(v,w)$ are created.

\subsubsection{The GNG Algorithm for GMJE Component Learning}
By sequentially feeding the joint embeddings into this dynamic graph, the continuous data manifold is optimally discretized. The final set of vertices directly yields our $K$ optimal global prototypes $\mu_k$. This explicit procedure is detailed in Algo.\ref{alg:gmje_gng}

\begin{algorithm}[H]
\caption{GMJE-GNG: Dynamic GMJE Component Discovery via GNG}
\label{alg:gmje_gng}
\begin{algorithmic}[1]
\REQUIRE Stream of joint embeddings $Z$, growth interval $\lambda$, max nodes $K_{max}$, learning rates $\epsilon_b, \epsilon_n$, max edge age $a_{max}$, local error decay $\alpha$, global error decay $\beta$.
\STATE Initialize vertices $\mathcal{V} = \{v_1, v_2\}$ with random $Z$, and edges $\mathcal{E} = \{(v_1, v_2)\}$.
\STATE Initialize errors $E \leftarrow \mathbf{0}$.
\FOR{each sampled $Z$ in training step $t$}
    \STATE Find nearest prototype $s_1$ and second nearest $s_2$ to $Z$.
    \STATE Update winner error: $E_{s_1} \leftarrow E_{s_1} + \|Z - \mu_{s_1}\|^2$.
    \STATE Pull $s_1$ towards $Z$: $\mu_{s_1} \leftarrow \mu_{s_1} + \epsilon_b(Z - \mu_{s_1})$.
    \STATE Pull topological neighbors towards $Z$: $\mu_n \leftarrow \mu_n + \epsilon_n(Z - \mu_n)$ for $n \in \mathcal{N}(s_1)$.
    \STATE Create or refresh edge $(s_1, s_2)$ by setting its age to $0$.
    \STATE Increment age of all edges connected to $s_1$.
    \STATE Remove edges with age $> a_{max}$. If any vertices have no remaining edges, remove them from $\mathcal{V}$.
    \IF{$t \pmod \lambda == 0$ \AND $|\mathcal{V}| < K_{max}$}
        \STATE Find $u = \arg\max_{v \in \mathcal{V}} E_v$.
        \STATE Find neighbor $v = \arg\max_{n \in \mathcal{N}(u)} E_n$.
        \STATE Insert new prototype $w$ with $\mu_w = \frac{1}{2}(\mu_u + \mu_v)$.
        \STATE Update graph: $\mathcal{V} \leftarrow \mathcal{V} \cup \{w\}$, add edges $(u,w), (v,w)$, remove $(u,v)$.
        \STATE Decay local error: $E_u \leftarrow \alpha E_u$, $E_v \leftarrow \alpha E_v$, $E_w \leftarrow E_u$.
    \ENDIF
    \STATE Decay all global errors: $E_v \leftarrow \beta E_v$ for all $v \in \mathcal{V}$.
\ENDFOR
\STATE \textbf{Return} Extracted prototypes $\mu_k = \mu_v$ for $v \in \mathcal{V}$.
\end{algorithmic}
\end{algorithm}

Note that, we incorporate an edge aging mechanism (Lines 9-10) to accurately map dynamically changing topologies. Because the prototypes $\mu_k$ are continuously updated and move through the latent space, initial edges may eventually span empty regions between divergent clusters. By incrementing the age of inactive edges and removing those that exceed $a_{max}$, the graph organically severs obsolete topological connections. This severing mechanism is critical, as it allows the algorithm to natively break apart and isolate entirely disjoint semantic clusters without enforcing artificial connectivity across the data manifold.

\textbf{Complexity.} Computationally, the per-sample time complexity of Algo.\ref{alg:gmje_gng} is dominated by the search for the two nearest prototypes (Eq.\ref{eq:GNG_search_neighbours} and Line 4) and the topological neighbor updates (Line 7). Let $D = 2d$ denote the dimensionality of the joint embedding $(z_c,z_t)$. Both operations scale \textit{linearly} with the maximum number of components, yielding an asymptotic time complexity of $\mathcal{O}(K_{max} \cdot D)$ per step. Memory-wise, the algorithm is exceptionally lightweight; it only requires storing the $K_{max}$ prototype vectors and an adjacency matrix to track edge ages, resulting in a space complexity of $\mathcal{O}(K_{max} \cdot D + K_{max}^2)$. Because $K_{max}$ (the number of semantic concepts) is fundamentally bounded and typically orders of magnitude smaller than the massive memory banks (whose size is often denoted as $M$, discussed next in Section.\ref{sec:GMJE_Contrastive}) required by non-parametric contrastive frameworks, this mechanism adds negligible overhead compared to their heavy $\mathcal{O}(M \cdot d)$ memory footprint.

The methods discussed thus far comprise the \textbf{parametric branch} of GMJE, focusing on learning a highly compressed, finite set of structural modes (whether globally static, conditionally dynamic, or topologically grown). But what if we discard the concept of a pre-specified or estimated $K$ entirely, and instead treat every single observed data point as its own mixture component? This question bridges GMJE directly into the realm of \textit{non-parametric density estimation}, with connection to modern \textit{Contrastive Learning}.

\subsection{Data as Mode GMJE (DaM-GMJE): Contrastive Learning as Non-Parametric GMJE} \label{sec:GMJE_Contrastive}
The GMJE framework provides a unified mathematical perspective on modern representation learning; instead of learning $K$ fixed global prototypes, suppose we utilize a non-parametric formulation where we treat every previously encoded data pair in a memory bank\footnote{The memory bank is essentially a dictionary storing the target embeddings for all recent images.} (or current batch) of size $M$, denoted $\{ (z_c^{(m)}, z_t^{(m)}) \}_{m=1}^M$, as an \textit{equiprobable} mode\footnote{Note that, adding a new data point (i.e. treating a sample as a new mode) increases the number of mixture components $M$, but it does not alter the dimensionality of the underlying probability distribution. The GMM continues to operate strictly within the fixed $2d$-dimensional joint continuous space, consistent with the primal, feature space view.}.

This is equivalent to performing \textit{Kernel Density Estimation}\footnote{KDE is a \textit{non-parametric mixture model}. It places a Gaussian bump on every single data point and sums their probabilities: $p(z) \propto \Sigma \mathcal{N}(z|z_i,\sigma^2)$. Because it is a sum of Gaussians, it is multi-modal.} (KDE \cite{Hardle1991KDE}) on the joint distribution. We define the joint distribution as a \textit{uniform} mixture of $M$ joint Gaussians \textit{centered precisely on these data points}:
\begin{equation} \label{eq:GMJE_instance_based_joint}
    p(z_c, z_t) = \frac{1}{M} \sum_{m=1}^M \mathcal{N} \left( \begin{bmatrix} z_c \\ z_t \end{bmatrix} \middle| \begin{bmatrix} z_c^{(m)} \\ z_t^{(m)} \end{bmatrix}, \Sigma \right)
\end{equation}

We note the connection of this approach to standard Contrastive Learning. Contrastive methods inherently assume that, \textit{conditional on being generated by a specific instance $m$}, the augmentations (the positive pairs, e.g. two crops of the same image) forming the context and target views are statistically independent and uniformly spherical \cite{wu2018unsupervisedfeaturelearningnonparametric,hjelm2019learningdeeprepresentationsmutual,Bachman2019learning,He2020MoCo,Chen2020simCLR,wang2020understanding}.
Mathematically, this corresponds to restricting the component covariance $\Sigma$ to an isotropic, block-diagonal matrix governed by a scalar variance (temperature) hyperparameter $\tau$:
\begin{equation*}
    \Sigma = \begin{bmatrix} \tau I & \mathbf{0} \\ \mathbf{0} & \tau I \end{bmatrix}
\end{equation*}
Because the off-diagonal cross-covariance blocks are strictly zero, the joint Gaussian component factorizes exactly into the product of two independent marginals:
\begin{equation}
    p(z_c, z_t) = p(z_c) \times p(z_t) = \frac{1}{M} \sum_{m=1}^M \mathcal{N}(z_c \mid z_c^{(m)}, \tau I) \mathcal{N}(z_t \mid z_t^{(m)}, \tau I)
\end{equation}

Again, if we wish to evaluate the conditional NLL, $-\log p(z_t \mid z_c)$, of a true matching pair $(z_c, z_t)$ under this non-parametric mixture, we first apply Bayes' theorem. By integrating the joint distribution over $z_t$, the marginal distribution trivially evaluates to $p(z_c) = \frac{1}{M} \sum_{j=1}^M \mathcal{N}(z_c \mid z_c^{(j)}, \tau I)$. The conditional distribution is thus:
\begin{equation} \label{eq:GMJE_KDE_cond}
    p(z_t \mid z_c) = \frac{p(z_c, z_t)}{p(z_c)} = \frac{\sum_{m=1}^M \mathcal{N}(z_c \mid z_c^{(m)}, \tau I) \mathcal{N}(z_t \mid z_t^{(m)}, \tau I)}{\sum_{j=1}^M \mathcal{N}(z_c \mid z_c^{(j)}, \tau I)}
\end{equation}

To simplify this into the standard \textit{contrastive} format, we apply two structural assumptions. \textit{First, we assume all embeddings are $L^2$-normalized} ($\|z\|^2 = 1$). Expanding the squared Euclidean distance inside the isotropic Gaussian density yields:
\begin{align} \label{eq:GMJE_L2_expansion}
    \mathcal{N}(z_c \mid z_c^{(m)}, \tau I) &= \frac{1}{(2\pi\tau)^{d/2}} \exp \left( - \frac{1}{2\tau} \| z_c - z_c^{(m)} \|^2 \right) \nonumber \\
    &= \frac{1}{(2\pi\tau)^{d/2}} \exp \left( - \frac{1}{2\tau} (\|z_c\|^2 - 2z_c^T z_c^{(m)} + \|z_c^{(m)}\|^2) \right) \nonumber \\
    &= \underbrace{\frac{e^{-1/\tau}}{(2\pi\tau)^{d/2}}}_{C_0} \exp\left( \frac{z_c^T z_c^{(m)}}{\tau} \right)
\end{align}

\textit{Second}, in instance-level discrimination\footnote{In standard instance-level discrimination (like SimCLR \cite{Chen2020simCLR} or MoCo \cite{He2020MoCo}), $z_c^{(m)}$ and $z_t^{(m)}$ are just two different augmented views of the exact same underlying image (instance $m$). The entire contrastive objective is explicitly designed to make the network invariant to these augmentations, pulling these two vectors together. Over the course of training, they naturally converge to represent roughly the exact same point on the hypersphere.}, the context prototype for mode $m$ is effectively represented by its paired target view in the memory bank\footnote{In practical implementations like MoCo, we do not store both a context bank and a target bank. We only cache the target embeddings (the "keys" from the momentum encoder) in a single memory queue. When computing the denominator (the negative samples), the anchor $z_c$ is compared directly against these stored target keys $z_t^{(j)}$. Therefore, the stored target key $z_t^{(m)}$ is forced to act as the single proxy prototype for the entire instance.}, implying $z_c^{(m)} \approx z_t^{(m)}$. Substituting this and the expanded Gaussian density (Eq.\ref{eq:GMJE_L2_expansion}) into the denominator of Eq.\ref{eq:GMJE_KDE_cond}, the marginal distribution simplifies to $p(z_c) = \frac{C_0}{M} \sum_{j=1}^M \exp(z_c^T z_t^{(j)} / \tau)$.

For the numerator in Eq.\ref{eq:GMJE_KDE_cond}, we evaluate the joint probability $p(z_c, z_t) = \frac{1}{M} \sum_{m=1}^M \mathcal{N}(z_c \mid z_c^{(m)}, \tau I) \mathcal{N}(z_t \mid z_t^{(m)}, \tau I)$. 
Suppose the continuous observation $(z_c, z_t)$ is a true matching pair generated by mode $m^*$. By definition\footnote{In the context of instance-level discrimination, the current observation $(z_c, z_t)$ represents the embeddings of two augmented views from a specific image (instance $m^*$) in the current forward pass. Because $z_t$ is exactly the target representation generated for this instance, it is mathematically identical to the target key $z_t^{(m^*)}$ stored in the memory bank for that specific slot. Consequently, the distance $\|z_t - z_t^{(m^*)}\|^2$ is strictly zero, causing the Gaussian density to evaluate at its absolute peak, which elegantly simplifies the numerator.}, its target embedding $z_t$ perfectly coincides with its own memory bank entry $z_t^{(m^*)}$. The target Gaussian for this specific matching mode evaluates exactly at its peak density, as the distance is zero:
\begin{equation*}
    \mathcal{N}(z_t \mid z_t^{(m^*)}, \tau I) = \mathcal{N}(z_t \mid z_t, \tau I) = \frac{1}{(2\pi\tau)^{d/2}} \exp(0) = (2\pi\tau)^{-d/2} \equiv C_1
\end{equation*}

Conversely, for all non-matching modes $m \neq m^*$, the distance $\|z_t - z_t^{(m)}\|^2 \gg 0$. Under the standard contrastive assumption where the variance $\tau$ is extremely small (the "low temperature" limit), the exponential term $\exp(-\|z_t - z_t^{(m)}\|^2 / 2\tau)$ decays to practically zero. Consequently, the summation over the $M$ joint components is overwhelmingly dominated by the single matching mode $m^*$, allowing us to drop the remaining terms:
\begin{equation*}
    p(z_c, z_t) \approx \frac{1}{M} \mathcal{N}(z_c \mid z_c^{(m^*)}, \tau I) \mathcal{N}(z_t \mid z_t^{(m^*)}, \tau I) = \frac{1}{M} \mathcal{N}(z_c \mid z_c^{(m^*)}, \tau I) \cdot C_1
\end{equation*}

Applying our previous $L^2$-expansion (Eq.\ref{eq:GMJE_L2_expansion}) to the context Gaussian, and substituting our proxy assumption $z_c^{(m^*)} \approx z_t^{(m^*)} = z_t$, the density for the context view evaluates to $C_0 \exp(z_c^T z_t / \tau)$. 

Substituting this simplified numerator, alongside the denominator we derived for $p(z_c)$, back into Bayes' theorem (Eq.\ref{eq:GMJE_KDE_cond}) yields the final conditional probability:
\begin{equation*}
    p(z_t \mid z_c) = \frac{p(z_c, z_t)}{p(z_c)} \approx \frac{\frac{1}{M} C_0 \exp(z_c^T z_t / \tau) \cdot C_1}{\frac{1}{M} C_0 \sum_{j=1}^M \exp(z_c^T z_t^{(j)} / \tau)} = C_1 \frac{\exp(z_c^T z_t / \tau)}{\sum_{m=1}^M \exp(z_c^T z_t^{(m)} / \tau)}
\end{equation*}

Taking the negative natural logarithm, the scaling constants separate cleanly into an additive constant $C = -\log C_1$, and the log-likelihood factorizes into:
\begin{equation} \label{eq:InfoNCE_bridge}
    -\log p(z_t \mid z_c) = - \log \frac{\exp(z_c^T z_t / \tau)}{\sum_{m=1}^M \exp(z_c^T z_t^{(m)} / \tau)} + C
\end{equation}
Remarkably, Eq.\ref{eq:InfoNCE_bridge} is precisely the \textbf{InfoNCE loss} \cite{oord2019InfoNCE} heavily utilized in state-of-the-art Contrastive Learning frameworks (e.g. SimCLR \cite{Chen2020simCLR}, MoCo \cite{He2020MoCo}). This connection between contrastive learning and our GMJE constitutes a major theoretical result of this work: \textit{standard Contrastive Learning is mathematically proven to be a special, degenerate case of the GMJE framework}. Specifically, it is estimating a density using a \textit{non-parametric}, \textit{isotropic} GMM.

Establishing this mathematical equivalence required two structural assumptions which reveal the theoretical shortcomings of standard contrastive methods and articulates the advantages of our fully generalized GMJE approach:
\begin{itemize}[label=-]
    \item \textit{The Isotropic, Independent Covariance Assumption:} by restricting the covariance $\Sigma$ of the instance-based joint distribution (Eq.\ref{eq:GMJE_instance_based_joint}) to an isotropic block-diagonal matrix ($\tau I$), contrastive methods inherently assume all latent dimensions vary equally and independently. More critically, they force the cross-covariance block ($\Sigma_{ct}$) to strictly zero. Mathematically, setting $\Sigma_{ct} = \mathbf{0}$ breaks all relational links between the context and target spaces, partitioning them into disconnected universes. To force alignment despite this mathematical flaw, contrastive frameworks must rely on a brute-force architectural workaround: they mandate strict \textit{Siamese weight sharing} ($\theta = \theta'$, or slow copying) and define positive pairs via dot-product maximization (Eq.\ref{eq:InfoNCE_bridge}), which physically forces the representations to occupy the exact same coordinate on the unit hypersphere ($z_c \approx z_t$). Consequently, contrastive learning acts strictly as an \textit{invariance learner}: it minimizes loss by systematically erasing any dynamic features that make the two views distinct (e.g. erasing the semantic difference between a dog "sitting" vs "running"). In contrast, by evaluating a full covariance structure, GMJE preserves the cross-covariance block ($\Sigma_{ct}$) which natively models the transformation rules between the two states. This allows $z_c$ and $z_t$ to reside at distinct coordinates (preserving rich, view-specific descriptive features) while flawlessly aligning them via the learned cross-view mutual information.
    \item \textit{The Non-Parametric, Instance-Level Assumption:} by treating every individual data point as an independent mode (the non-parametric KDE formulation), contrastive learning ignores global semantic clustering. Because the objective operates strictly at the instance level without any semantic awareness, it inherently suffers from the ``class collision'' (or ``false negative'') problem \cite{khosla2020supervised, chuang2020debiased}. In standard instance-level discrimination, any two distinct images in a batch or memory queue are mathematically forced to be negative samples and are repelled in the latent space, even if they represent the exact same underlying concept (e.g. two different images of a golden retriever). This forces the network to expend representational capacity memorizing arbitrary, instance-level granularities to keep semantically identical images apart, rather than learning generalized class features. Further, accurately estimating the marginal distribution (the denominator) using an unweighted KDE requires an enormous number of modes $M$, necessitating computationally expensive, massive memory banks. Transitioning to the \textit{parametric} GMJE (Sections.\ref{subsec:GMJE_reduced_EM} to \ref{subsec:GMJE_GNG}) resolves these class collision and memory bottlenecks by replacing the $M$ instances with $K \ll M$ learnable semantic prototypes (i.e. the parametric GMJE approach).
\end{itemize}

\paragraph{Empirical Optimization and the Mahalanobis Trace Trap.}
While the theoretical advantages of a full-covariance generative objective are profound, realizing them empirically introduces a critical optimization vulnerability. As detailed in Appendix \ref{app:mahalanobis_trap}, attempting to optimize the pure joint distribution using a dynamically computed empirical batch covariance triggers the "\textit{Mahalanobis Trace Trap}". This algebraic cancellation mathematically eradicates the exact cross-view pulling force ($C_{ct}^{-1}$) we wish to preserve, causing representations to disconnect and dimensionally collapse. Consequently, the three advanced formulations proposed in this work, i.e. Parametric Decoupling via Prototypes (Section \ref{subsec:GMJE_reduced_EM}), Conditional Factorization via MDN (Section \ref{subsec:GMJE_MDN}), and Dynamic Non-Parametric Density via SMC (Section \ref{sec:SMC_Memory_Bank}), are not merely alternative design choices; they are carefully engineered structural remedies designed specifically to break this algebraic cancellation, safely restoring the generative pulling force while explicitly maximizing entropy.

While the fully parametric GMJE framework bypasses the limitations of instance discrimination entirely, one may still wish to leverage the non-parametric $M$-mode formulation typical of contrastive learning. However, our theoretical bridge exposes a fundamental operational flaw in this standard setup: it relies on a naive, uniform, un-informative prior FIFO queue to manage its massive number of modes. To resolve this uniform-prior limitation, the rigid queue must be replaced by a dynamic weighting mechanism to mitigate redundant repelling. We address this next by transforming the memory bank into a probability-weighted particle system via \textit{Sequential Monte Carlo}.

\subsection{Sequential Monte Carlo (SMC) for Dynamic Memory Bank Optimization} \label{sec:SMC_Memory_Bank}

In Section.\ref{sec:GMJE_Contrastive}, we established a profound theoretical equivalence: standard contrastive learning is mathematically equivalent to estimating a density using a \textit{Non-Parametric Gaussian Mixture Model}. Under this non-parametric GMJE framework, we can mathematically view every single one of the $M$ negative keys in a contrastive memory bank as the center ($\mu_m$) of a tiny, fixed-variance Gaussian component. The memory bank itself \textit{is} the Gaussian Mixture Model representing the marginal distribution $p(z_c)$.

Architecturally, frameworks such as MoCo \cite{He2020MoCo} implement this memory bank via a \textit{First-In-First-Out} (FIFO) queue to manage representation staleness. However, from a probabilistic standpoint, this FIFO queue implicitly assumes a \textit{rigid, uniform prior}: it assumes every single component (key) in the mixture has the exact same mixing weight ($\pi_m = \frac{1}{M}$). When a new batch of embeddings is generated, the oldest embeddings are deterministically discarded regardless of their informational value. 

This unweighted FIFO mechanism is theoretically suboptimal. It frequently discards highly informative representations (e.g. rare classes or "hard negatives") simply because they are temporally old, while simultaneously retaining redundant, uninformative samples (e.g. continuous streams of easy background images). To resolve this, we propose abandoning the uniform prior and formulating the memory bank as a dynamic, probability-weighted particle system using \textit{Sequential Monte Carlo} (SMC) \cite{doucet2001sequential,chopin2020SMC}, commonly known as \textit{Particle Filtering}.

If the memory bank is a non-parametric GMM, forcing a uniform prior ($\pi_m = \frac{1}{M}$) via a FIFO queue is mathematically restrictive and suboptimal. SMC perfectly instantiates the GMJE solution by relaxing this constraint:
\begin{itemize}[label=-]
    \item \textit{The Particles are the Mixture Components:} in SMC, the particles are not merely data points; they are explicitly the $M$ components of the non-parametric GMM.
    \item \textit{Dynamically Updating the Weights:} instead of a uniform prior, SMC mathematically computes the true posterior mixing weights ($\pi_m$) of those components using their likelihoods under the current batch.
    \item \textit{Resampling:} it dynamically reallocates the mixture components to cluster densely around the hardest, most complex parts of the data manifold, preventing redundant uniform coverage.
\end{itemize}

\subsubsection{The Memory Bank as a Particle System}
Under the SMC framework, we treat the $M$ stored embeddings as a set of discrete particles $\{ Z^{(m)} \}_{m=1}^M = \{ (z_c^{(m)}, z_t^{(m)}) \}_{m=1}^M$ used to approximate the continuous marginal distribution of the latent space. Importantly, each particle is assigned a dynamic importance weight $W^{(m)}$, constrained such that $\sum_{m=1}^M W^{(m)} = 1$. 

The instance-based joint distribution (Eq.\ref{eq:GMJE_instance_based_joint}) is thus generalized into a weighted Kernel Density Estimate:
\begin{equation} \label{eq:GMJE_SMC_joint}
    p(z_c, z_t) = \sum_{m=1}^M W^{(m)} \mathcal{N} \left( \begin{bmatrix} z_c \\ z_t \end{bmatrix} \middle| \begin{bmatrix} z_c^{(m)} \\ z_t^{(m)} \end{bmatrix}, \Sigma \right)
\end{equation}

By carrying these weights through the exact Bayes' theorem derivation previously established, the conditional contrastive objective (Eq.\ref{eq:InfoNCE_bridge}) elegantly upgrades into an \textit{SMC-weighted InfoNCE loss}:
\begin{equation} \label{eq:SMC_InfoNCE}
    -\log p(z_t \mid z_c) = - \log \frac{\exp(z_c^T z_t / \tau)}{\sum_{m=1}^M W^{(m)} \exp(z_c^T z_t^{(m)} / \tau)} + C
\end{equation}

\subsubsection{Sequential Importance Weighting and Resampling}
To dynamically maintain the optimal set of particles, the SMC memory bank updates via two principled steps during each training iteration $t$:

\textbf{1. Importance Weight Update:} as a new batch of query embeddings $Z_{query}$ arrives, we evaluate \textit{how well our existing particles explain this new data}. Particles that consistently yield high likelihoods under the current queries (i.e. those that act as "hard negatives" by lying densely close to the current manifold) should be amplified. The unnormalized weight $\tilde{W}_t^{(m)}$ of each particle is updated recursively based on its similarity to the current batch:
\begin{equation} \label{eq:SMC_GMJE_weights_update}
    \tilde{W}_t^{(m)} = W_{t-1}^{(m)} \times \mathbb{E}_{z_c \in Z_{query}} \left[ \exp(z_c^T z_t^{(m)} / \tau) \right]
\end{equation}
the weights are then normalized to sum to one: $W_t^{(m)} = \tilde{W}_t^{(m)} / \sum_{j} \tilde{W}_t^{(j)}$.

\textbf{2. Sequential Importance Resampling (Replacing FIFO):} In a standard Sequential Monte Carlo particle filter, the system will inevitably suffer from weight degeneracy over time, where a few highly relevant particles accumulate all the probability mass. To prevent this, we execute a \textit{resampling} step at every iteration, effectively replacing the naive deterministic FIFO drop mechanism:
\begin{itemize}[label=\textendash]
    \item We form a combined pool of size $M + B$ (the existing memory bank plus the incoming batch).
    \item We sample exactly $M$ particles from this combined pool with replacement, drawing with probabilities proportional to their normalized posterior weights $W_t^{(m)}$.
    \item Highly weighted particles (valuable hard negatives) are naturally retained or even duplicated (cloned), while particles with vanishing weights (redundant or useless easy negatives) are probabilistically discarded.
    \item All weights in the newly sampled memory bank of size $M$ are then reset to a uniform $1/M$ to begin the next cycle.
\end{itemize}
To monitor the health of this dynamic sampling process, we track the Effective Sample Size ($\text{ESS} = 1 / \sum (W^{(m)})^2$) over the combined pool before resampling. A dropping ESS correctly indicates that the filter is aggressively concentrating its capacity on dense, hard-negative regions of the manifold rather than maintaining uniform coverage.

\textbf{Dual Optimization Dynamics.} We have dual processes here: optimizing the neural network parameters of the two encoders and maintaining the memory bank. The SMC weight update and resampling operations are purely algebraic, gradient-free processes (Eq.\ref{eq:SMC_GMJE_weights_update}). Simultaneously, the encoder networks are optimized via standard backpropagation using the SMC-weighted InfoNCE loss (Eq.\ref{eq:SMC_InfoNCE}), wherein the particle weights $W^{(m)}$ act as constant scalars scaling the contrastive pushing force. Therefore, \textit{the InfoNCE Loss + Backprop creates better representations, while the SMC Update creates a better memory bank to contrast those representations against.}

\textbf{Theoretical Advantages.} Integrating SMC provides a mathematically principled solution to "hard negative mining" - a practice that typically relies on messy, ad-hoc heuristics in standard contrastive learning \cite{robinson2020contrastive, kalantidis2020hard}. Because hard negatives inherently yield higher likelihoods under the current batch's distribution, they organically accumulate higher importance weights $W^{(m)}$, dominating the denominator of Eq.\ref{eq:SMC_InfoNCE} and providing stronger, more informative gradient signals. Further, by selectively retaining high-value representations across many epochs, SMC dramatically increases the informational density of the memory bank, potentially allowing for much smaller queue sizes $M$ without sacrificing representation quality.

\subsubsection{Algorithm: SMC-GMJE for Dynamic Memory Bank Optimization}
The complete procedure for dynamically updating the particle-based memory bank during contrastive training is detailed in Algo.\ref{alg:smc_gmje}. By caching the dot-products already calculated during the forward pass, the weight update adds negligible overhead to standard instance-level discrimination.

\begin{algorithm}[H]
\caption{SMC-GMJE: Dynamic Memory Bank Optimization via SMC}
\label{alg:smc_gmje}
\begin{algorithmic}[1]
\REQUIRE Stream of query batches $Z_{query} = \{(z_c, z_t)\}_{b=1}^B$ generated by encoders, memory bank size $M$, temperature $\tau$.
\STATE Initialize memory bank $\mathcal{M}$ with $M$ random target embeddings and uniform weights $W \leftarrow 1/M$.
\FOR{each training step $t$ with incoming batch $Z_{query}$}
    
    \STATE \textbf{// Loop A: Encoders Optimization (Requires Gradients)}
    \STATE \textbf{Forward Pass:} Compute SMC-weighted InfoNCE loss $\mathcal{L}$ (Eq.\ref{eq:SMC_InfoNCE}) using the current memory bank $\mathcal{M}$ and its weights $W$.
    \STATE \textbf{Backward Pass:} Compute gradients $\nabla \mathcal{L}$ and update both encoder networks' parameters via gradient descent.
    
    \STATE
    \STATE \textbf{// Loop B: SMC Memory Bank Optimization (Gradient-Free)}
    \STATE \textbf{Combine:} Pool $\mathcal{M}_{pool} \leftarrow \mathcal{M} \cup \{z_t \in Z_{query}\}$, size $M+B$.
    \STATE \textbf{Prior Weights:} $W_{prior}^{(m)} \leftarrow \frac{1}{M+B}$ for new batch, scale old weights by $\frac{M}{M+B}$.
    
    \STATE \textbf{Importance Update:} For each particle $m \in \{1 \dots M+B\}$:
    \STATE \quad $\tilde{W}_t^{(m)} \leftarrow W_{prior}^{(m)} \times \frac{1}{B} \sum_{z_c \in Z_{query}} \exp(z_c^T z_t^{(m)} / \tau)$
    \STATE \textbf{Normalize:} $W_{pool}^{(m)} \leftarrow \tilde{W}_t^{(m)} / \sum_{j=1}^{M+B} \tilde{W}_t^{(j)}$
    
    \STATE \textit{// Execute Sequential Importance Resampling}
    \STATE Sample $M$ particles from $\mathcal{M}_{pool}$ with replacement, proportional to $W_{pool}^{(m)}$.
    \STATE Replace $\mathcal{M}$ with these sampled particles.
    \STATE Reset all memory bank weights to uniform: $W \leftarrow 1/M$.
\ENDFOR
\end{algorithmic}
\vspace{1ex}
{\footnotesize \textbf{Algorithm Note:} as we execute Sequential Importance Resampling (SIR) at every step, the probability mass is perfectly transferred into the physical multiplicity (cloning) of the particles, allowing the memory bank weights to be reset to $1/M$. Consequently, during the forward pass, standard unweighted InfoNCE (Cross-Entropy) over the resampled bank is mathematically exact and strictly equivalent to evaluating the explicit SMC-weighted loss in Eq.\ref{eq:SMC_InfoNCE}.\par}
\end{algorithm}

\textbf{Complexity.} Computationally, the time complexity of the SMC-GMJE memory update is strictly bounded. The importance weight update (Line 12) requires the dot products between the $B$ context queries and the $M+B$ pooled memory bank particles. However, because these exact dot products $\exp(z_c^T z_t^{(m)} / \tau)$ are largely already computed to evaluate the InfoNCE denominator in the forward pass, they can be cached and re-used. Thus, the weight update only requires an $\mathcal{O}(M)$ summation overhead. The resampling step (Line 15), implemented via systematic resampling or the alias method, operates in linear $\mathcal{O}(M)$ time. Therefore, the overall time complexity remains $\mathcal{O}(B \cdot M \cdot d)$ per step, asymptotically identical to standard contrastive methods such as MoCo \cite{He2020MoCo}. Space complexity is similarly preserved; storing the weights adds a trivial $\mathcal{O}(M)$ scalars to the $\mathcal{O}(M \cdot d)$ memory bank footprint.

It is important to emphasize that Algo.\ref{alg:smc_gmje} is designed specifically for the non-parametric, isotropic special case of GMJE, which we established in Section.\ref{sec:GMJE_Contrastive} is mathematically equivalent to standard contrastive learning. This is evident from the mathematical formulation of the algorithm: (1) \textit{The Loss Function} (Line 5) explicitly optimizes the SMC-weighted InfoNCE loss (Eq.\ref{eq:SMC_InfoNCE}, same as used in contrastive learning); (2) \textit{The Covariance Assumption} (Line 12) utilizes the dot-product similarity $\exp(z_c^T z_t^{(m)} / \tau)$, which only emerges when assuming a strictly isotropic covariance $\Sigma = \tau I$ over $L^2$-normalized embeddings; and (3) \textit{The Components} (Line 11) iterate over discrete particles in a memory bank, treating every individual data instance as a fixed mode rather than optimizing global, learnable prototypes. Therefore, SMC-GMJE is a direct, drop-in upgrade for contrastive learning frameworks such as MoCo \cite{He2020MoCo}. It takes the standard contrastive architecture but replaces the heuristically uninformative FIFO queue with a probabilistically rigorous particle filter.

\paragraph{SMC Beyond Contrastive Learning: Generalizing SMC to general GMJE.} 
While Algo.\ref{alg:smc_gmje} specifically formulates SMC-GMJE as a non-parametric drop-in replacement for standard contrastive learning (relying on the InfoNCE objective and isotropic, fixed-variance assumptions), the SMC mechanism is theoretically universal. It is not fundamentally restricted to the special contrastive case. As detailed in Appendix.\ref{app:general_smc_gmje}, SMC can be directly generalized to optimize the full, parametric GMJE objective (Algo.\ref{alg:general_smc_gmje}). In this generalized setting, the particles are no longer restricted to isotropic spherical bounds; instead, each particle $m$ can be evaluated and dynamically re-weighted using the exact general joint NLL formulated in Eq.\ref{eq:GMJE_prototypical_loss}. Under this generalized SMC-GMJE formulation, the network organically maintains a probability-weighted particle system of full-covariance components, seamlessly bridging the gap between massive non-parametric memory banks and rich, multi-modal density estimation.

\subsection{Generative GMJE: Sampling from the Learned Latent Manifold} \label{subsec:GMJE_generative}

A direct consequence of the probabilistic GMJE framework is that, by explicitly optimizing the NLL of the embedding space, the network does not merely learn a discriminative mapping function $x \to z$; it natively learns a continuous, closed-form generative probability density function over the latent manifold\footnote{From a Bayesian perspective, any valid probability distribution over the data or latent space natively constitutes a generative model. Modern generative frameworks (e.g. GANs, VAEs, or Score-Based models \cite{huang2022classificationscorebasedgenerativemodelling}) implicitly or explicitly learn these underlying distributions or their gradient fields.}. 

Recall that the architecture explicitly models the $2d$-dimensional joint distribution of the context and target embeddings, $Z = [z_c^T, z_t^T]^T$, as a GMM (Eq.\ref{eq:GMJE_joint}):
\begin{equation*}
    p(z_c, z_t) = \sum_{k=1}^K \pi_k \mathcal{N} \left( \begin{bmatrix} z_c \\ z_t \end{bmatrix} \middle| \begin{bmatrix} \mu_{c,k} \\ \mu_{t,k} \end{bmatrix}, \begin{bmatrix} \Sigma_{cc,k} & \Sigma_{ct,k} \\ \Sigma_{tc,k} & \Sigma_{tt,k} \end{bmatrix} \right)
\end{equation*}

In standard downstream generative tasks (e.g. unconditional image synthesis or latent data augmentation), we typically require a standard $d$-dimensional representation, rather than a concatenated $2d$-dimensional relational pair. To obtain this, we must marginalize the joint distribution. By the fundamental properties of multivariate Gaussians, integrating out the context variable $z_c$ trivially yields the marginal distribution of the target representation space $p(z_t)$, which elegantly remains a closed-form GMM:
\begin{equation} \label{eq:GMJE_marginal_target_replicated2} \tag{cc.Eq.\ref{eq:GMJE_marginal_target}}
    p(z_t) = \int p(z_c, z_t) dz_c = \sum_{k=1}^K \pi_k \mathcal{N}(z_t \mid \mu_{t,k}, \Sigma_{tt,k})
\end{equation}

Because we possess this explicit, marginalized generative formulation, we can \textit{unconditionally generate novel, mathematically valid embeddings} without requiring any real input data $x$. Using e.g. Gaussian Mixture Approximation (GMA) sampling \cite{Huang2025GMA}, we can draw \textit{synthetic}, meaningful representations $z_t$ via a highly efficient, gradient-free two-step sampling procedure:
\begin{enumerate}
    \item \textit{Discrete Mode Selection:} sample a component index $k$ from the categorical distribution defined by the learned mixing weights, $k \sim \text{Categorical}(\pi_1, \dots, \pi_K)$.
    \item \textit{Continuous Feature Generation:} given the selected mode $k$, draw a continuous $d$-dimensional embedding vector $z_t$ from its corresponding marginalized Gaussian component, $z_t \sim \mathcal{N}(\mu_{t,k}, \Sigma_{tt,k})$, computed practically via the Cholesky decomposition of the target covariance matrix $\Sigma_{tt,k} = A_k A_k^T$ such that $z_t = \mu_{t,k} + A_k \epsilon$ where $\epsilon \sim \mathcal{N}(\mathbf{0}, I)$.
\end{enumerate}

In contrast, standard self-supervised frameworks (e.g. JEPA, SimCLR, SwAV) only train \textit{discriminative} encoders. Because they lack a normalized probability density function $p(z)$ and do not penalize "empty space" between representations, attempting to sample random coordinates from their latent spaces invariably yields out-of-distribution (OOD) garbage vectors\footnote{An exception is the BiJEPA architecture \cite{huang2026BiJEPA}, which learns a manifold, reshaped by rule constraints, in the embedding space via a shared predictor; it can then be used to generate e.g. new rules set.}. By contrast, the GMJE latent space is a strict generative manifold. The ability to sample directly from it opens the door to powerful downstream applications, e.g. generative replay for continual learning, latent data augmentation, and unconditional image synthesis, which we evaluate later in Section.\ref{sec:experiments}.

\vspace{1em}
\noindent \textbf{Summary of the GMJE Framework.} This concludes the mathematical formulation of the GMJE paradigm. To summarize the structural organization of this generalized framework:
\begin{itemize}[label=-]
    \item \textbf{Sections \ref{subsec:GMJE_reduced_EM} (Reduced EM), \ref{subsec:GMJE_MDN} (GMJE-MDN), and \ref{subsec:GMJE_GNG} (GMJE-GNG)} constitute the \textbf{Parametric GMJE} branch. Here, the architecture learns $K$ highly compressed, full-covariance prototypes (whether globally static, conditionally dynamic, or topologically grown) to map the semantic manifold, and does not require a memory bank. The learned GMM can be used for downstream discriminative or generative tasks.
    \item \textbf{Sections \ref{sec:GMJE_Contrastive} (the InfoNCE bridge) and \ref{sec:SMC_Memory_Bank}} (SMC-GMJE) constitute the \textbf{Non-Parametric GMJE} branch. Here, the architecture treats $M$ encoded data instances as fixed, isotropic modes, which fundamentally requires a dynamic memory bank to estimate the marginal distribution (the contrastive denominator).
\end{itemize}

\section*{\textmd{\textit{The Primal-Dual Paradigm Shift}}}

To conclude our theoretical exposition, it is helpful to contrast the two mathematical perspectives used throughout this work: the Dual (sample-space) view of GJE and the Primal (feature-space) view of GMJE. Together, these complementary perspectives provide a probabilistically grounded framework for multi-modal representation learning. To make this duality concrete, consider a batch of embeddings arranged as a matrix $Z \in \mathbb{R}^{N \times d}$, where each of the $N$ rows corresponds to an individual data sample and each of the $d$ columns corresponds to a latent feature.
\begin{itemize}[label=$\circ$]
    \item \textbf{The Dual ``Sample Space'' View (Section \ref{sec:GJE}):} the pure GJE framework operates in the dual space, evaluating the $N \times N$ Gram matrix ($Z Z^T$). This perspective prioritizes \textit{sample similarity}, measuring the relationships between individual images within a batch. Its geometric regularizer ($\frac{1}{2}\log|K_{cc}|$) maximizes \textit{batch diversity}, guarding against standard \textbf{Representation Collapse} (where all $N$ samples map to the exact same point). While theoretically elegant for continuous functional mapping, it suffers from an $\mathcal{O}(N^3)$ computational bottleneck (necessitating approximations like RFF) and structurally struggles to isolate discrete multi-modal clusters.
    \item \textbf{The Primal ``Feature Space'' View (Section \ref{sec:GMJE}):} the GMJE framework executes a classical shift into the primal space, evaluating the $d \times d$ feature covariance matrix ($Z^T Z$). This perspective prioritizes \textit{feature independence}, directly modeling the geometric shape of the data manifold. Its geometric regularizer ($\frac{1}{2}\log|\Sigma_c|$) maximizes the \textit{volume of the latent space}, specifically preventing \textbf{Dimensional Collapse} (where latent features become perfectly correlated and lazily flatten the data into a lower-dimensional line or subspace). By operating in the primal space, the computation scales optimally with batch size ($\mathcal{O}(d^3)$ rather than $\mathcal{O}(N^3)$) and natively provides the coordinate framework required to deploy $K$ discrete prototypes for multi-modal semantic routing.
\end{itemize}

\paragraph{Why Multi-Modality is Difficult to Realize in the Dual Sample Space.}
To motivate the architectural necessity of GMJE, it is helpful to examine why extending pure GJE to capture multi-modal structure directly in the dual sample space ($N \times N$) is poorly suited to stable semantic modeling. In the primal feature space ($d \times d$), the network operates in a fixed global coordinate system\footnote{Geometrically, the primal feature space is analogous to a map with fixed coordinate axes, where distinct semantic clusters can be permanently pinned to specific global locations.}, allowing $K$ semantic modes ($\mu_k$) to be represented as stable spatial regions. By contrast, the dual space is defined only through the Gram matrix ($ZZ^T$), which replaces absolute coordinates with relative pairwise similarities\footnote{Conversely, the dual sample space is analogous to a dynamic similarity table recording only relative affinities between entities, without preserving their absolute global coordinates.}. 

This creates several structural difficulties for mixture modeling. First, without explicit feature-space axes, any cluster prototype must be represented indirectly, typically as a linear combination of the current data points. Second, in stochastic self-supervised learning, the reference samples in a mini-batch change at every optimization step\footnote{For example, if one Gram matrix is formed from similarities among a dog, a car, and a boat, while the next is formed from a plane, a bird, and a frog, then the underlying reference system changes completely across iterations.}, making it difficult to learn persistent global semantic modes across batches and epochs. Third, a mixture model defined directly over the Gram matrix does not naturally correspond to clustering points into distinct spatial semantic categories; instead, it tends to model multiple \textit{similarity-generating patterns} (i.e., predicting that the similarity between two images was drawn from one of $K$ different distributions). 

For this reason, the transition to GMJE, which operates through the $d \times d$ feature covariance in primal space, is not merely a computational preference but a more natural and principled formulation for representing stable multi-modal semantic structure.

\section[Experiments]{Experiments\protect\footnote{The experimental Python codes were partially enabled with the kind assistance of Gemini 3.0 \cite{google2026gemini3docs}, for which the author acknowledges.}} \label{sec:experiments}

\subsection{Ambiguous Alignment of Synthetic Embedding Representations} \label{subsec:exp_synthetic}

\paragraph{Motivation: multi-modal inverse prediction.}
A classical difficulty in regression arises when the underlying mapping is \emph{multi-valued}. Bishop \cite{bishop1994mixture} illustrated this using the inverse kinematics of a two-joint robot arm: while the forward problem is deterministic, the inverse problem is ambiguous, since multiple joint configurations can produce the same end-point coordinate. When such a problem is trained with squared loss, a deterministic neural predictor converges to the conditional mean\footnote{A predictor trained with squared loss converges to the conditional mean $\mathbb{E}[y \mid x]$ under standard risk minimization; see Appendix.~\ref{app:mse_derivation}.} rather than to the distinct valid solutions. In multi-modal settings, this average may lie between modes, in regions unsupported by the true data distribution.

The same issue appears naturally in self-supervised representation learning. Predicting a target view from a heavily degraded or partially observed context view is often an inverse problem: a single context may be compatible with several plausible target completions. For example, if $x_c$ contains only a partial view of a dog, then the corresponding target $x_t$ may plausibly depict different breeds, poses, or body configurations. A deterministic predictor trained with MSE cannot represent such ambiguity explicitly and instead tends to average over the valid modes. This motivates a generative formulation capable of routing predictions toward multiple conditional possibilities.

\paragraph{Objective.}
One of the central claims of this work is that deterministic predictive architectures (e.g. classic JEPA) and unimodal generative models (e.g. single-Gaussian GJE) are limited when the true conditional distribution $p(x_t \mid x_c)$ is strongly multi-modal. To examine this claim in a controlled setting, we construct synthetic ambiguous-alignment tasks inspired by inverse problems \cite{bishop1994mixture}. Because the ground-truth data-generating process (DGP) is known exactly, these experiments allow us to compare how different architectures recover the underlying conditional structure, both visually and probabilistically.

\paragraph{Task design.}
To isolate the predictive and generative mechanisms from representation learning itself, we bypass the dual encoders in this experiment and treat $x_c$ and $x_t$ as direct observations. This lets us focus exclusively on the conditional modeling problem.

We construct two synthetic datasets in which a one-dimensional context variable $x_c \sim \mathcal{U}(-1,1)$ maps to a target variable
\[
x_t = f(x_c) + \epsilon,
\]
where the branch function $f(x_c)$ is selected uniformly at random for each sample, and $\epsilon \sim \mathcal{N}(0, 0.05^2)$ is small Gaussian observation noise. Thus, for a fixed $x_c$, the conditional distribution of $x_t$ is an equal-weight mixture of several branches.

\begin{itemize}[label=-]
    \item \textit{Dataset A (Separated Branches):} 
    \[
    f(x_c) \in \left\{x_c^2 + 0.5,\; -x_c^2 - 0.5,\; x_c^3 \right\},
    \]
    producing three non-linear branches with clear spatial separation and visible low-density gaps.

    \item \textit{Dataset B (Intersecting Branches):}
    \[
    f(x_c) \in \left\{\sin(3x_c),\; -\sin(3x_c),\; 0 \right\},
    \]
    producing three branches that intersect exactly at the origin, thereby creating a more challenging ambiguity structure.
\end{itemize}

For each dataset, we generate $N=3000$ training pairs. At evaluation time, we probe the learned predictive distributions on a uniformly spaced grid of $300$ test context values over $[-1,1]$.

\paragraph{Baselines and model configurations.}
We compare six architectures spanning deterministic prediction, dual-space conditional modeling, fixed global Gaussian structure, and flexible multi-modal generative routing:

\begin{enumerate}
    \item \textit{Classic JEPA (MSE):} a deterministic predictor trained with Mean Squared Error to regress $x_t$ from $x_c$.

    \item \textit{Dual-space kernel baseline (RBF GJE):} a Gaussian Process regression model with an RBF kernel, used to represent the exact dual $N \times N$ sample-space formulation. To expose its behavior under severe multi-modal ambiguity and to avoid unstable hyperparameter collapse on these synthetic tasks, we fix the kernel hyperparameters (length-scale $=0.5$, noise level $=0.1$).

    \item \textit{GMJE-EM ($K=1$):} a single-component Gaussian model in primal space, fitted by Expectation-Maximization, used to illustrate the limitations of representing the full structure with only one global covariance matrix.

    \item \textit{GMJE-EM ($K=3$):} a fixed-$K$ parametric mixture baseline with three global prototypes, also optimized by EM. This provides a direct comparison against a standard finite-mixture model with the correct number of components.

    \item \textit{GMJE-GNG:} a non-parametric topology-adaptive variant using Growing Neural Gas to discover the number and placement of prototypes automatically, without assuming $K$ in advance (Section~\ref{subsec:GMJE_GNG}).

    \item \textit{GMJE-MDN:} an instance-conditional mixture model in which a parameter network maps $x_c$ to mixture parameters $\mu_{t\mid c,k}(x_c)$, $\Sigma_{t\mid c,k}(x_c)$, and $\pi_k(x_c)$, thereby defining an input-dependent conditional density (Section~\ref{subsec:GMJE_MDN}).
\end{enumerate}

\paragraph{Results and analysis.}
The learned predictive distributions for Dataset A and Dataset B are shown in Fig.~\ref{fig:synthetic_exp_sep} and Fig.~\ref{fig:synthetic_exp_int}, respectively.

\begin{itemize}[label=-]
    \item \textit{Deterministic conditional-mean limitation.}  
    As predicted by squared-loss theory (Appendix.~\ref{app:mse_derivation}), Classic JEPA does not recover the multi-branch conditional structure. Instead, the predictor collapses towards the conditional mean of the target distribution. In Dataset A (Fig.~\ref{fig:synthetic_exp_sep}a), this places predictions in low-density regions between the true branches. In Dataset B (Fig.~\ref{fig:synthetic_exp_int}a), the same mechanism drives the predictor toward the central average around zero, again failing to isolate the valid modes.

    \item \textit{Limitation of unimodal dual-space prediction.}  
    The dual-space RBF kernel baseline (Fig.~\ref{fig:synthetic_exp_sep}b and Fig.~\ref{fig:synthetic_exp_int}b) yields a flexible non-linear conditional mean, but its predictive distribution remains unimodal Gaussian. Consequently, although it can interpolate smoothly in input space, it cannot separate multiple concurrent conditional branches. The result is a broad uncertainty band that covers the overall spread of the data without resolving the distinct semantic modes.

    \item \textit{Limitation of a single global covariance.}  
    GMJE-EM with $K=1$ (Fig.~\ref{fig:synthetic_exp_sep}c and Fig.~\ref{fig:synthetic_exp_int}c) illustrates the limitation of fitting the entire structure with one Gaussian component. To cover several curved branches simultaneously, the learned density expands into a large ellipse, producing substantial over-smoothing and poor local alignment. This provides a visual demonstration of why a single global covariance is insufficient for strongly multi-modal geometry.

    \item \textit{Fixed-$K$ mixtures improve mode separation but remain rigid.}  
    Moving to $K=3$ (Fig.~\ref{fig:synthetic_exp_sep}d and Fig.~\ref{fig:synthetic_exp_int}d), the parametric GMJE-EM model distributes probability mass across multiple components and therefore avoids the collapse of the single-component baseline. However, each component still retains a fixed global covariance structure, so the resulting ellipses only coarsely approximate the curved branches. This effect is especially visible when the geometry is highly non-linear or intersecting.

    \item \textit{Adaptive topology discovery with GMJE-GNG.}  
    GMJE-GNG (Fig.~\ref{fig:synthetic_exp_sep}e and Fig.~\ref{fig:synthetic_exp_int}e) avoids fixing the number of components in advance and instead places prototypes adaptively along the data manifold. In both datasets, the learned nodes track the branch topology more faithfully than the rigid global ellipses of GMJE-EM, yielding a discrete but informative approximation of the underlying structure.

    \item \textit{Instance-conditional mixture routing with GMJE-MDN.}  
    Among the tested models, GMJE-MDN (Fig.~\ref{fig:synthetic_exp_sep}f and Fig.~\ref{fig:synthetic_exp_int}f) provides the most expressive conditional generative representation. Because its mixture parameters are conditioned directly on $x_c$, it can adapt both the local means and local variances to the geometry of the branches. As a result, it tracks the non-linear multi-modal structure closely and produces tight, localized uncertainty bands around the valid target regions.
\end{itemize}

To examine this last point more closely, we inspect the internal parameter routing of GMJE-MDN in Fig.~\ref{fig:mdn_internals_sep} and Fig.~\ref{fig:mdn_internals_int}. Several patterns are evident. First, the learned conditional means in panel (b) closely follow the ground-truth branch functions. Second, the predicted conditional standard deviations in panel (c) remain near the true observation noise level $\epsilon = 0.05$, indicating that the model captures the local dispersion without artificially inflating uncertainty. Third, the conditional mixing weights in panel (a) remain close to the ground-truth uniform value of $1/3$, which is consistent with the fact that the generating process selects each branch equiprobably. Taken together, these results suggest that GMJE-MDN recovers the synthetic conditional density substantially more faithfully than the alternative baselines.

% --- Figures for Dataset A ---
\begin{figure}[H]
    \centering
    \includegraphics[width=1.0\textwidth]{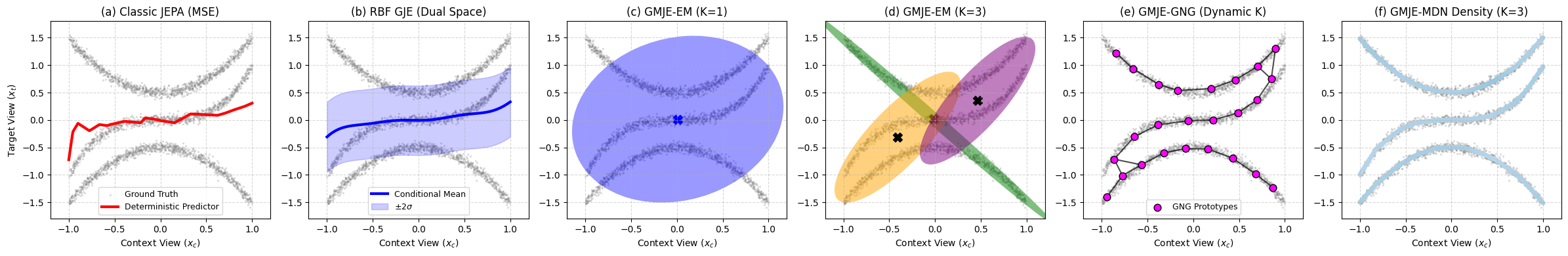}
    \caption{Predictive distributions on Dataset A (Separated Branches). (a) Classic JEPA collapses toward the conditional average and fails to recover the three valid branches. Ground-truth samples are shown in gray. (b) The dual-space RBF kernel baseline learns a flexible mean but retains a unimodal Gaussian predictive distribution, producing a broad uncertainty band rather than resolving separate branches. (c) GMJE-EM ($K=1$) fits a single global ellipse and heavily over-smooths the manifold. (d) GMJE-EM ($K=3$) separates the density across multiple components, but the fixed covariance ellipses remain geometrically rigid. (e) GMJE-GNG adaptively places prototypes along the non-linear ridges and captures the manifold topology more faithfully. (f) GMJE-MDN produces the closest qualitative match to the ground-truth conditional density by predicting instance-dependent mixture parameters.}
    \label{fig:synthetic_exp_sep}
\end{figure}

\begin{figure}[H]
    \centering
    \includegraphics[width=1.0\textwidth]{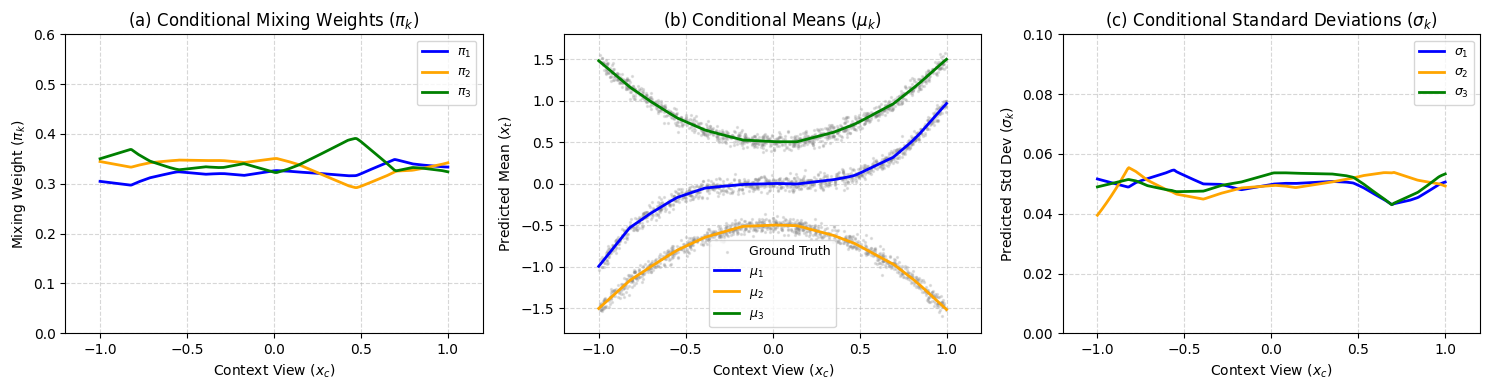}
    \caption{Internal parameter routing of GMJE-MDN on Dataset A. (a) The conditional mixing weights remain close to the ground-truth uniform value of $1/3$. (b) The learned conditional means align closely with the three underlying branch functions. (c) The predicted conditional standard deviations remain near the true Gaussian observation noise level $\epsilon = 0.05$.}
    \label{fig:mdn_internals_sep}
\end{figure}

% --- Figures for Dataset B ---
\begin{figure}[H]
    \centering
    \includegraphics[width=1.0\textwidth]{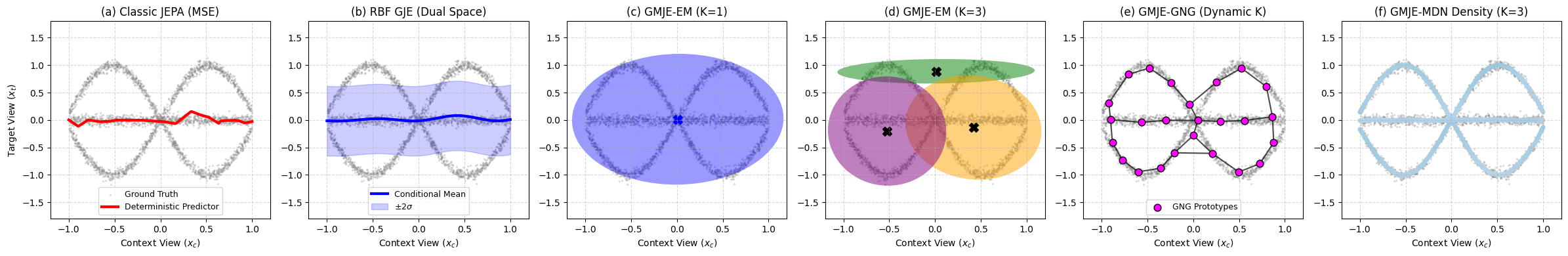}
    \caption{Predictive distributions on Dataset B (Intersecting Branches). The overall trends are consistent with Dataset A but the intersection structure makes the task more challenging. (a) Classic JEPA again collapses toward the conditional average near zero. (b) The dual-space RBF kernel baseline remains unimodal and cannot isolate the intersecting branches. (c--d) The fixed-covariance parametric GMJE baselines capture part of the multi-modal structure but remain limited by global geometric rigidity. (e) GMJE-GNG provides an adaptive topological approximation of the manifold. (f) GMJE-MDN most accurately captures the intersecting conditional density among the compared methods.}
    \label{fig:synthetic_exp_int}
\end{figure}

\begin{figure}[H]
    \centering
    \includegraphics[width=1.0\textwidth]{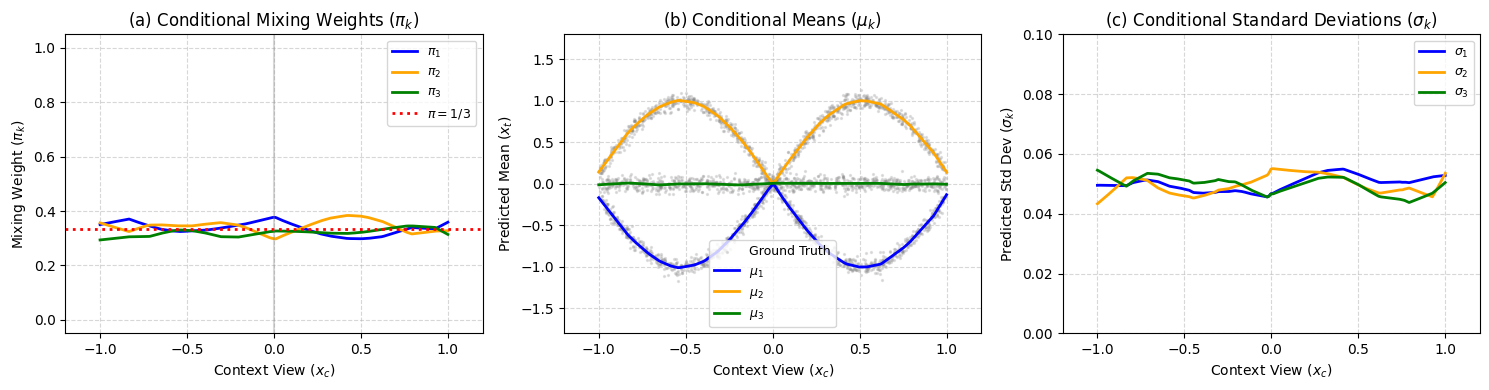}
    \caption{Internal parameter routing of GMJE-MDN on Dataset B. The learned means track the three intersecting target branches, the standard deviations remain close to the true observation noise, and the mixing weights remain near the ground-truth value of $1/3$ across the context domain.}
    \label{fig:mdn_internals_int}
\end{figure}

\paragraph{Summary of synthetic experiments.}
Across both ambiguous-alignment tasks, the experiments support the theoretical picture developed in the preceding sections. Deterministic prediction collapses towards the conditional mean and therefore fails to model multi-valued targets. Unimodal dual-space and single-Gaussian primal-space models capture only coarse averages of the conditional structure and cannot represent several branches simultaneously. Fixed-$K$ parametric mixtures improve mode separation but remain limited by rigid global covariance geometry. In contrast, the non-parametric GMJE-GNG variant preserves the manifold topology more effectively, while the instance-conditional GMJE-MDN delivers the most faithful recovery of the underlying multi-modal conditional density. These results provide a clear qualitative demonstration that generative joint modeling is substantially better suited than deterministic prediction for ambiguous semantic alignment.

\subsection{Representation Learning on Vision Benchmarks} \label{subsec:exp_vision}

\paragraph{Objective.} Having validated the generative modeling capabilities of GMJE on complex multi-modal \textit{synthetic} manifolds, we now evaluate its scalability and feature extraction (representation) quality on standard high-dimensional computer vision benchmarks. This study has two complementary goals. First, we test whether the proposed Sequential Monte Carlo (SMC) memory bank improves the utility of contrastive negatives relative to a standard FIFO queue under severe memory constraints. Second, we evaluate how the resulting non-parametric GMJE variant compares with established SSL baselines on CIFAR-10 under a common short-horizon pre-training regime.

\paragraph{Setup.}
All two sub-experiments were conducted on CIFAR-10 \cite{krizhevsky2009learning}. Following common SSL protocols for low-resolution vision benchmarks \cite{Chen2020simCLR,He2020MoCo}, we use a modified ResNet-18 backbone as the base encoder $E_\theta$, replacing the initial $7 \times 7$ convolution with a $3 \times 3$ convolution and removing the first max-pooling layer. The input views $x_c$ and $x_t$ are generated using the standard SimCLR augmentation pipeline (random resized crops, color jitter, grayscale, and horizontal flips) \cite{Chen2020simCLR}. We evaluate the quality of the learned representations using a standard \textit{Linear Probing} protocol: after unsupervised pre-training, the encoder is frozen and a single linear classifier is trained on top of the global average-pooled features.

\subsubsection{The Efficiency of SMC Memory Banks vs. FIFO}

In Section \ref{sec:SMC_Memory_Bank}, we theoretically exposed standard contrastive learning as a non-parametric GMJE utilizing a naive, uniform-prior FIFO queue which approximates the marginal latent distribution, and sample replacement depends only on temporal age rather than informativeness. To test our proposed dynamic alternative, i.e. \textbf{SMC-GMJE}, we evaluate the representation learning capabilities of \textbf{SMC-GMJE} against a standard FIFO-based instance discrimination baseline (MoCo v2). 

\paragraph{Protocol.}
To isolate the effect of the memory-bank update rule, both SMC-GMJE and MoCo v2 use the same ResNet-18 encoder, the same augmentation pipeline, and the same EMA momentum rate. Both models are trained from scratch for 1000 epochs on CIFAR-10 with a severely constrained memory capacity of $M=256$. We report three diagnostics: the pre-training InfoNCE loss, the downstream linear probing accuracy, and the Effective Sample Size (ESS) of the SMC particle system. The results are shown in Fig.~\ref{fig:smc_vs_fifo}.

\begin{figure}[H]
    \centering
    \includegraphics[width=1.0\textwidth]{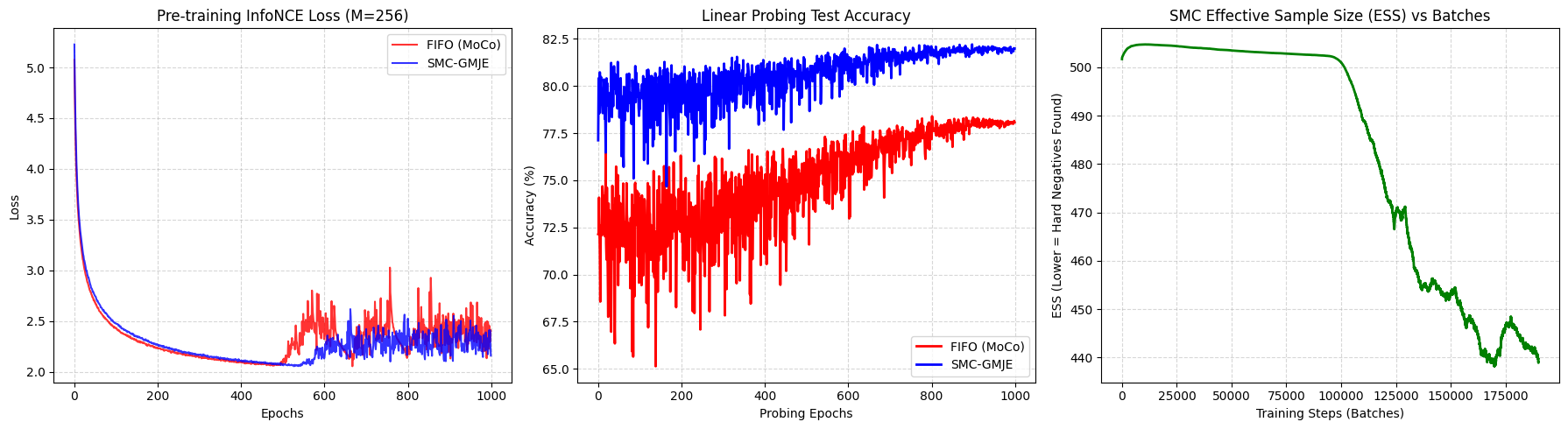} 
    \caption{Empirical training dynamics comparing a standard FIFO queue (MoCo) against our proposed SMC particle bank (SMC-GMJE) under a constrained memory capacity ($M=256$). (Left) Pre-training InfoNCE loss. Note the severe volatility of the FIFO queue starting around Epoch 500. (Center) Linear probing test accuracy, showing SMC maintaining a distinct advantage. (Right) The Effective Sample Size (ESS) of the SMC filter. The sharp ESS drop at $\sim 100,000$ batches (Epoch 500, after the warmup stage) indicates increasing concentration on a smaller subset of informative particles, coincides with the stabilization of the latent space and the onset of targeted hard-negative mining.}
    \label{fig:smc_vs_fifo}
\end{figure}

\paragraph{Analysis.}
Fig.\ref{fig:smc_vs_fifo} shows that the two memory-bank strategies behave similarly during the early training phase, but diverge substantially later. The FIFO baseline begins to exhibit loss instability after approximately 500 epochs, whereas SMC-GMJE remains more stable and achieves a higher final linear-probe accuracy. A plausible explanation is that, once the latent space begins to stabilize and becomes more structured, and distinct semantic clusters (e.g. dogs, cars, ships) are forming, a small FIFO queue ($M=256$) is increasingly unable to preserve sufficiently informative negatives across all semantic regions (insufficient to cover all 10 CIFAR-10 classes simultaneously). Because FIFO deterministically overwrites the oldest representations regardless of their informational value, it frequently purges rare classes or "hard negatives" (e.g. overwriting the only "ship" embeddings in the queue simply because they expired). This leaves the network temporarily blind to certain semantic boundaries, resulting in severe loss volatility and a suppressed final linear probing accuracy of 78.08\%.

By contrast, the SMC bank adaptively reweights particles according to their current utility. This is reflected in the ESS curve: during the early warmup stage the ESS remains high, indicating relatively uniform importance across particles; however, precisely when the semantic clusters begin to form (epoch 500), the ESS decreases as probability mass concentrates on a smaller subset of more informative negatives (hard-negative particles are becoming scarce but informative). Rather than uniformly exploring, it assigns them massive posterior weights and clones them during the sequential resampling step, keeping them alive in the 256-slot memory bank long past their temporal expiration date. Under this severe memory constraint, the resulting contrastive signal is more stable and yields stronger downstream performance (\textbf{81.99\%} final accuracy, which is +3.91\% over FIFO), which empirically proves that transitioning from a deterministic queue to a probabilistic particle system dramatically increases the informational density of the memory bank, yielding stable, rich gradient signals even under extreme hardware limitations.

\subsubsection{SMC-GMJE vs. Standard Baselines}

The previous experiment isolates the memory-bank mechanism under extreme capacity constraints. We now turn to a broader benchmark question: whether the resulting non-parametric GMJE formulation can learn competitive visual representations compared with established SSL baselines under a common, shorter training regime.

\paragraph{Protocol.}
We compare \textbf{SMC-GMJE} with 3 widely used baselines\footnote{We also tested SwAV, VicReg, Parametric-GMJE, and GMJE-MDN. As they were unstable or under-tuned from scratch in this short training setting, they are not presented here. SwAV, VICReg, Parametric-GMJE, and GMJE-MDN are much harder to optimize from scratch under the present CIFAR-10 regime, which likely require stronger initialization, longer schedules, or more carefully matched objectives.}: SimCLR \cite{Chen2020simCLR}, MoCo v2 \cite{He2020MoCo}, and BYOL \cite{grill2020byol}. All models use the same ResNet-18 backbone and are pre-trained from scratch for 200 epochs, followed by 100 epochs of linear probing.

\begin{table}[H]
    \centering
    \caption{Linear Probing Top-1 accuracy on CIFAR-10 using a ResNet-18 backbone.}
    \label{tab:vision_results}
    \renewcommand{\arraystretch}{1.2}
    \begin{tabular}{l l c}
        \hline
        \textbf{Method} & \textbf{Architecture Type} & \textbf{CIFAR-10 Accuracy} \\
        \hline
        SimCLR \cite{Chen2020simCLR} & Symmetric + Large Batch Contrastive & 81.50\% \\
        BYOL \cite{grill2020byol} & Asymmetric EMA + Predictor & 79.04\% \\
        MoCo v2 \cite{He2020MoCo} & Asymmetric EMA + FIFO Queue & 79.02\% \\
        \hline
        \textbf{SMC-GMJE} (Ours) & Asymmetric EMA + Dynamic Particle Bank & \textbf{74.52\%} \\
        \hline
    \end{tabular}
    
    \vspace{0.5em}
    \raggedright
    {\footnotesize These short-horizon results should be interpreted as relative comparisons under limited compute, rather than as fully tuned final accuracies.\par}
\end{table}

\paragraph{Analysis.}
Table~\ref{tab:vision_results} shows that SimCLR, BYOL, and MoCo v2 all reach approximately 79-82\% linear-probe accuracy after 200 epochs pre-training, whereas SMC-GMJE reaches 74.52\%. Thus, under this short training budget, the non-parametric GMJE variant is competitive but does not yet surpass the strongest heuristic baselines. This is an important complement to the long-horizon memory experiment above: the 1000-epoch $M=256$ study isolates the benefits of SMC under severe queue constraints, while the 200-epoch benchmark evaluates overall representation quality in a broader SSL comparison.

SMC-GMJE trains stably and improves steadily, suggesting that the non-parametric particle-bank formulation is viable in practice: by replacing the rigid FIFO queue with a dynamic SMC particle filter, the memory bank successfully tracks the marginal distribution\footnote{In contrastive learning, the memory bank serves as an empirical approximation of the marginal distribution $p(z_c)$ over the latent space, which constitutes the denominator of the InfoNCE objective. Geometrically, this represents the overall density and shape of the dataset's manifold. While a rigid FIFO queue indiscriminately discards older representations—thereby losing track of true density peaks, the SMC particle filter dynamically re-weights and resamples particles based on their likelihoods. This ensures the memory bank continuously and accurately maps the high-density regions and hard-negative ridges of the underlying latent landscape.}, and provides stable, informative gradients without succumbing to weight degeneracy or representation collapse. However, the gap to SimCLR, BYOL, and MoCo indicates that the current implementation remains \textit{under-optimized} relative to those mature baselines. This is not surprising: the heuristic baselines have benefited from extensive design refinement, whereas SMC-GMJE is a first principled instantiation of the non-parametric GMJE view.

\vspace{1em}
\paragraph{Summary of vision benchmarks.}
Taken together, the vision experiments support two conclusions. First, \textit{under strong memory constraints}, dynamic particle reweighting is a viable and often more effective alternative to a rigid FIFO queue, yielding a more informative contrastive memory bank and more stable long-horizon training. Second, when evaluated as a practical SSL method on CIFAR-10, the current non-parametric GMJE implementation learns competitive representations but does not yet outperform the strongest heuristic baselines under short pre-training. These findings support the usefulness of the GMJE perspective while also indicating that further optimization and larger-scale tuning are needed for it to fully match mature contrastive and predictive SSL systems.

\subsection{Generative GMJE: Unconditional Image Synthesis via Latent Sampling} \label{subsec:exp_generative}

\paragraph{Objective.} In Section \ref{subsec:GMJE_generative}, we established that GMJE natively acts as a generative model of the representation space, $p(z_t)$. The purpose of this experiment is to evaluate that claim empirically by testing whether unconditional latent samples drawn from GMJE decode into realistic and semantically coherent images, and how this compares with two alternatives: (i) a post-hoc density model fitted to a discriminative SimCLR representation space, and (ii) a unimodal GJE latent density. The key question is whether explicit multi-modal latent density modeling yields a more suitable representation space for unconditional generation.

\paragraph{Task and setup.}
To isolate the generative properties of the learned latent spaces, we construct a simple latent-sampling pipeline on MNIST \cite{lecun1998mnist}, with all images zero-padded to $32 \times 32$. The experiment proceeds in three phases:
\begin{enumerate}
    \item \textit{Manifold learning (encoders):} we train three symmetric encoder families from scratch for 50 epochs to induce their respective latent geometries: SimCLR (InfoNCE), Unimodal GJE (MSE + entropy regularization), and Parametric GMJE (log-sum-exp + entropy regularization).
    
    \item \textit{Latent inversion (decoder):} after pre-training, each encoder is frozen, and a lightweight convolutional decoder $D(z)$ is trained for 50 epochs with an MSE reconstruction loss to map the frozen latent embeddings $z_i$ back to their corresponding images $x_i$.
    
    \item \textit{Unconditional sampling:} the original image dataset is then discarded, and new latent points $\hat{z}$ are sampled directly from the learned latent distributions, projected onto the hypersphere, and passed through the decoder $D(\hat{z})$ to synthesize images.
\end{enumerate}
This setup intentionally separates \emph{representation learning}, \emph{latent density modeling}, and \emph{image decoding}. As a result, sample quality depends both on the geometry of the learned latent distribution and on whether the sampled points remain within regions where the decoder has been trained to invert the latent manifold.

\paragraph{Sampling baselines.}
Because not all SSL models define a native latent density $p(z)$, we compare three different sampling strategies:
\begin{itemize}[label=-]
    \item \textit{SimCLR (Post-Hoc GMM):} we fit a post-hoc Gaussian Mixture Model ($K=50$ components, full covariance, optimized by EM) to the frozen SimCLR embeddings, and sample $\hat{z}$ from this externally fitted distribution.
    
    \item \textit{Unimodal Primal-GJE:} we sample $\hat{z}$ from a single global multivariate Gaussian $\mathcal{N}(\mu,\Sigma)$ estimated from the full latent distribution learned by unimodal GJE.
    
    \item \textit{Parametric (Primal) GMJE:} we sample $\hat{z}$ directly from the mixture distribution natively learned during pre-training,
    $\sum_{k=1}^{K} \pi_k \, \mathcal{N}(\mu_k,\Sigma_k)$.
\end{itemize}

\begin{figure}[H]
    \centering
    \includegraphics[width=1.0\textwidth]{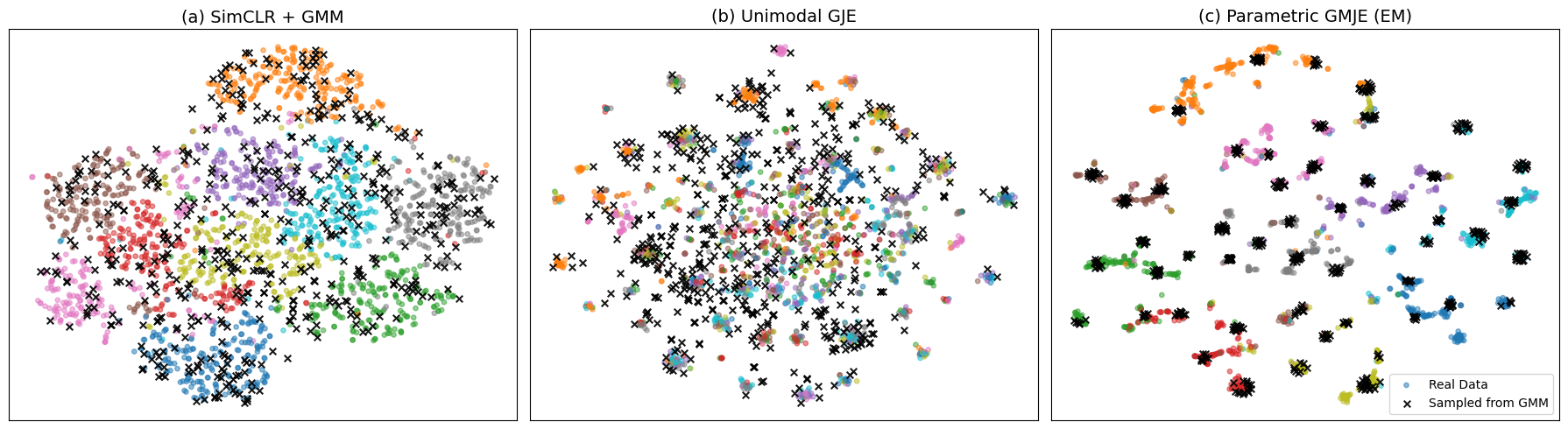}
    \caption{t-SNE visualization of real image embeddings (colored by class) and synthetic latent samples (black crosses). (a) A post-hoc GMM fitted to SimCLR places substantial mass between semantic clusters (e.g. empty regions). (b) Unimodal GJE spans the latent space broadly but lacks clear multi-cluster organization. (c) Parametric GMJE generates samples that align more closely with the clustered semantic structure of the learned manifold.}
    \label{fig:tsne_generative}
\end{figure}

\begin{figure}[H]
    \centering
    \includegraphics[width=0.8\textwidth]{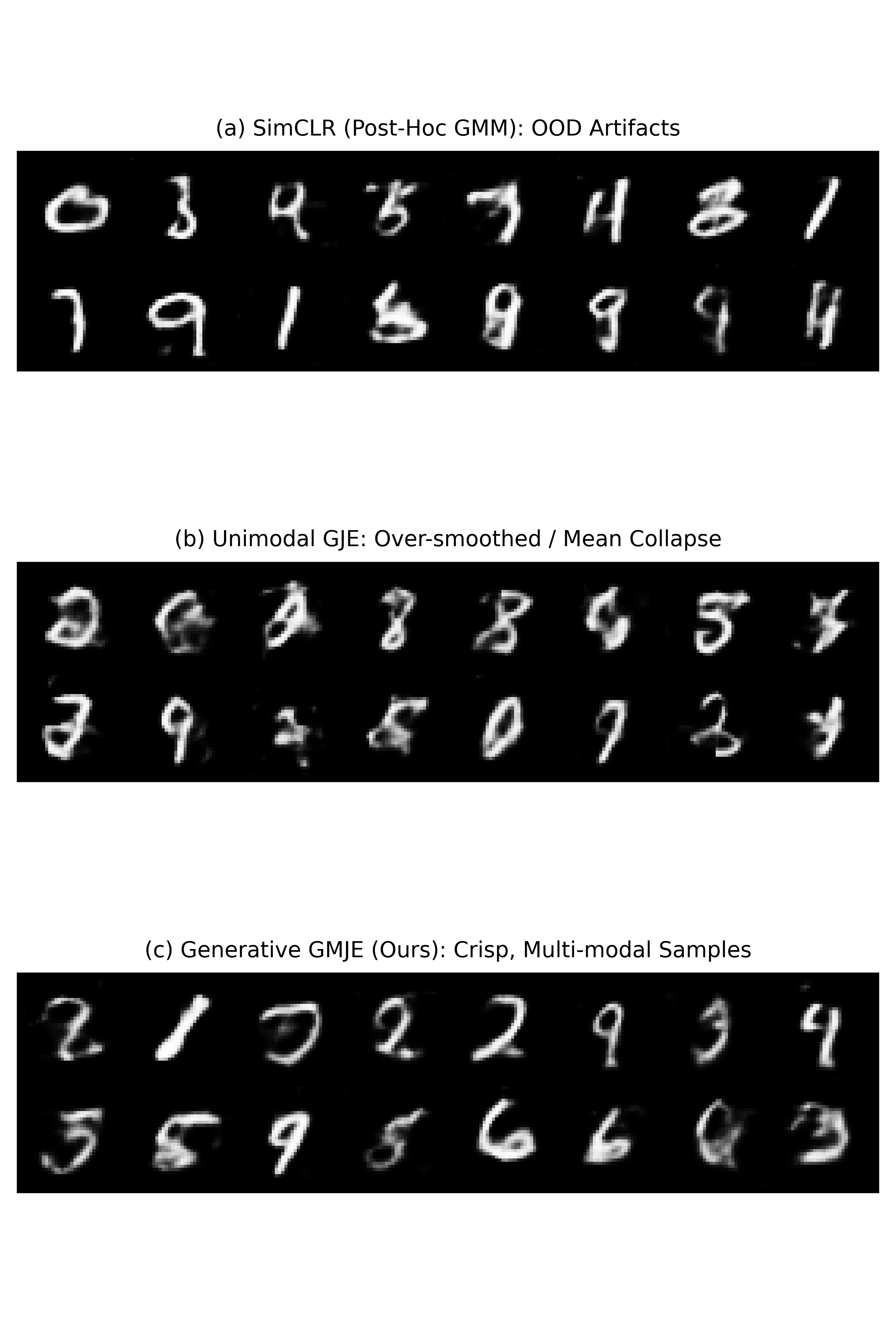} 
    \caption{Unconditional image synthesis from latent samples. (a) SimCLR produces blurry, out-of-distribution chimeras. (b) Unimodal GJE suffers from mean collapse, producing nearly identical, heavily smoothed averages with reduced diversity. (c) Parametric GMJE yields the sharpest and most diverse decoded samples among the three methods.}
    \label{fig:generative_samples}
\end{figure}

{\scriptsize
\begin{table}[H]
    \centering
    \setlength{\tabcolsep}{3pt}
    \renewcommand{\arraystretch}{1.0}
    \caption{Geometric properties of the learned latent distributions ($d=128$). These statistics provide a compact summary of component volume, total variance, and cross-component separation.}
    \label{tab:dist_geometry}
    \resizebox{\linewidth}{!}{
    \begin{tabular}{l c c c}
        \hline
        \textbf{Model Framework} & \textbf{Avg. $\log |\Sigma_k|$ (Volume)} & \textbf{Avg. $Tr(\Sigma_k)$ (Variance)} & \textbf{Prototype Spread ($L^2$)} \\
        \hline
        SimCLR (Post-Hoc GMM) & -878.99 & 0.85 & 0.55 \\
        Unimodal GJE ($K=1$) & -636.15 & 0.99 & N/A \\
        \textbf{Parametric GMJE (EM)} & \textbf{-621.58} & \textbf{1.00} & \textbf{1.43} \\
        \hline
    \end{tabular}}
\end{table}
}

\paragraph{Results and analysis.}
By jointly examining the geometric statistics in Table~\ref{tab:dist_geometry}, the latent topology in Fig.~\ref{fig:tsne_generative}, and the decoded synthetic samples in Fig.~\ref{fig:generative_samples}, we observe three distinct generative behaviors.

\begin{itemize}[label=-]
    \item \textit{Post-hoc density fitting on SimCLR latents.}  
    SimCLR learns useful discriminative representations, but its latent space is not trained as an explicit unconditional generative density. When a post-hoc GMM is fitted to the frozen embeddings, the resulting mixture places non-negligible probability mass in low-density regions between semantic clusters. This is visible in Fig.~\ref{fig:tsne_generative}(a), where sampled points frequently appear in inter-cluster regions rather than concentrating inside the class manifolds. After decoding, these off-manifold samples lead to blurrier and more class-ambiguous digits, as shown in Fig.~\ref{fig:generative_samples}(a). The relatively small prototype spread reported in Table~\ref{tab:dist_geometry} is consistent with this weak macroscopic separation.

    \item \textit{Over-smoothing in unimodal GJE.}  
    The unimodal GJE baseline maintains substantial latent variance and volume (Table~\ref{tab:dist_geometry}), but a single Gaussian component is too restrictive to model the multi-class structure of MNIST. As a result, unconditional samples cover the latent space broadly but lack explicit multi-modal separation. This is reflected in Fig.~\ref{fig:tsne_generative}(b), where sampled points are dispersed without clear cluster structure, and in Fig.~\ref{fig:generative_samples}(b), where the decoded images appear over-smoothed and less diverse. In this sense, the unimodal model captures global density but not class-conditional structure.

    \item \textit{Structured multi-modality in GMJE.}  
    Parametric GMJE combines explicit latent density modeling with multiple separated mixture components. Table~\ref{tab:dist_geometry} shows that it achieves the largest prototype spread while preserving non-degenerate component volumes and overall variance. In Fig.~\ref{fig:tsne_generative}(c), latent samples drawn from the learned mixture align much more closely with clustered semantic regions than either of the two baselines. Correspondingly, the decoded samples in Fig.~\ref{fig:generative_samples}(c) are sharper and more visually diverse. Taken together, these results suggest that GMJE provides a substantially better latent generative model than either post-hoc density fitting on a discriminative space or a single-Gaussian latent approximation.
\end{itemize}

\paragraph{Summary.}
These results highlight an important distinction between \emph{discriminative} and \emph{generative} latent geometry. A representation space that is effective for downstream discrimination does not automatically define a good unconditional sampling distribution. Post-hoc density fitting can partially recover latent structure, but it remains mismatched to an encoder that was not trained with a generative objective. Likewise, a single Gaussian latent model captures only coarse global geometry. By contrast, GMJE directly learns a multi-modal latent density during representation learning itself, which leads to both better-aligned sampled embeddings and more plausible decoded images.

Overall, the MNIST latent-sampling experiment supports the claim that GMJE is not only a representation learner but also a meaningful latent generative model. Compared with a post-hoc GMM on SimCLR and a unimodal GJE density, GMJE produces better-separated latent samples and the strongest qualitative unconditional generations among the tested models. This provides empirical evidence that explicit multi-modal latent density modeling is a useful mechanism for bridging self-supervised representation learning and generative sampling.

\section{Discussion} \label{sec:discussion}

A central contribution of this work is to show that shifting from deterministic prediction to generative joint modeling provides a principled probabilistic perspective on representation collapse in SSL. In particular, by formulating the joint embedding distribution through Gaussian and Gaussian-mixture models, the GMJE framework offers an explicit alternative to heuristic architectural interventions and provides a unified approach to multi-modal representation learning.

\paragraph{Unifying SSL Paradigms.}
Prior to this work, the SSL landscape was largely developed through separate design traditions. In this work, we show that GMJE provides a common probabilistic perspective that connects several of these fragmented directions:
\begin{itemize}[label=-]
    \item \textit{The Contrastive Paradigm (Non-Parametric GMJE):} we show that standard contrastive learning methods (e.g. SimCLR, MoCo) can be interpreted as a degenerate non-parametric special case of the GMJE framework (Section \ref{sec:GMJE_Contrastive}). This perspective highlights a limitation of standard FIFO memory queues, namely that sample replacement is based on age rather than informativeness. To address this, we introduce an \textit{SMC Particle Bank} (Section \ref{sec:SMC_Memory_Bank}), which adaptively reweights memory samples so that informative hard negatives are emphasized while stale or redundant samples are downweighted.
    
    \item \textit{The Predictive Paradigm (Parametric GMJE):} we address the multi-modal alignment difficulty that arises in inverse-style augmentation tasks by replacing deterministic predictors (classic JEPA) with dynamic conditional mixtures (\textit{GMJE-MDN}). In addition, to reduce the rigidity of fixed mixture cardinality $K$, we introduce a topology-preserving \textit{Growing Neural Gas} variant (\textit{GMJE-GNG}) which adaptively adds prototypes according to local quantization error.
\end{itemize}

\paragraph{The Elimination of Asymmetric Heuristics.}
Many empirically successful non-contrastive predictive methods, including BYOL and related Siamese architectures, rely on asymmetric architectural or optimization mechanisms, e.g. EMA target encoders, stop-gradients, or predictor asymmetry, to mitigate representation collapse \cite{grill2020byol,chen2021simsiam,zhuo2023unifiedtheoreticalunderstandingnoncontrastive}. In contrast, the GMJE formulation is designed to mitigate dimensional collapse through an explicit covariance-aware geometric regularizer, including the log-determinant term $\frac{1}{2}\log|\Sigma|$. Under our probabilistic interpretation, the resulting tension between the Mahalanobis data-fitting term and the covariance volume penalty promotes a non-degenerate, full-rank representation geometry. As a result, gradients can be propagated symmetrically through both encoders without relying on the asymmetric dampers or rank-differential mechanisms that are commonly used in prior non-contrastive SSL methods.

\paragraph{Generative vs. Discriminative Manifolds.}
This work connects discriminative representation learning and explicit generative modeling in latent space. Standard predictive and contrastive architectures do not natively define a continuous probabilistic density over the representation space. As demonstrated in our generative experiments (Section.\ref{subsec:exp_generative}), applying post-hoc density estimation to a uniformly scattered contrastive space (e.g. SimCLR) inevitably samples from low-density gaps, producing out-of-distribution chimeric artifacts. Conversely, strictly unimodal constraints collapse generative sampling into an over-smoothed, ambiguous average. In contrast, GMJE equips the latent space with an explicit Gaussian-mixture density model during representation learning itself. By balancing local variance compression with global prototype spread, it constructs a mathematically valid generative manifold capable of crisp, unconditional sampling without sacrificing discriminative utility.

\paragraph{Dynamic Component Selection and GMJE-GNG.}
A persistent challenge in parametric density estimation is selecting the number of mixture components, $K$, in advance. In our Parametric GMJE experiments, $K$ was fixed empirically (e.g. $K=50$ in Section \ref{subsec:exp_generative}) as a rough approximation to the semantic granularity of the dataset. In realistic, uncurated settings, however, the true number of latent modes is rarely known beforehand. Overestimating $K$ can lead to fragmented semantic classes and inactive or redundant prototypes, whereas underestimating $K$ can force several distinct modes to be absorbed into overly broad components, thereby blurring semantic boundaries. Traditional model-selection criteria such as AIC or BIC are mainly post-hoc and are difficult to integrate into end-to-end deep representation learning.

Fortunately, our non-parametric formulation (\textit{SMC-GMJE}, or in general \textit{DaM-GMJE}) inherently bypasses this limitation. By dynamically tracking the marginal distribution using a SMC particle filter, it maps the latent manifold without ever requiring an a priori declaration of $K$. Within a purely parametric architecture, we have \textit{GMJE-GNG} as well. Rather than relying on a fixed $K$, GMJE-GNG represents the marginal distribution through a dynamic topological graph. By monitoring localized representation error and the evolving density of the latent manifold, the model can recruit new Gaussian prototypes in regions requiring additional capacity and prune dormant components when they are no longer useful. In this way, the generative memory bank can adapt its structural complexity to the observed topology of the data, reducing the need to hand-tune $K$ in advance.

\paragraph{The Gaussian Beauty.}
This work further highlights the central role of Gaussian probabilistic structures in modern machine learning. Gaussian models have long served as a foundation for uncertainty-aware inference and function-space learning, as exemplified by Gaussian Processes \cite{rasmussen2006gaussian}, and they also remain closely connected to modern generative modeling through diffusion and score-based methods \cite{huang2022classificationscorebasedgenerativemodelling,Huang2025GMA}. By bringing Gaussian and Gaussian-mixture modeling together with MDN in a symmetric self-supervised framework, our results suggest that representation learning can be elegantly viewed through the lens of probabilistic density estimation\footnote{This work also marks an milestone in the author's broader research on Gaussian families. When I was studying at Oxford, a professor once remarked on his inference course that the Gaussian distribution is one of the most beautiful, and his favorite mathematical structures in existence. I did not understand this fully at the time, but its depth became clearer over the years. I began working on Gaussian processes and MDN in 2020 (with practical use dating back to 2017), then explored score-based generative models and SVGD in 2022, and finally the unification of these ideas under GMM in 2025.}.

\section{Conclusion} \label{sec:conclusion}

In this work, we introduced Gaussian Joint Embeddings (GJE) and its multi-modal extension, Gaussian Mixture Joint Embeddings (GMJE), as a probabilistically grounded framework for self-supervised representation learning. We analyzed the limitations of deterministic predictive architectures, particularly their difficulty in navigating multi-modal inverse problems, and their vulnerability to representation collapse without asymmetric heuristics.

By formulating self-supervised learning as the optimization of a joint generative density $p(z_c, z_t)$, GMJE replaces deterministic black-box prediction with closed-form conditional inference under an explicit probabilistic model, while also providing uncertainty estimates. We further identified the \textit{Mahalanobis Trace Trap} as a failure mode of naive empirical batch optimization and proposed several structural remedies spanning parametric, adaptive, and non-parametric settings, including decoupled global prototypes, dynamic Mixture Density Networks (MDN), topology-adaptive Growing Neural Gas (GMJE-GNG), and Sequential Monte Carlo (SMC) particle filtering.

Across synthetic tasks and high-dimensional vision benchmarks, GMJE recovered complex multi-modal conditional structure and demonstrated strong discriminative performance. By penalizing representation volume while explicitly maintaining multi-modal covariance structure, GMJE also defines a continuous latent density suitable for unconditional sampling. More broadly, the results suggest that generative joint modeling offers a promising route towards more interpretable and flexible representation learning systems that can better accommodate the multi-modal structure of real-world data.

\vspace{1em}
\noindent \textbf{Future Work.} Future extensions of this framework will investigate the integration of GMJE into sequential and autoregressive settings, such as video prediction and Large Language Models (LLMs), where predictive ambiguity is often inherently multi-modal. Another promising direction is to combine the continuous density modeling of GMJE with the discrete neuro-symbolic constraints (e.g. RiJEPA \cite{huang2026RiJEPA}), with the aim of learning more structured and causally informed generative manifolds.

\section{Related Works} \label{sec:related_work}

Self-supervised learning (SSL) has developed through several major paradigms, including pretext-based prediction, contrastive representation learning, non-contrastive joint-embedding methods, predictive architectures, and classical probabilistic modeling. Our Gaussian Mixture Joint Embeddings (GMJE) framework is most closely related to these latter three directions: contrastive learning, non-contrastive and predictive joint-embedding methods, and probabilistic latent-variable approaches.

\paragraph{Self-Supervised and Contrastive Learning.}
The modern era of SSL was strongly influenced by Contrastive Predictive Coding (CPC) \cite{oord2019InfoNCE}, which introduced the InfoNCE objective as a tractable contrastive criterion related to mutual information maximization. While CPC provided an influential information-theoretic perspective on representation learning, the tightness of the InfoNCE mutual-information bound generally improves as more negative samples are used, which motivates large effective contrastive dictionaries \cite{oord2019InfoNCE,wu2021rethinking}. To reduce the dependence on large minibatches, MoCo \cite{He2020MoCo} decoupled dictionary size from minibatch size by introducing a momentum-updated encoder together with a queue-based memory bank, in which the current minibatch is enqueued and the oldest entries are dequeued (i.e. \textit{First-In-First-Out}). This enabled substantially larger sets of negative samples with modest computational overhead, although the replacement rule is based on temporal age rather than sample informativeness. SimCLR \cite{Chen2020simCLR} simplified the contrastive pipeline further by removing the memory bank and instead relying on strong data augmentation, a non-linear projection head, and very large in-batch negative sets. Although architecturally simple and highly effective, SimCLR remains computationally demanding because its performance is closely tied to large-batch training \cite{Chen2020simCLR,yuan2022sogclr}.

\paragraph{Non-Contrastive Learning and Collapse Prevention.}
An early alternative to pairwise discrimination was DeepCluster \cite{caron2019deepcluster}, which alternated between clustering learned representations with $k$-means and using the resulting assignments as pseudo-labels for representation learning. While effective at capturing global semantic structure, this alternating offline clustering procedure was computationally cumbersome and only loosely coupled to continuous network updates \cite{caron2019deepcluster}. To reduce both the reliance on negative sampling and the cost of offline clustering, later methods explored several prominent directions, including asymmetric predictive learning and online clustering. BYOL \cite{grill2020byol} proposed predicting the representation of one augmented view from another using an online network (i.e. a predictor), and a momentum-updated target encoder. Subsequent theoretical analyses showed that architectural asymmetry, especially the predictor and stop-gradient mechanism, plays a central role in preventing representational collapse in such non-contrastive Siamese frameworks \cite{tian2021understanding,halvagal2023implicit,wang2022asymmetry}. In parallel, SwAV \cite{caron2021swav} transformed the DeepCluster-style paradigm into an online assignment framework by mapping representations to learnable prototypes and computing assignments with the Sinkhorn-Knopp algorithm. This reduced the need for explicit pairwise negatives and was motivated in part by avoiding issues such as class collision, although its balanced assignment mechanism can be restrictive when the underlying semantic frequencies are naturally imbalanced \cite{caron2021swav}. SimSiam \cite{chen2021simsiam} later demonstrated that stop-gradients alone, without EMA momentum, were theoretically sufficient to prevent collapse in Siamese networks, and subsequent analyses indicated that stop-gradient together with predictor asymmetry can be sufficient to avoid collapse in simplified Siamese settings \cite{tian2021understanding,chen2021simsiam}.

Expanding on these asymmetric designs, DINO \cite{caron2021emerging} achieved self-distillation with no labels using a student-teacher architecture with an EMA-updated teacher, together with centering and sharpening, notably without requiring a predictor module. DINO was particularly successful at learning object-centric and semantically meaningful visual features, although its collapse-avoidance mechanisms were initially introduced in a largely empirical form \cite{caron2021emerging}. More recently, the Rank Differential Mechanism (RDM) was proposed as a unified theoretical account of several non-contrastive methods \cite{zhuo2023unifiedtheoreticalunderstandingnoncontrastive}. RDM argues that asymmetric designs, including predictors, stop-gradients, and centering, induce a consistent rank difference between the two branches, thereby improving effective dimensionality and alleviating both complete and dimensional feature collapse \cite{zhuo2023unifiedtheoreticalunderstandingnoncontrastive}. Many of these methods therefore rely on asymmetric architectural or optimization mechanisms to stabilize learning. VICReg \cite{bardes2021vicreg} explored a more symmetric alternative by explicitly regularizing the variance and covariance of the learned features. However, it does so through feature-level regularization rather than by modeling a full joint generative dependency between context and target representations \cite{bardes2021vicreg}. In contrast, our GMJE framework is designed to provide an explicitly probabilistic alternative: by modeling the full joint covariance structure, it replaces heuristic collapse-prevention mechanisms with a generative objective that directly couples data fit and covariance-aware geometric regularization.

\paragraph{Deterministic and Probabilistic JEPA.}
Moving away from instance alignment, the Joint Embedding Predictive Architecture (JEPA) \cite{assran2023ijepa,lecun2022path} shifted SSL towards masked latent-space prediction. Rather than reconstructing pixels or tokens, JEPA-style models predict target representations directly in latent space, aiming to capture higher-level semantic structure while filtering out low-level stochastic detail \cite{lecun2022path,assran2023ijepa}. This paradigm has since expanded across domains, including image-based JEPA (I-JEPA) \cite{assran2023ijepa}, video-based JEPA (V-JEPA) \cite{bardes2024vjepa,bardes2025VJEPA2}, motion-content JEPA (MC-JEPA) \cite{bardes2023mcjepa}, skeletal JEPA (S-JEPA) \cite{abdelfattah2024sjepa}, Point-JEPA for point clouds \cite{saito2025pointjepa}, Graph-JEPA \cite{Skenderi2025graphJEPA}, and vision-language extensions such as VL-JEPA \cite{chen2025vlJEPA}. Related developments have also extended JEPA toward planning and world modeling, including JEPA World Models \cite{terver2025jepaworldmodels} and Value-Guided JEPA \cite{Destrade2026ValueGuidedJEPA}. 

Despite this growing family of methods, standard JEPA remains based on a deterministic neural predictor, typically optimized with a regression-style objective such as MSE \cite{assran2023ijepa,bardes2024vjepa}. This design is highly scalable and avoids spending modeling capacity on pixel-level reconstruction \cite{assran2023ijepa}, but deterministic squared-loss prediction is fundamentally limited in genuinely multi-modal settings, since it converges toward the conditional mean and may therefore place predictions in low-density regions (e.g. ``empty regimes'') between valid modes. To address these limitations, several recent probabilistic and structured JEPA variants have emerged. Variational JEPA (VJEPA) \cite{huang2026vjepa} introduces a probabilistic formulation that learns a predictive distribution over future latent states through a variational objective, linking JEPA-style representation learning with Bayesian filtering and Predictive State Representations (PSRs). Bi-directional JEPA (BiJEPA) \cite{huang2026BiJEPA} extends the predictive objective to both directions, encouraging cycle-consistent predictability between data segments and capturing semantic information inherent in inverse relationships, although symmetric prediction can amplify optimization instability and may require explicit norm control, such as $L^2$-normalization, to prevent representation explosion. Most recently, Rule-informed JEPA (RiJEPA) \cite{huang2026RiJEPA} incorporates neuro-symbolic inductive biases through Energy-Based Constraints (EBC), reshaping the latent geometry using structured logical priors to discourage shortcut learning. However, this relies on human-defined logical structure into the latent space. Taken together, these developments address deterministic JEPA from probabilistic, symmetric, and neuro-symbolic directions, yet a fully probabilistic joint-density view of JEPA-style representation learning remains underdeveloped, which motivates our GJE/GMJE framework.

\paragraph{GMM and Probabilistic Inference Mechanisms.}
The mathematical foundation of our GMJE framework builds on classical probabilistic inference and Gaussian density modeling. Mixture Density Networks (MDNs) \cite{bishop1994mixture} combine neural networks with mixture models to represent conditional probability densities and were originally motivated by multi-valued inverse problems. This makes them a natural starting point for conditional multi-modal prediction, although in our setting a naive symmetric use of MDNs can lead to information leakage across views; GMJE-MDN addresses this through an explicit conditional information bottleneck imposed by the joint-modeling formulation. Growing Neural Gas (GNG) \cite{fritzke1995growing} provides a topology-preserving mechanism for adaptively inserting prototypes based on local quantization error. Unlike fixed-$K$ mixture fitting, GNG can discover structure incrementally without requiring the number of components ($K$) in advance, and avoids the local minima of standard Expectation-Maximization (EM), although its traditional implementations are non-differentiable. Our GMJE-GNG adapts this prototype-growth logic for dynamic prototype discovery in continuous representation spaces.

To reduce the computational burden of exact kernel methods, Random Fourier Features (RFF) \cite{rahimi2007random} approximate shift-invariant kernels by projecting inputs into a randomized finite-dimensional Fourier feature space, reducing the computational bottleneck of Gram matrix inversion from $\mathcal{O}(N^3)$ to $\mathcal{O}(ND^2)$. However, as an approximation to dual-space kernel modeling, RFF does not by itself resolve the limitations of unimodal conditional prediction under strong multi-modality, which motivates our shift towards explicit primal-space mixture modeling. Finally, probabilistic inference of GMMs can also be approached through sampling-based methods. For example, Sequential Monte Carlo (SMC) \cite{Moral2006SMC} provides a particle-based framework for sequential weighting, resampling, and posterior approximation, while Gaussian Mixture Approximation (GMA) \cite{Huang2025GMA} introduces a gradient-free optimization-resampling pipeline for approximating unnormalized target densities. Inspired by these sampling and resampling dynamics, we introduce an \textit{SMC-GMJE} variant, which incorporates SMC-style weighting into SSL memory banks in order to adaptively prioritize informative negatives and down-weight stale ones.

\bibliographystyle{plain}
\bibliography{reference}

\appendix

\section{Properties of Gaussian Distributions} \label{app:Gaussian_dist}

In this section, we review the fundamental properties of multivariate Gaussian distributions. These properties form the mathematical foundation for the generative modeling, closed-form inference, and information-theoretic regularization utilized in the Gaussian Joint Embeddings (GJE) framework.

\subsection{Formal Definition and Affine Transformations}
Formally, a $d$-dimensional random vector $\mathbf{x} \in \mathbb{R}^d$ follows a multivariate normal distribution, denoted $\mathbf{x} \sim \mathcal{N}_d(\boldsymbol{\mu}, \Sigma)$, if there exists a mean vector $\boldsymbol{\mu} \in \mathbb{R}^d$ and a matrix $\mathbf{A} \in \mathbb{R}^{d \times \ell}$ such that $\mathbf{x} = \mathbf{A}\mathbf{z} + \boldsymbol{\mu}$, where $\mathbf{z} \sim \mathcal{N}(\textbf{0}, I_\ell)$ is a vector of independent standard normal variables. The covariance matrix is given by $\Sigma = \mathbf{A}\mathbf{A}^\top$.

A critical property of the Gaussian distribution is its closure under affine transformations. If $\mathbf{y} = \mathbf{c} + \mathbf{B}\mathbf{x}$ is an affine transformation of $\mathbf{x} \sim \mathcal{N}(\boldsymbol{\mu}, \Sigma)$ (where $\mathbf{c}$ is a constant vector and $\mathbf{B}$ is a constant matrix), then $\mathbf{y}$ is also multivariate normal with distribution $\mathcal{N}(\mathbf{c} + \mathbf{B}\boldsymbol{\mu}, \mathbf{B}\Sigma\mathbf{B}^\top)$.

\subsection{Density Function and Geometric Interpretation}
When the symmetric covariance matrix $\Sigma$ is positive definite (the non-degenerate case), the probability density function (\textit{pdf}) exists and is defined as:
\begin{equation}
    p(\mathbf{x}) = \frac{1}{\sqrt{(2\pi)^d |\Sigma|}} \exp\left( -\frac{1}{2} (\mathbf{x} - \boldsymbol{\mu})^\top \Sigma^{-1} (\mathbf{x} - \boldsymbol{\mu}) \right)
\end{equation}
where $|\Sigma|$ is the determinant of $\Sigma$, also known as the generalized variance. 

The quantity $\sqrt{(\mathbf{x} - \boldsymbol{\mu})^\top \Sigma^{-1} (\mathbf{x} - \boldsymbol{\mu})}$ is the \textit{Mahalanobis distance}, representing the distance of the point $\mathbf{x}$ from the mean $\boldsymbol{\mu}$ scaled by the covariance. Geometrically, the equidensity contours of a non-singular multivariate normal distribution form \textit{ellipsoids}. The principal axes of these ellipsoids are dictated by the eigenvectors of $\Sigma$, and their squared relative lengths correspond to the respective eigenvalues.

\subsection{Joint Normality vs. Marginal Normality}
It is important to note that if two random variables are normally distributed and independent, they are jointly normally distributed. However, the converse is not generally true: the fact that two variables $\mathbf{x}_1$ and $\mathbf{x}_2$ are each marginally normally distributed does \textit{not} guarantee that their concatenated pair $(\mathbf{x}_1, \mathbf{x}_2)$ follows a joint multivariate normal distribution. In the GJE framework, we explicitly parameterize and optimize the joint Gaussian likelihood to enforce this joint normality. Further, for a jointly normal vector, components that are uncorrelated (cross-covariance is zero) are strictly independent.

\subsection{Marginal and Conditional Distributions}
Suppose we partition a $d$-dimensional Gaussian vector $\mathbf{x}$ into two blocks $\mathbf{x}_1$ and $\mathbf{x}_2$, with corresponding partitions for the mean $\boldsymbol{\mu}$ and covariance $\Sigma$:
\begin{equation}
    \mathbf{x} = \begin{bmatrix} \mathbf{x}_1 \\ \mathbf{x}_2 \end{bmatrix}, \quad
    \boldsymbol{\mu} = \begin{bmatrix} \boldsymbol{\mu}_1 \\ \boldsymbol{\mu}_2 \end{bmatrix}, \quad
    \Sigma = \begin{bmatrix} \Sigma_{11} & \Sigma_{12} \\ \Sigma_{21} & \Sigma_{22} \end{bmatrix}
\end{equation}

\textbf{Marginalization:} to obtain the marginal distribution of a subset of variables, one simply drops the irrelevant variables from the mean vector and covariance matrix. Thus, the marginal distribution of $\mathbf{x}_1$ is strictly $\mathcal{N}(\boldsymbol{\mu}_1, \Sigma_{11})$.

\textbf{Conditioning:} the distribution of $\mathbf{x}_1$ conditional on $\mathbf{x}_2 = \mathbf{a}$ is also multivariate normal, denoted $\mathbf{x}_1 | \mathbf{x}_2 = \mathbf{a} \sim \mathcal{N}(\bar{\boldsymbol{\mu}}, \bar{\Sigma})$, where:
\begin{equation}\label{eq:Gaussian_conditional_dist}
\begin{aligned}
    \bar{\boldsymbol{\mu}} &= \boldsymbol{\mu}_1 + \Sigma_{12} \Sigma_{22}^{-1} (\mathbf{a} - \boldsymbol{\mu}_2) \\
    \bar{\Sigma} &= \Sigma_{11} - \Sigma_{12} \Sigma_{22}^{-1} \Sigma_{21}
\end{aligned}
\end{equation}
A complete step-by-step derivation of this conditional distribution can be found later in Appendix Section.\ref{app:derivation_of_conditional_distribution}.

\textit{Geometric Intuition (The Homoscedasticity Trap):} it is interesting to examine the role of the observed variable $\mathbf{a}$ in Eq.\ref{eq:Gaussian_conditional_dist}. In the formula for the conditional mean $\bar{\boldsymbol{\mu}}$, the variable $\mathbf{a}$ is explicitly multiplied by the matrix block $\Sigma_{12} \Sigma_{22}^{-1}$. Consequently, the mean strictly depends on the observed condition; as $\mathbf{a}$ changes, the mean shifts, forming a linear regression path. 

Conversely, look at the equation for the conditional variance $\bar{\Sigma}$ (the Schur complement). The variable $\mathbf{a}$ is entirely absent. The matrices $\Sigma_{11}$, $\Sigma_{12}$, $\Sigma_{22}$, and $\Sigma_{21}$ are just static blocks of the global, fixed joint covariance matrix $\Sigma$. Because there is no $\mathbf{a}$ anywhere in that variance formula, the conditional variance always evaluates to a single, constant matrix. 

In statistics, this property of the Gaussian distribution is called \textit{homoscedasticity} (constant variance). Geometrically, this means if we have a global 2D Gaussian density (a giant oval) and we slice it vertically at different observation points (e.g. $x=-0.5$, $x=0$, or $x=0.9$), the center of the slice will slide up and down the regression line (due to $\bar{\boldsymbol{\mu}}$), but the thickness of the slice (the variance $\bar{\Sigma}$) will be exactly the same every single time. This rigid mathematical property highlights exactly why a single global Gaussian covariance fails to model complex, multi-modal manifolds where the underlying variance changes dynamically across the input space.

\subsection{Information-Theoretic Properties}
The multivariate normal distribution has profound connections to information theory, which theoretically justifies its use in self-supervised objective functions. 

\paragraph{Differential Entropy}
The differential entropy of a $d$-dimensional multivariate normal distribution with probability density function $p(\mathbf{x})$ is defined as the negative integral of the log-density. Evaluating this integral yields:
\begin{align} \label{eq:Gaussian_differential_entropy}
    H(\mathbf{x}) &= -\int_{-\infty}^\infty \dots \int_{-\infty}^\infty p(\mathbf{x}) \ln p(\mathbf{x}) d\mathbf{x} \nonumber \\
    &= \frac{1}{2}\ln|2\pi e\Sigma| = \frac{d}{2}(1 + \ln(2\pi)) + \frac{1}{2}\ln|\Sigma|
\end{align}
measured in \textit{nats}. This reveals that the log-determinant penalty in the GJE loss function (Eq.\ref{eq:primal_gje_loss} and Eq.\ref{eq:GJE_cond_loss}) directly regularizes the entropy of the learned representations.

\paragraph{Kullback-Leibler (KL) Divergence}
The KL divergence measures the distance from a reference probability distribution $\mathcal{N}_1(\boldsymbol{\mu}_1, \Sigma_1)$ to a target distribution $\mathcal{N}_0(\boldsymbol{\mu}_0, \Sigma_0)$. For non-singular matrices $\Sigma_1$ and $\Sigma_0$, it is given by:
\begin{equation}
    D_{\text{KL}}(\mathcal{N}_0 \parallel \mathcal{N}_1) = \frac{1}{2} \left\{ \text{tr}(\Sigma_1^{-1} \Sigma_0) + (\boldsymbol{\mu}_1 - \boldsymbol{\mu}_0)^T \Sigma_1^{-1} (\boldsymbol{\mu}_1 - \boldsymbol{\mu}_0) - d + \ln \frac{|\Sigma_1|}{|\Sigma_0|} \right\}
\end{equation}
where $\text{tr}(\cdot)$ is the trace operator, $\ln(\cdot)$ is the natural logarithm, and $d$ is the dimension of the vector space. Dividing this expression by $\ln(2)$ yields the divergence in \textit{bits}. 

Because our GJE formulation anchors the latent representations to zero-mean distributions ($\boldsymbol{\mu}_1 = \boldsymbol{\mu}_0 = \mathbf{0}$), the quadratic mean-difference term mathematically vanishes. In this specialized case, the divergence simplifies beautifully to:
\begin{equation}
    D_{\text{KL}}(\mathcal{N}_0 \parallel \mathcal{N}_1) = \frac{1}{2} \left\{ \text{tr}(\Sigma_1^{-1} \Sigma_0) - d + \ln \frac{|\Sigma_1|}{|\Sigma_0|} \right\}
\end{equation}

\paragraph{Mutual Information}
The mutual information of a multivariate normal distribution is a special case of the Kullback-Leibler divergence where the joint distribution $P$ is compared against $Q$, the product of its marginal distributions. It can be derived directly step-by-step from its fundamental relationship with differential entropy. For a jointly normal partitioned vector consisting of context embeddings $\mathbf{x}_1 \in \mathbb{R}^{d_1}$ and target embeddings $\mathbf{x}_2 \in \mathbb{R}^{d_2}$, the mutual information is defined as the sum of their marginal entropies minus their joint entropy:
\begin{equation} \label{eq:Gaussian_dist_mutual_infomation}
\begin{aligned}
    I(\mathbf{x}_1, \mathbf{x}_2) &= H(\mathbf{x}_1) + H(\mathbf{x}_2) - H(\mathbf{x}_1, \mathbf{x}_2) \\
    &= \left( \frac{d_1}{2}\ln(2\pi e) + \frac{1}{2}\ln\det(\Sigma_{11}) \right) + \left( \frac{d_2}{2}\ln(2\pi e) + \frac{1}{2}\ln\det(\Sigma_{22}) \right) \\
    &\quad - \left( \frac{d_1 + d_2}{2}\ln(2\pi e) + \frac{1}{2}\ln\det(\Sigma) \right) \\
    &= \frac{1}{2}\ln\det(\Sigma_{11}) + \frac{1}{2}\ln\det(\Sigma_{22}) - \frac{1}{2}\ln\det(\Sigma) \\
    &= \frac{1}{2} \ln \left( \frac{\det(\Sigma_{11}) \det(\Sigma_{22})}{\det(\Sigma)} \right)
\end{aligned}
\end{equation}
where $\Sigma$ is the full joint covariance block, and $\Sigma_{11}, \Sigma_{22}$ are the marginal auto-covariances corresponding to the respective partitions. 

This formula provides powerful theoretical validation for our GJE framework. By minimizing the joint complexity penalty $\frac{1}{2}\log|\Sigma|$ from our NLL objective (Eq.\ref{eq:primal_gje_loss}), the network is mathematically driven to maximize the mutual information between the context and target embeddings, provided the marginal volumes ($|\Sigma_{11}|$ and $|\Sigma_{22}|$) do not collapse.

\subsection{Log-Likelihood and Sampling}
If the mean and covariance matrix are known, the log-likelihood of an observed vector $\mathbf{x}$ is the log of the PDF:
\begin{equation}
    \ln L(\mathbf{x}) = -\frac{1}{2} \left[ \ln(|\Sigma|) + (\mathbf{x} - \boldsymbol{\mu})^\top \Sigma^{-1} (\mathbf{x} - \boldsymbol{\mu}) + d\ln(2\pi) \right]
\end{equation}
which is distributed as a generalized chi-squared variable.

Finally, to synthetically draw (sample) values from this generative distribution, one can compute the Cholesky decomposition $\mathbf{A}\mathbf{A}^\top = \Sigma$. By drawing a vector of independent standard normal variables $\mathbf{z} \sim \mathcal{N}(\textbf{0}, I)$, a valid sample from the target joint distribution is obtained via the affine transformation $\mathbf{x} = \boldsymbol{\mu} + \mathbf{A}\mathbf{z}$.

\section{Block Matrix Inversion and Determinant}
\label{app:block_matrix}
For a block matrix $P$ partitioned into four conformable blocks:
\begin{equation}
P = \begin{bmatrix} A & B \\ C & D \end{bmatrix}
\end{equation}
where $A$ and $D$ are square blocks of arbitrary size.

\subsection{Block Matrix Inversion}
If the matrix $A$ is invertible, we define the Schur complement of $A$ in $P$ as $P/A = D - CA^{-1}B$. Provided this Schur complement is also invertible, the matrix can be inverted blockwise as:
\begin{equation}
\begin{bmatrix} A & B \\ C & D \end{bmatrix}^{-1} = 
\begin{bmatrix} 
A^{-1} + A^{-1}B(D - CA^{-1}B)^{-1}CA^{-1} & -A^{-1}B(D - CA^{-1}B)^{-1} \\ 
-(D - CA^{-1}B)^{-1}CA^{-1} & (D - CA^{-1}B)^{-1} 
\end{bmatrix}
\end{equation}

Equivalently, by permuting the blocks and focusing on $D$, we define the Schur complement of $D$ in $P$ as $P/D = A - BD^{-1}C$. If $D$ and $P/D$ are invertible, the inverse is:
\begin{equation}
\begin{bmatrix} A & B \\ C & D \end{bmatrix}^{-1} = 
\begin{bmatrix} 
(A - BD^{-1}C)^{-1} & -(A - BD^{-1}C)^{-1}BD^{-1} \\ 
-D^{-1}C(A - BD^{-1}C)^{-1} & D^{-1} + D^{-1}C(A - BD^{-1}C)^{-1}BD^{-1} 
\end{bmatrix}
\end{equation}

By the symmetry of the block inversion formula, if the inverse matrix $P^{-1}$ is partitioned conformally into blocks $E, F, G,$ and $H$, the inverse of any principal submatrix can be computed from the corresponding blocks. For example, $A^{-1} = E - FH^{-1}G$.

\paragraph{Inverse of the joint covariance matrix (Eq.\ref{eq:primal_GJE_cov_inverse})} By substituting $A=K_{cc}$, $B=K_{ct}$, $C=K_{tc}$, and $D=K_{tt}$ into the standard block matrix inversion identities, we can express the precision matrix $\Lambda$ in two equivalent forms depending on which Schur complement we factorize.

\textbf{Way 1: using the Schur complement of the context covariance} \\
Let $S_{cc} = K_{tt} - K_{tc} K_{cc}^{-1} K_{ct}$ be the Schur complement of $K_{cc}$ (which corresponds exactly to the conditional covariance $\Sigma_{t|c}$, see Eq.\ref{eq:GJE_conditional_mean_cov_final} in Appendix.\ref{app:derivation_of_conditional_distribution}). The inverse matrix expands as:
\begin{equation} \label{eq:GJE_cov_inv1}
    \begin{bmatrix} K_{cc} & K_{ct} \\ K_{tc} & K_{tt} \end{bmatrix}^{-1} 
    = \begin{bmatrix} \Lambda_{cc} & \Lambda_{ct} \\ \Lambda_{tc} & \Lambda_{tt} \end{bmatrix}
    = \begin{bmatrix} 
    K_{cc}^{-1} + K_{cc}^{-1} K_{ct} S_{cc}^{-1} K_{tc} K_{cc}^{-1} & -K_{cc}^{-1} K_{ct} S_{cc}^{-1} \\ 
    -S_{cc}^{-1} K_{tc} K_{cc}^{-1} & S_{cc}^{-1} 
    \end{bmatrix}
\end{equation}

\textbf{Way 2: using the Schur complement of the target covariance} \\
Alternatively, let $S_{tt} = K_{cc} - K_{ct} K_{tt}^{-1} K_{tc}$ be the Schur complement of $K_{tt}$ (which corresponds to the conditional covariance for the reverse prediction, $\Sigma_{c|t}$). The inverse matrix expands as:
\begin{equation} \label{eq:GJE_cov_inv2}
    \begin{bmatrix} K_{cc} & K_{ct} \\ K_{tc} & K_{tt} \end{bmatrix}^{-1} 
    = \begin{bmatrix} \Lambda_{cc} & \Lambda_{ct} \\ \Lambda_{tc} & \Lambda_{tt} \end{bmatrix}
    = \begin{bmatrix} 
    S_{tt}^{-1} & -S_{tt}^{-1} K_{ct} K_{tt}^{-1} \\ 
    -K_{tt}^{-1} K_{tc} S_{tt}^{-1} & K_{tt}^{-1} + K_{tt}^{-1} K_{tc} S_{tt}^{-1} K_{ct} K_{tt}^{-1} 
    \end{bmatrix}
\end{equation}

The "Way 2" formulation is exactly what allows us to isolate the $z_t$ terms to complete the square when deriving the conditional distribution $p(z_t | z_c)$ (Gaussian conditioning see Appendix.\ref{app:derivation_of_conditional_distribution}). Specifically, looking at the bottom right block ($\Lambda_{tt} = S_{cc}^{-1}$), we can immediately see where the conditional covariance $\Sigma_{t|c}$ comes from.

\subsection{Block Matrix Determinant}
The determinant of the block matrix $P$ can also be factorized using the Schur complement. If $A$ is invertible, the determinant is computed as:
\begin{equation}
\det \begin{bmatrix} A & B \\ C & D \end{bmatrix} = \det(A) \det(D - CA^{-1}B)
\end{equation}

Conversely, if $D$ is invertible, the determinant can be computed as:
\begin{equation}
\det \begin{bmatrix} A & B \\ C & D \end{bmatrix} = \det(D) \det(A - BD^{-1}C)
\end{equation}

Further, if the blocks are square matrices of the \textit{same} size, specific commutativity conditions allow for simpler formulas. For instance, if $C$ and $D$ commute ($CD = DC$), then:
\begin{equation}
\det \begin{bmatrix} A & B \\ C & D \end{bmatrix} = \det(AD - BC)
\end{equation}
Similarly, if $A=D$ and $B=C$ (even if $A$ and $B$ do not commute), the determinant simplifies to the product of characteristic polynomials:
\begin{equation}
\det \begin{bmatrix} A & B \\ B & A \end{bmatrix} = \det(A - B) \det(A + B)
\end{equation}

\section{Derivation of the Conditional Distribution via Affine Transform, and Joint Distribution and Completing the Square} \label{app:derivation_of_conditional_distribution} 

To provide a rigorous mathematical foundation, we present two distinct but consistent methods to derive the conditional distribution $p(\mathbf{x}_1 | \mathbf{x}_2)$ of a general multivariate Gaussian. 

Let the concatenated vector $\mathbf{x} = [\mathbf{x}_1^T, \mathbf{x}_2^T]^T$ follow a joint Gaussian distribution $\mathcal{N}(\boldsymbol{\mu}, \Sigma)$, partitioned as:
\begin{equation*}
    \boldsymbol{\mu} = \begin{bmatrix} \boldsymbol{\mu}_1 \\ \boldsymbol{\mu}_2 \end{bmatrix}, \quad
    \Sigma = \begin{bmatrix} \Sigma_{11} & \Sigma_{12} \\ \Sigma_{21} & \Sigma_{22} \end{bmatrix}, \quad 
    \Lambda = \Sigma^{-1} = \begin{bmatrix} \Lambda_{11} & \Lambda_{12} \\ \Lambda_{21} & \Lambda_{22} \end{bmatrix}
\end{equation*}

\subsection{Method 1: Derivation via Affine Transformation} 

We can derive the conditional distribution by constructing a new random vector $\mathbf{z}$ that is strictly independent of $\mathbf{x}_2$ using an affine transformation.

Let us define a new random vector $\mathbf{z} = \mathbf{x}_1 - \Sigma_{12}\Sigma_{22}^{-1}\mathbf{x}_2$. To formalize this, we express the concatenated vector $[\mathbf{z}^T, \mathbf{x}_2^T]^T$ as a direct linear matrix transformation of the original partitioned joint vector:
\begin{equation*}
    \begin{bmatrix} \mathbf{z} \\ \mathbf{x}_2 \end{bmatrix} = \begin{bmatrix} \mathbf{I} & -\Sigma_{12}\Sigma_{22}^{-1} \\ \mathbf{0} & \mathbf{I} \end{bmatrix} \begin{bmatrix} \mathbf{x}_1 \\ \mathbf{x}_2 \end{bmatrix}
\end{equation*}

Because the original vector $[\mathbf{x}_1^T, \mathbf{x}_2^T]^T$ is jointly Gaussian by definition, and the multivariate Gaussian distribution is closed under affine transformations, this matrix multiplication strictly guarantees that the new concatenated vector $[\mathbf{z}^T, \mathbf{x}_2^T]^T$ remains jointly Gaussian. 

Since they are jointly Gaussian, we determine their dependency by computing their cross-covariance:
\begin{align*}
    \text{Cov}(\mathbf{z}, \mathbf{x}_2) &= \text{Cov}(\mathbf{x}_1 - \Sigma_{12}\Sigma_{22}^{-1}\mathbf{x}_2, \mathbf{x}_2) \\
    &= \text{Cov}(\mathbf{x}_1, \mathbf{x}_2) - \Sigma_{12}\Sigma_{22}^{-1} \text{Cov}(\mathbf{x}_2, \mathbf{x}_2) \\
    &= \Sigma_{12} - \Sigma_{12}\Sigma_{22}^{-1}\Sigma_{22} = \mathbf{0}
\end{align*}
For jointly Gaussian vectors, a cross-covariance of zero implies strict statistical independence. Thus, $\mathbf{z}$ and $\mathbf{x}_2$ are entirely independent. 

Next, we evaluate the marginal mean and covariance of this new variable $\mathbf{z}$:
\begin{align*}
    \mathbb{E}[\mathbf{z}] &= \mathbb{E}[\mathbf{x}_1] - \Sigma_{12}\Sigma_{22}^{-1}\mathbb{E}[\mathbf{x}_2] = \boldsymbol{\mu}_1 - \Sigma_{12}\Sigma_{22}^{-1}\boldsymbol{\mu}_2 \\
    \text{Cov}(\mathbf{z}, \mathbf{z}) &= \text{Cov}(\mathbf{x}_1 - \Sigma_{12}\Sigma_{22}^{-1}\mathbf{x}_2, \mathbf{x}_1 - \Sigma_{12}\Sigma_{22}^{-1}\mathbf{x}_2) \\
    &= \Sigma_{11} - \Sigma_{12}\Sigma_{22}^{-1}\Sigma_{21} - \Sigma_{12}\Sigma_{22}^{-1}\Sigma_{21} + \Sigma_{12}\Sigma_{22}^{-1}\Sigma_{22}\Sigma_{22}^{-1}\Sigma_{21} \\
    &= \Sigma_{11} - \Sigma_{12}\Sigma_{22}^{-1}\Sigma_{21}
\end{align*}

Now, we can express our target variable $\mathbf{x}_1$ in terms of our independent variable: $\mathbf{x}_1 = \mathbf{z} + \Sigma_{12}\Sigma_{22}^{-1}\mathbf{x}_2$. Conditioning this equation on the observation $\mathbf{x}_2 = \mathbf{a}$ gives:
\begin{equation*}
    \mathbf{x}_1 | (\mathbf{x}_2 = \mathbf{a}) = \mathbf{z} | (\mathbf{x}_2 = \mathbf{a}) + \Sigma_{12}\Sigma_{22}^{-1}\mathbf{a}
\end{equation*}
Because $\mathbf{z}$ is independent of $\mathbf{x}_2$, observing $\mathbf{x}_2 = \mathbf{a}$ provides no information about $\mathbf{z}$. The conditional distribution of $\mathbf{z}$ is simply its marginal distribution. Since adding a constant vector to a Gaussian vector results in another Gaussian vector with a shifted mean, we conclude $\mathbf{x}_1 | \mathbf{x}_2 = \mathbf{a}$ is exactly Gaussian with parameters:
\begin{empheq}[box=\fbox]{equation}
\begin{aligned}
    \bar{\boldsymbol{\mu}} &= \boldsymbol{\mu}_1 + \Sigma_{12}\Sigma_{22}^{-1}(\mathbf{a} - \boldsymbol{\mu}_2) \\
    \bar{\Sigma} &= \Sigma_{11} - \Sigma_{12}\Sigma_{22}^{-1}\Sigma_{21}
\end{aligned}
\label{eq:gen_cond_mean_cov}
\end{empheq}

\subsection{Method 2: Derivation via Joint Distribution and Completing the Square}
Alternatively, we utilize Bayes Theorem: $p(\mathbf{x}_1 | \mathbf{x}_2) \propto p(\mathbf{x}_1, \mathbf{x}_2)$. Because the marginal distribution $p(\mathbf{x}_2)$ acts as a constant with respect to $\mathbf{x}_1$, the terms in the joint Gaussian exponent that depend on $\mathbf{x}_1$ must perfectly match the terms in the conditional Gaussian exponent. 

The joint PDF is proportional to $\exp\left( -\frac{1}{2} (\mathbf{x} - \boldsymbol{\mu})^T \Lambda (\mathbf{x} - \boldsymbol{\mu}) \right)$. Expanding the quadratic form and isolating terms containing $\mathbf{x}_1$:
\begin{equation} \label{eq:gen_expanded_quad}
    (\mathbf{x}_1 - \boldsymbol{\mu}_1)^T \Lambda_{11} (\mathbf{x}_1 - \boldsymbol{\mu}_1) + 2 (\mathbf{x}_1 - \boldsymbol{\mu}_1)^T \Lambda_{12} (\mathbf{x}_2 - \boldsymbol{\mu}_2) + \dots
\end{equation}

We must manipulate these $\mathbf{x}_1$-dependent terms into the standard form of a conditional Gaussian exponent:
\begin{equation} \label{eq:gen_target_quad}
    (\mathbf{x}_1 - \bar{\boldsymbol{\mu}})^T \bar{\Sigma}^{-1} (\mathbf{x}_1 - \bar{\boldsymbol{\mu}}) = \mathbf{x}_1^T \bar{\Sigma}^{-1} \mathbf{x}_1 - 2 \mathbf{x}_1^T \bar{\Sigma}^{-1} \bar{\boldsymbol{\mu}} + \bar{\boldsymbol{\mu}}^T \bar{\Sigma}^{-1} \bar{\boldsymbol{\mu}}
\end{equation}

By matching the coefficients between our expanded joint equation (Eq.\ref{eq:gen_expanded_quad}) and our target conditional equation (Eq.\ref{eq:gen_target_quad}), we solve for the conditional parameters:

\textbf{1. Matching the quadratic term ($\mathbf{x}_1^T [\cdot] \mathbf{x}_1$):}
\begin{empheq}[box=\fbox]{equation}
\bar{\Sigma}^{-1} = \Lambda_{11} \implies \bar{\Sigma} = \Lambda_{11}^{-1}
\label{eq:gen_cond_var_prec}
\end{empheq}

\textbf{2. Matching the linear term ($\mathbf{x}_1^T [\cdot]$):}
From Eq.\ref{eq:gen_expanded_quad}, the linear terms in $\mathbf{x}_1$ expand to $-2\mathbf{x}_1^T \Lambda_{11} \boldsymbol{\mu}_1 + 2\mathbf{x}_1^T \Lambda_{12} (\mathbf{x}_2 - \boldsymbol{\mu}_2)$. Equating this to the target linear term $-2\mathbf{x}_1^T \bar{\Sigma}^{-1} \bar{\boldsymbol{\mu}}$ and substituting $\bar{\Sigma}^{-1} = \Lambda_{11}$ yields:
\begin{align*}
    -2\mathbf{x}_1^T \Lambda_{11} \bar{\boldsymbol{\mu}} &= -2\mathbf{x}_1^T \Lambda_{11} \boldsymbol{\mu}_1 + 2\mathbf{x}_1^T \Lambda_{12} (\mathbf{x}_2 - \boldsymbol{\mu}_2) \\
    -\Lambda_{11} \bar{\boldsymbol{\mu}} &= -\Lambda_{11} \boldsymbol{\mu}_1 + \Lambda_{12} (\mathbf{x}_2 - \boldsymbol{\mu}_2)
\end{align*}
Therefore:
\begin{empheq}[box=\fbox]{equation}
\bar{\boldsymbol{\mu}} = \boldsymbol{\mu}_1 - \Lambda_{11}^{-1} \Lambda_{12} (\mathbf{x}_2 - \boldsymbol{\mu}_2)
\label{eq:gen_cond_mean_prec}
\end{empheq}

\textbf{Consistency of both derivation methods:} by the standard block matrix inversion identities (see Appendix.\ref{app:block_matrix}), $\Lambda_{11}^{-1} = \Sigma_{11} - \Sigma_{12} \Sigma_{22}^{-1} \Sigma_{21}$ and $\Lambda_{12} = -\Lambda_{11} \Sigma_{12} \Sigma_{22}^{-1}$. 

First, substituting the identity for $\Lambda_{11}^{-1}$ directly into Eq.\ref{eq:gen_cond_var_prec} immediately verifies the conditional covariance matches Method 1 (Eq.\ref{eq:gen_cond_mean_cov}):
\begin{equation*}
    \bar{\Sigma} = \Lambda_{11}^{-1} = \Sigma_{11} - \Sigma_{12} \Sigma_{22}^{-1} \Sigma_{21}
\end{equation*}

Second, substituting both identities into the conditional mean from Eq.\ref{eq:gen_cond_mean_prec} reveals:
\begin{equation*}
    \bar{\boldsymbol{\mu}} = \boldsymbol{\mu}_1 - \Lambda_{11}^{-1} (-\Lambda_{11} \Sigma_{12} \Sigma_{22}^{-1}) (\mathbf{x}_2 - \boldsymbol{\mu}_2) = \boldsymbol{\mu}_1 + \Sigma_{12} \Sigma_{22}^{-1} (\mathbf{x}_2 - \boldsymbol{\mu}_2)
\end{equation*}
which proves that, both derivation methods mathematically arrive at the exact same conditional mean and covariance.

\subsection{Application to Gaussian Joint Embeddings (GJE)}
In the specific context of the GJE framework presented in Section.\ref{sec:GJE}, we define our variables as $Z = [z_c^T, z_t^T]^T$. We map the target variable $\mathbf{x}_1 \rightarrow z_t$, the condition variable $\mathbf{x}_2 \rightarrow z_c$, and the observation $\mathbf{a} \rightarrow z_c$. 

Specifically, GJE anchors the latent representations to a \textbf{zero-mean prior} (Eq.\ref{eq:primal_GJE_joint_Gaussian_pdf}), setting $\boldsymbol{\mu}_1 = \mathbf{0}$ and $\boldsymbol{\mu}_2 = \mathbf{0}$. The covariance blocks map as $\Sigma_{11} \rightarrow \Sigma_{tt}$, $\Sigma_{22} \rightarrow \Sigma_{cc}$, $\Sigma_{12} \rightarrow \Sigma_{tc}$, and $\Sigma_{21} \rightarrow \Sigma_{ct}$.
Applying these zero-mean constraints to the general formulas derived above instantly yields the closed-form GJE predictive conditional mean ($\mu_{t|c}$) and epistemic uncertainty ($\Sigma_{t|c}$):

\textbf{Via Covariance Blocks} (from Eq.\ref{eq:gen_cond_mean_cov}):
\begin{empheq}[box=\fbox]{equation}
\begin{aligned}
    \mu_{t|c} &= \Sigma_{tc} \Sigma_{cc}^{-1} z_c \\
    \Sigma_{t|c} &= \Sigma_{tt} - \Sigma_{tc} \Sigma_{cc}^{-1} \Sigma_{ct}
\end{aligned}
\label{eq:GJE_conditional_mean_cov_final}
\end{empheq}

\textbf{Via Precision Blocks} (from Eq.\ref{eq:gen_cond_mean_prec} and Eq.\ref{eq:gen_cond_var_prec}):
\begin{empheq}[box=\fbox]{equation}
\begin{aligned}
    \mu_{t|c} &= -\Lambda_{tt}^{-1} \Lambda_{tc} z_c \\
    \Sigma_{t|c} &= \Lambda_{tt}^{-1}
\end{aligned}
\end{empheq}

\section{Primal-GJE Learning Objective: Joint vs. Conditional Likelihood} \label{app:primal_GJE_joint_vs_conditional_objectives}

In Primal-GJE (Section.\ref{subsec:primal_GJE}), by modeling the context/target representations as jointly Gaussian, our optimization objective can be formulated from two distinct mathematical perspectives: the \textit{generative} (joint) perspective and the \textit{predictive} (conditional) perspective. 

\subsection{The Generative Objective (Joint NLL)}
To learn the optimal encoder weights $\theta$ and $\theta'$ directly, we minimize the Negative Log Likelihood (NLL) of the joint distribution. Let $z_i = [z_{c,i}^T, z_{t,i}^T]^T \in \mathbb{R}^d$ represent a single concatenated context-target embedding pair from the batch. The probability density function (pdf) of our zero-mean joint Gaussian is:
\begin{equation} \label{eq:GJE_joint_pdf1}
    p(z_{c,i}, z_{t,i}) = \frac{1}{\sqrt{(2\pi)^{d} |C_{joint}|}} \exp\left(-\frac{1}{2} \begin{bmatrix} z_{c,i} \\ z_{t,i} \end{bmatrix}^T C_{joint}^{-1} \begin{bmatrix} z_{c,i} \\ z_{t,i} \end{bmatrix} \right)
\end{equation}
Taking the negative natural logarithm and dropping the constant term, the joint loss function evaluated at this single data point becomes:
\begin{equation} \label{eq:primal_GJE_joint_loss_single_data_point}
\begin{aligned}
\mathcal{L}_{\text{joint}}^{(i)}(\theta, \theta') 
&= \underbrace{\frac{1}{2} z_i^T C_{joint}^{-1} z_i}_{\text{data-fit term}} + \underbrace{\frac{1}{2} \log |C_{joint}|}_{\text{regularizer}} \\
&= \frac{1}{2}
\begin{bmatrix} z_{c,i} \\ z_{t,i} \end{bmatrix}^T
\begin{bmatrix} C_{cc} & C_{ct} \\ C_{tc} & C_{tt} \end{bmatrix}^{-1}
\begin{bmatrix} z_{c,i} \\ z_{t,i} \end{bmatrix}
+ \frac{1}{2} \log \left|
\begin{bmatrix} C_{cc} & C_{ct} \\ C_{tc} & C_{tt} \end{bmatrix}
\right|
\end{aligned}
\end{equation}
The second term, i.e. the log-determinant complexity penalty $\log |C_{joint}|$, serves as \textit{the geometric regularizer} which prevents the volume of the embeddings from expanding infinitely to trivially minimize the data-fitting term.
Note that this single-sample loss is exactly the summand of the full empirical batch objective $\mathcal{L}_{\text{Primal-GJE}}$ (Eq.\ref{eq:primal_gje_loss}). It is this strict \textit{generative joint objective} $p(z_c, z_t)$ that mathematically underpins the Primal-GJE formulation. This stands in contrast to the Dual-GJE (GPJE) formulation derived earlier in Section \ref{subsec:dual_GJE}, which optimizes the Gaussian Process marginal likelihood - a strictly \textit{conditional} distribution\footnote{Se  Eq.(2.29) and Eq.(2.30) in the GP for ML book \cite{rasmussen2006gaussian}.} $p(Z_t \mid Z_c)$ that models the target features given the context inputs, thereby lacking a native marginal regularizer for the context space.

\subsection{The Predictive Objective (Conditional NLL)}
Alternatively, as $z_c$ and $z_t$ are jointly Gaussian, the conditional distribution of the target given the context, $p(z_t | z_c)$, is guaranteed to be Gaussian. Using the Schur complement (derivations see Appendix.\ref{app:derivation_of_conditional_distribution}), the conditional distribution is exactly $\mathcal{N}(\mu_{t|c}, \Sigma_{t|c})$ with\footnote{Note that a joint Gaussian distribution is strictly \textit{unimodal} at any conditional slice.}:
\begin{equation} \label{eq:GJE_conditional_mean_cov_final_replicated} \tag{cc.Eq.\ref{eq:primal_GJE_mean_and_cov}}
\begin{aligned}
    \mu_{t|c} &= C_{tc} C_{cc}^{-1} z_c \\ 
    \Sigma_{t|c} &= C_{tt} - C_{tc} C_{cc}^{-1} C_{ct}  
\end{aligned}
\end{equation}
This allows us to formulate a conditional objective, $\mathcal{L}_{\text{cond}} \propto - \log p(z_t | z_c)$, i.e. taking the negative log likelihood of $p(z_t | z_c)$:
\begin{equation} \label{eq:GJE_cond_loss}
\mathcal{L}_{\text{cond}}(\theta, \theta') = \frac{1}{2} (z_t - \mu_{t|c})^T \Sigma_{t|c}^{-1} (z_t - \mu_{t|c}) + \frac{1}{2} \log |\Sigma_{t|c}|
\end{equation}
This formulation is interesting: it reveals that the optimal "predictor" in our architecture is not a black-box neural network, but the closed-form conditional mean, accompanied by a dynamic covariance matrix that quantifies epistemic uncertainty. 

Further, it mathematically bridges our framework with classic JEPA. If one assumes the conditional covariance is a fixed identity matrix ($\Sigma_{t|c} = I$) and drops the resulting constant regularization term, Eq.\ref{eq:GJE_cond_loss} trivially degenerates into a standard Mean Squared Error (MSE) objective: $\mathcal{L}_{\text{MSE}} = ||z_t - g(z_c)||^2$, which is exactly the heuristic objective of standard JEPA. Thus, dual-GJE is a strict, probabilistically grounded generalization of JEPA.

\subsection{Joint NLL: Closed-Form Inference via Block Inversion} 
Once the empirical covariance matrix is computed from the training representations, the optimal "predictor" falls out for free. The matrix blocks $[C_{cc}, C_{ct}, C_{tc}, C_{tt}]$ are fixed. Given a new test context $z_c^*$, we wish to find $p(z_t^* \mid z_c^*)$. By the definition of conditional probability for joint Gaussians, we can derive this using block matrix inversion and the Schur complement.

We partition the inverse covariance matrix (precision matrix) $\Lambda = C_{joint}^{-1}$ as\footnote{This equation isn't actually calculating the inverse; it simply defines the block variables. Letting the Schur complement be $S = C_{tt} - C_{tc}C_{cc}^{-1}C_{ct}$, the blocks evaluate to: $\Lambda_{tt} = S^{-1}$, $\Lambda_{ct} = -C_{cc}^{-1}C_{ct}S^{-1}$, $\Lambda_{tc} = -S^{-1}C_{tc}C_{cc}^{-1}$, and $\Lambda_{cc} = C_{cc}^{-1} + C_{cc}^{-1}C_{ct}S^{-1}C_{tc}C_{cc}^{-1}$. For detailed mathematical derivations of block matrix inversion, see Appendix.\ref{app:block_matrix}.}:
\begin{equation} \label{eq:primal_GJE_cov_inverse}
\begin{aligned}
     C_{joint}^{-1} =
    \begin{bmatrix} C_{cc} & C_{ct} \\ C_{tc} & C_{tt} \end{bmatrix}^{-1} 
    &= \begin{bmatrix} \Lambda_{cc} & \Lambda_{ct} \\ \Lambda_{tc} & \Lambda_{tt} \end{bmatrix} \\
    &= \begin{bmatrix} 
    C_{cc}^{-1} + C_{cc}^{-1} C_{ct} (C_{tt} - C_{tc}C_{cc}^{-1}C_{ct})^{-1} C_{tc} C_{cc}^{-1} & -C_{cc}^{-1} C_{ct} (C_{tt} - C_{tc}C_{cc}^{-1}C_{ct})^{-1} \\ 
    -(C_{tt} - C_{tc}C_{cc}^{-1}C_{ct})^{-1} C_{tc} C_{cc}^{-1} & (C_{tt} - C_{tc}C_{cc}^{-1}C_{ct})^{-1} 
    \end{bmatrix} \\
    &= \begin{bmatrix} 
    (C_{cc} - C_{ct} C_{tt}^{-1} C_{tc})^{-1} & -(C_{cc} - C_{ct} C_{tt}^{-1} C_{tc})^{-1} C_{ct} C_{tt}^{-1} \\ 
    -C_{tt}^{-1} C_{tc} (C_{cc} - C_{ct} C_{tt}^{-1} C_{tc})^{-1} & C_{tt}^{-1} + C_{tt}^{-1} C_{tc} (C_{cc} - C_{ct} C_{tt}^{-1} C_{tc})^{-1} C_{ct} C_{tt}^{-1} 
    \end{bmatrix}
\end{aligned}
\end{equation}

The quadratic form $z^T C_{joint}^{-1} z$ in the joint generative exponential (Eq.\ref{eq:GJE_joint_pdf1}) can be expanded algebraically into $z_c^{*T} \Lambda_{cc} z_c^* + 2 z_t^{*T} \Lambda_{tc} z_c^* + z_t^{*T} \Lambda_{tt} z_t^*$. By isolating the terms involving the target $z_t^*$ and completing the square (the full step-by-step derivation is provided in Appendix \ref{app:derivation_of_conditional_distribution}), we find the conditional distribution is exactly Gaussian, $p(z_t^* \mid z_c^*) = \mathcal{N}(\mu_{t|c}, \Sigma_{t|c})$, with mean and covariance (Eq.\ref{eq:primal_GJE_mean_and_cov}):
\begin{align*}
    \mu_{t|c} &= C_{tc} C_{cc}^{-1} z_c^* \\
    \Sigma_{t|c} &= C_{tt} - C_{tc} C_{cc}^{-1} C_{ct}
\end{align*}
This elegantly establishes an exact, mathematically optimal linear projection for the prediction ($\mu_{t|c}$), permanently coupled with calibrated, dataset-driven epistemic uncertainty ($\Sigma_{t|c}$).

\subsection{Combating Collapse: Why Joint Optimization is Required}

While the conditional objective ($\mathcal{L}_{\text{cond}}$) elegantly maps to predictive tasks, optimizing a deterministic variant of it in isolation is vulnerable to representation collapse. By the laws of probability, the joint likelihood factors exactly into the conditional and marginal likelihoods:
\begin{equation} \label{eq:GJE_joint_pdf2}
    -\log p(z_c, z_t) = \underbrace{-\log p(z_t | z_c)}_{\mathcal{L}_{\text{cond}}} \underbrace{- \log p(z_c)}_{\text{marginal regularization}}
\end{equation}
This factorization provides a profound theoretical symmetry: the marginal term ($-\log p(z_c)$) prevents \textit{context} representation collapse, while the conditional term ($\mathcal{L}_{\text{cond}}$) prevents \textit{target} representation collapse. 
If we bring in $\mathcal{L}_{\text{cond}}$ from Eq.\ref{eq:GJE_cond_loss} and expand $\log p(z_c)$ in Eq.\ref{eq:GJE_joint_pdf2}:
\begin{equation} \label{eq:GJE_joint_pdf3}
\begin{aligned}
    \mathcal{L}_{\text{joint}} &= \mathcal{L}_{\text{cond}} 
    + \left( \frac{1}{2} z_c^T C_{cc}^{-1} z_c + \frac{1}{2} \log |C_{cc}| \right) \\
    &= \left( \frac{1}{2} (z_t - \mu_{t|c})^T \Sigma_{t|c}^{-1} (z_t - \mu_{t|c}) + \frac{1}{2} \log |\Sigma_{t|c}| \right)
    + \left( \frac{1}{2} z_c^T C_{cc}^{-1} z_c + \frac{1}{2} \log |C_{cc}| \right)
\end{aligned}
\end{equation}

This reveals two competing forces driving the loss: the \textit{data-fit terms} and the \textit{geometric regularizers}. First, the data-fit terms prevent representation collapse across both branches. If the context or target encoders collapse their respective embeddings, the corresponding auto-covariances ($C_{cc}$ or $C_{tt}$) approach singularity. The polynomial explosion of their inverses ($C_{cc}^{-1}$ in the marginal distribution, and the resulting $\Sigma_{t|c}^{-1}$ in the conditional distribution) drives the loss to infinity, strictly guarding against trivial collapsed solutions. 

Conversely, the complexity penalty terms, i.e. $\frac{1}{2} \log |\Sigma_{t|c}|$ and $\frac{1}{2} \log |C_{cc}|$, act as a \textit{geometric regularizer} bounded by deep information-theoretic principles (detailed in Appendix.\ref{app:Gaussian_dist}). They directly regularize the \textit{differential entropy} of the learned representations\footnote{For Gaussian distributions, the differential entropy $H(\mathbf{x}) \propto \log|\Sigma|$; see Eq.\ref{eq:Gaussian_differential_entropy} in Appendix.\ref{app:Gaussian_dist}.}, preventing the embeddings from infinitely expanding their variance to artificially minimize the data-fit term. Further, minimizing this joint log-determinant natively forces the network to maximize the \textit{Mutual Information} between the context and target spaces, provided the marginal volumes do not collapse (see Eq.\ref{eq:Gaussian_dist_mutual_infomation} in Appendix.\ref{app:Gaussian_dist}).

This also exposes the fundamental flaw in standard predictive architectures. Classic JEPA heuristically optimizes only a \textit{degenerate}, \textit{deterministic} version of the conditional loss (MSE), dropping the conditional covariance entirely, and it completely discards the marginal distribution of the context space ($-\log p(z_c)$). By throwing away the covariance-dependent terms, classic JEPA loses the native defense against entropy collapse and severs the link to mutual information maximization. Consequently, it is forced to rely on asymmetric stop-gradients and EMA updates to survive. By optimizing the full joint objective ($\mathcal{L}_{\text{joint}}$), Primal-GJE natively preserves these probabilistic and information-theoretic guards, ensuring a diverse, full-rank embedding space across both branches purely through the laws of probability.

\section{Representation Collapse in Symmetric Dual GJE} \label{app:dual_collapse_proof}

As discussed in Section.\ref{subsec:dual_GJE}, unlike our primal feature-space objectives, the dual objective of GPJE cannot be optimized symmetrically. In this appendix, we provide the formal mathematical proof demonstrating why optimizing the Dual-GJE objective via symmetric gradient descent instantly leads to catastrophic representation collapse, formally justifying the necessity of EMA target networks \cite{grill2020byol,assran2023ijepa, assran2023ijepa}.

Consider the exact Dual-GJE negative marginal log-likelihood objective (Eq.\ref{eq:dual_gje_loss1}):
\begin{equation} \label{eq:dual_gje_loss1_replicated} \tag{cc.Eq.\ref{eq:dual_gje_loss1}}
    \mathcal{L} = \frac{1}{2} \text{Tr}\left(Z_t^T K_{cc}^{-1} Z_t\right) + \frac{d_t}{2} \log |K_{cc}|
\end{equation}

If the weights of both the context encoder $E_\theta$ and the target encoder $E_{\theta'}$ are updated symmetrically \textit{without stop-gradients}, the optimization dynamics follow a deterministic two-step collapse trajectory.

\paragraph{Step 1: The Target Space Collapses to Zero.}
First, consider the gradients flowing into the target encoder $E_{\theta'}$. The only term in \ref{eq:dual_gje_loss1_replicated} containing the target embeddings $Z_t$ is the data-fit term (the trace term). Because the inverse Gram matrix $K_{cc}^{-1}$ is strictly positive-definite by definition, this trace operation evaluates to a pure, lower-bounded quadratic penalty. To minimize this term to its absolute theoretical minimum ($0$), the target encoder trivially learns to map all inputs to the origin.
Consequently, $Z_t \to \mathbf{0}$, and the trace term becomes exactly $0$.

\paragraph{Step 2: The Context Space Collapses to a Constant.}
Once the target space has collapsed ($Z_t = \mathbf{0}$), the data-fit expansive force is entirely neutralized. The objective degenerates completely to the Complexity Penalty:
\begin{equation}
    \mathcal{L}_{\text{degenerated}} = 0 + \frac{d_t}{2} \log |K_{cc}|
\end{equation}
Now, analyzing the gradients flowing into the context encoder $E_\theta$, the optimizer seeks to minimize $\log |K_{cc}|$. This is equivalent to shrinking the differential volume of the context Gram matrix to zero. To achieve this, the context encoder maps every input image to the exact \textit{same constant embedding vector}. As a result, $K_{cc}$ degenerates into a singular matrix of identical values (up to the $\epsilon I$ jitter), driving the log-determinant toward $-\infty$. 

Thus, the system achieves a loss of $-\infty$ by learning a completely uninformative, degenerate latent space ($Z_c \to \text{constant}$, $Z_t \to \mathbf{0}$).

\subsection{The Root Cause: The Asymmetry of the Conditional Likelihood}
This catastrophic failure stems from the inherent directional asymmetry of Gaussian Processes. A standard GP models the conditional\footnote{See Eq.(2.29) or Eq.(2.30) in \cite{rasmussen2006gaussian}} distribution of outputs given inputs ($X \to Y$). As established in Section \ref{subsec:dual_GJE}, the GP \textit{marginal likelihood} used in Dual-GJE mathematically evaluates a strictly \textit{conditional} Negative Log-Likelihood, $p(Z_t \mid Z_c)$. 

In traditional machine learning paradigms, this asymmetry is perfectly sound because the targets $Y$ are fixed, ground-truth observations. The conditional objective natively penalizes the volume and complexity of the \textit{input} space ($\log |K_{cc}|$) to enforce smoothness, but it applies no corresponding volume penalty to the \textit{output} space because empirical outputs are inherently immutable. 

In the self-supervised representation learning setting, however, \textit{the "targets" $Z_t$ are not fixed}; they are dynamically learned parameters. Because the conditional GP formulation ($p(Z_t \mid Z_c)$) inherently lacks the target-space marginal regularizer ($-\log p(Z_t)$) that would be present in a full joint likelihood formulation, the target embeddings feel no expansive force. This fundamental absence of output regularization leads directly to the Step 1 collapse, mathematically mandating the use of artificial target anchors like EMA.

\subsection{The Structural Resolution: EMA and Stop-Gradients}
This mathematical vulnerability formally validates the necessity of the architectural heuristics popularized by BYOL \cite{grill2020byol} and JEPA \cite{assran2023ijepa}. 
To successfully optimize the Dual-GJE objective, we must artificially enforce the assumption of the Gaussian Process: the targets must be treated as fixed observations. By maintaining the target encoder as an EMA of the context encoder and applying a strict stop-gradient to $Z_t$, the target embeddings act as a diverse, static topological anchor. Because the optimizer is mathematically barred from shifting $Z_t \to \mathbf{0}$, the trace penalty cannot be trivially bypassed. Instead, the trace term correctly functions as an expansive force, compelling the context embeddings $Z_c$ to spread out and strictly match the fixed pairwise diversity of the target memory bank, safely preventing system collapse.

Importantly, this vulnerability is entirely circumvented when transitioning the framework to the symmetric \textit{primal feature space}. Because primal models (like GMJE and HSIC) evaluate the volume of the feature covariance matrices directly, and natively incorporate symmetric regularizers that penalize the collapse of both spaces, they eliminate the need for asymmetric EMA architectures entirely.

\section{Training Dual-GJE: Kernel Optimization and EMA} \label{app:training_dual_gje}

In standard Gaussian Process regression, as detailed in classic texts such as Rasmussen \& Williams \cite{rasmussen2006gaussian}, training the GP model is equivalent to hyperparameter optimization. Given a covariance function $k_\phi$ parameterized by hyperparameters $\phi$ (e.g. length-scale $\ell$, observation noise $\sigma^2$), the optimal parameters $\phi^*$ are found by maximizing the log marginal likelihood\footnote{Eq.(2.29) or Eq.(2.30) in \cite{rasmussen2006gaussian}.} of the observed targets given the inputs. 

In the Dual-GJE framework, we extend this principle to deep representation learning. The context encoder weights $\theta$ act as the ultimate "hyperparameters" of the kernel, dynamically shaping the input space $Z_c = E_\theta(X_c)$ upon which the Gram matrix is computed. Thus, training the Dual-GJE model involves taking the gradients of the conditional Negative Log-Likelihood (NLL) with respect to both the kernel parameters $\phi$ and the context encoder weights $\theta$. 

As proven in Appendix.\ref{app:dual_collapse_proof}, optimizing the target encoder symmetrically under this conditional objective causes catastrophic collapse. Therefore, the target encoder $E_{\theta'}$ is detached from the computational graph and updated exclusively via an Exponential Moving Average (EMA). The exact $\mathcal{O}(N^3)$ procedure is detailed in Algo.\ref{alg:exact_dual_gje}.

\begin{algorithm}[H]
\caption{Training Exact Dual-GJE Model}
\label{alg:exact_dual_gje}
\begin{algorithmic}[1]
\REQUIRE Dataset $\mathcal{D}$, Context Encoder $E_\theta$, Target Encoder $E_{\theta'}$. Kernel function $k_\phi$ with parameters $\phi$. Batch size $N$, Learning rate $\eta$, EMA momentum $\tau$.

\WHILE{network has not converged}
    \STATE Sample a batch of $N$ augmented view pairs $\{(x_c^{(i)}, x_t^{(i)})\}_{i=1}^N \sim \mathcal{D}$
    
    \STATE \textbf{// 1. Forward Pass}
    \STATE Compute context embeddings: $Z_c = E_\theta(X_c) \in \mathbb{R}^{N \times d_c}$
    \STATE \textit{No-grad:} Compute target anchors: $Z_t = \text{stop\_grad}(E_{\theta'}(X_t)) \in \mathbb{R}^{N \times d_t}$

    \STATE \textbf{// 2. Exact Dual Covariance Matrix \& Loss}
    \STATE Compute $N \times N$ Gram matrix $K_{cc}$ over $Z_c$ using kernel $k_\phi$.
    \STATE Add jitter for numerical stability: $K_{cc} \leftarrow K_{cc} + \epsilon I_{N}$
    \STATE Compute conditional NLL loss (Eq.\ref{eq:dual_gje_loss1}): 
    \STATE \quad $\mathcal{L}(\theta, \phi) = \frac{1}{2} \text{Tr}(Z_t^T K_{cc}^{-1} Z_t) + \frac{d_t}{2} \log |K_{cc}|$

    \STATE \textbf{// 3. Backward Pass (Asymmetric Optimization)}
    \STATE Compute gradients: $\nabla_{\theta, \phi} \mathcal{L}$
    \STATE Update context network and kernel using optimizer (e.g. Adam):
    \STATE \quad $\theta \leftarrow \theta - \eta \nabla_\theta \mathcal{L}$
    \STATE \quad $\phi \leftarrow \phi - \eta \nabla_\phi \mathcal{L}$
    \STATE \textbf{// 4. Target Network EMA Update}
    \STATE \quad $\theta' \leftarrow \tau \theta' + (1 - \tau) \theta$
\ENDWHILE
\end{algorithmic}
\end{algorithm}

\textbf{Computational Complexity.} The overall computational complexity of the exact Dual-GJE algorithm evaluates to $\mathcal{O}(N^2 d_c + N^2 d_t + N^3)$. In Step 2, computing the $N \times N$ empirical Gram matrix $K_{cc}$ across the batch requires $\mathcal{O}(N^2 d_c)$ operations. Subsequently, evaluating the conditional NLL loss necessitates computing both the inverse ($K_{cc}^{-1}$) and the log-determinant ($\log |K_{cc}|$) of this Gram matrix. These operations, typically implemented via a Cholesky decomposition, impose a severe $\mathcal{O}(N^3)$ computational bottleneck. Further, calculating the trace of the data-fit term involves matrix multiplications that scale as $\mathcal{O}(N^2 d_t)$. Because self-supervised learning inherently requires massive batch sizes ($N$) to provide a sufficiently diverse topological landscape for the contrastive and predictive forces, the cubic $\mathcal{O}(N^3)$ term overwhelmingly dominates the computational footprint. This strict intractability for large $N$ heavily motivates the $\mathcal{O}(N D^2)$ scalable Random Fourier Features (RFF) approximation detailed below, and the foundational pivot to the $\mathcal{O}(d^3)$ Primal-GJE formulation.

\subsection{Scalable Dual-GJE via Random Fourier Features (RFF)} \label{app:rff_theory}

As established above, formulating GJE in the dual sample space requires computing the inverse and log-determinant of an $N \times N$ Gram matrix. This imposes a strict $\mathcal{O}(N^3)$ computational bottleneck, rendering exact dual optimization intractable for large batch sizes. To resolve this, we utilize Random Fourier Features (RFF) \cite{rahimi2007random} to approximate the kernel mapping, reducing the complexity to $\mathcal{O}(ND^2)$. This section details the theoretical roots of Kernel methods to the explicit application of RFF for the dual GJE objective.

\subsubsection{Kernel Machines and Bochner's Theorem}
By \textit{Mercer's theorem} \cite{mercer1909functions}, any continuous, symmetric, positive-definite kernel $k(x, y)$ can be expressed as an inner product in a high-dimensional feature space: $k(x, y) = \langle \varphi(x), \varphi(y) \rangle_{\mathcal{V}}$. To circumvent the $\mathcal{O}(N^3)$ Gram matrix bottleneck associated with this kernel trick, Rahimi and Recht \cite{rahimi2007random} proposed approximating $k(x, y)$ using a randomized, low-dimensional feature map $z: \mathbb{R}^d \mapsto \mathbb{R}^D$ (where $D \ll N$), such that $k(x, y) \approx z(x)^T z(y)$.

This approximation relies on \textit{Bochner's Theorem} \cite{rudin1990fourier}, which states that a continuous, shift-invariant kernel $k(x, y) = k(x - y)$ is positive-definite if and only if $k(\Delta)$ is the Fourier transform of a non-negative probability measure $p(\omega)$. Thus, the kernel can be expressed as an expectation:
\begin{equation} 
    k(x - y) = \mathbb{E}_{\omega \sim p(\omega)} \left[ e^{i\omega^T (x - y)} \right]
\end{equation}

By discarding the imaginary component via Euler's formula and introducing a random phase shift $b \sim \mathcal{U}(0, 2\pi)$ \cite{gundersen2019rff}, we can draw $D$ samples of frequencies $\omega_1, \dots, \omega_D \sim p(\omega)$ to construct the explicit $D$-dimensional feature vector:
\begin{equation} 
    \psi(x) = \sqrt{\frac{2}{D}} \begin{bmatrix} \cos(\omega_1^T x + b_1) \\ \vdots \\ \cos(\omega_D^T x + b_D) \end{bmatrix} \in \mathbb{R}^D
\end{equation}
Taking the inner product of two such feature vectors elegantly yields the Monte Carlo approximation of the kernel expectation: $\psi(x)^T \psi(y) \approx k(x, y)$.

\subsubsection{Applying RFF to the Dual GJE Matrix}
For a batch of $N$ context embeddings, let $\Psi \in \mathbb{R}^{N \times D}$ be the matrix where each row corresponds to the evaluated Fourier features $\psi(z_{c,i})^T$. The empirical $N \times N$ Gram matrix can thus be approximated as a low-rank linear inner product: $K_{cc} \approx \Psi \Psi^T$.

To optimize the dual GJE loss, we must evaluate the inverse and log-determinant of the jitter-regularized Gram matrix $K_\epsilon = \Psi \Psi^T + \epsilon I_N$. By factoring the intractable $N \times N$ matrix into the product of these low-dimensional $\Psi$ matrices, we invoke the \textbf{Woodbury Matrix Identity} \cite{woodbury1950inverting} to compute the inverse:
\begin{equation}
    (\Psi \Psi^T + \epsilon I_N)^{-1} = \epsilon^{-1} I_N - \epsilon^{-1} \Psi (\Psi^T \Psi + \epsilon I_D)^{-1} \Psi^T
\end{equation}
Similarly, we apply the \textbf{Weinstein-Aronszajn Identity} \cite{pozrikidis2014grids} to evaluate the log-determinant:
\begin{equation}
    \log |\Psi \Psi^T + \epsilon I_N| = \log |\Psi^T \Psi + \epsilon I_D| + (N-D) \log \epsilon
\end{equation}

By structurally rewriting the objective entirely in terms of the $D \times D$ primal feature covariance matrix $C = \Psi^T \Psi + \epsilon I_D$, we reduce the computational complexity from $\mathcal{O}(N^3)$ to $\mathcal{O}(ND^2)$, enabling strictly linear scaling with respect to the batch size $N$. The implementable training algorithm is detailed in Algo.\ref{alg:GJE_RFF}.

\begin{algorithm}[H]
\caption{Training Dual-GJE with Random Fourier Features (GJE-RFF)}
\label{alg:GJE_RFF}
\begin{algorithmic}[1]
\REQUIRE Dataset $\mathcal{D}$, Context Encoder $E_\theta$, Target Encoder $E_{\theta'}$. RFF dimension $D$, Batch size $N$, Learning rate $\eta$, EMA momentum $\tau$. Kernel parameters $\phi$ determining $p(\omega)$.

\WHILE{network has not converged}
    \STATE Sample a batch of $N$ augmented view pairs $\{(x_c^{(i)}, x_t^{(i)})\}_{i=1}^N \sim \mathcal{D}$
    
    \STATE \textbf{// 1. Forward Pass (Asymmetric)}
    \STATE Compute context embeddings: $Z_c = E_\theta(X_c) \in \mathbb{R}^{N \times d_c}$
    \STATE \textit{No-grad:} Compute target anchors: $Z_t = \text{stop\_grad}(E_{\theta'}(X_t)) \in \mathbb{R}^{N \times d_t}$

    \STATE \textbf{// 2. Scalable RFF Approximation}
    \STATE Sample frequencies $\Omega \in \mathbb{R}^{d_c \times D} \sim p(\omega)$ and biases $b \in \mathbb{R}^{1 \times D} \sim \mathcal{U}(0, 2\pi)$.
    \STATE Compute Fourier features over context: $\Psi = \sqrt{\frac{2}{D}} \cos(Z_c \Omega + b) \in \mathbb{R}^{N \times D}$
    \STATE Compute $D \times D$ feature covariance: $C = \Psi^T \Psi + \epsilon I_D$
    
    \STATE \textbf{// 3. Fast Inverse and Log-Determinant Computation}
    \STATE Use Woodbury identity for inverse: $K_{cc}^{-1} \approx \epsilon^{-1} I_{N} - \epsilon^{-1} \Psi C^{-1} \Psi^T$
    \STATE Use Weinstein–Aronszajn identity for log-det: $\log |K_{cc}| \approx \log |C| + (N - D) \log \epsilon$
    \STATE Compute approximated conditional NLL loss:
    \STATE \quad $\mathcal{L}(\theta, \phi) \approx \frac{1}{2} \text{Tr}(Z_t^T K_{cc}^{-1} Z_t) + \frac{d_t}{2} \log |K_{cc}|$

    \STATE \textbf{// 4. Backward Pass}
    \STATE Compute gradients $\nabla_{\theta, \phi} \mathcal{L}$ and update context/kernel parameters.
    \STATE Update target network via EMA: $\theta' \leftarrow \tau \theta' + (1 - \tau) \theta$
\ENDWHILE
\end{algorithmic}
\end{algorithm}
\textit{Note on computational trick: by reordering the matrix multiplications inside the trace operator via cyclic permutation in Step 3, evaluating $\text{Tr}(\Psi C^{-1} \Psi^T Z_t Z_t^T)$ as $\text{Tr}((\Psi^T Z_t) (Z_t^T \Psi) C^{-1})$, we avoid ever materializing the $N \times N$ matrix in memory.}

\section{Training Primal-GJE: Symmetric Optimization and the Entropy Fix} \label{app:training_primal_gje}

In contrast to the Dual space, the Primal-GJE framework formulates the geometry in the $d \times d$ feature space using the exact joint likelihood $p(z_c, z_t)$. Because this generative formulation provides symmetric marginal regularization for both the context and target spaces, \textit{it mathematically eliminates the need for an EMA target network or stop-gradients}. The gradients flow synchronously through both encoders, providing a symmetric representation learning architecture.

However, as mentioned in Section.\ref{subsec:primal_GJE} and detailed in Appendix.\ref{app:mahalanobis_trap}, attempting to directly minimize the exact NLL using the empirically estimated batch covariance $C_{joint}$ triggers the "\textit{Mahalanobis Trace Trap}", algebraically collapsing the repulsive data-fit gradients to zero. 

To rescue the optimization dynamics while maintaining the symmetric structure, the empirical implementation of Primal-GJE structurally inverts the regularizer. Instead of minimizing the log-determinant (volume penalty) against a dead Mahalanobis term, we explicitly \textit{maximize} the differential entropy ($\frac{1}{2} \log |C_{joint}|$). By setting our loss function to minimize the negative of this entropy ($\mathcal{L} = -\frac{1}{2} \log |C_{joint}|$), this formulation actively forces the joint feature distribution to span the available geometry of the unit hypersphere, safely guarding against dimensional collapse. The complete symmetric training procedure is detailed in Algo.\ref{alg:primal_gje}.

\begin{algorithm}[H]
\caption{Training Primal-GJE Model (Symmetric Feature Covariance)}
\label{alg:primal_gje}
\begin{algorithmic}[1]
\REQUIRE Dataset $\mathcal{D}$, Context Encoder $E_\theta$, Target Encoder $E_{\theta'}$. Batch size $N$, Learning rate $\eta$.

\WHILE{network has not converged}
    \STATE Sample a batch of $N$ augmented view pairs $\{(x_c^{(i)}, x_t^{(i)})\}_{i=1}^N \sim \mathcal{D}$
    
    \STATE \textbf{// 1. Forward Pass (Symmetric)}
    \STATE Compute context embeddings: $Z_c = E_\theta(X_c) \in \mathbb{R}^{N \times d_c}$
    \STATE Compute target embeddings: $Z_t = E_{\theta'}(X_t) \in \mathbb{R}^{N \times d_t}$
    \STATE Optional: Apply $L^2$-normalization to bound the embeddings to a unit hypersphere.
    \STATE Concatenate joint embeddings (feature-wise): $Z = [Z_c, Z_t] \in \mathbb{R}^{N \times d}$

    \STATE \textbf{// 2. Feature Covariance Matrix \& Empirical Loss}
    \STATE Compute the $d \times d$ empirical covariance matrix: $C_{joint} = \frac{1}{N} Z^T Z$
    \STATE Add jitter for numerical stability: $C_{joint} \leftarrow C_{joint} + \epsilon I_{d}$
    \STATE Compute the explicit entropy maximization loss (inverted geometric regularizer): 
    \STATE \quad $\mathcal{L}(\theta, \theta') = -\frac{1}{2} \log |C_{joint}|$

    \STATE \textbf{// 3. Backward Pass (Synchronous Optimization)}
    \STATE Compute gradients: $\nabla_{\theta, \theta'} \mathcal{L}$
    \STATE Update parameters synchronously using an optimizer (e.g. Adam):
    \STATE \quad $\theta \leftarrow \theta - \eta \nabla_\theta \mathcal{L}$
    \STATE \quad $\theta' \leftarrow \theta' - \eta \nabla_{\theta'} \mathcal{L}$
    \STATE \textit{Note: No EMA target network or stop-gradients are applied.}
\ENDWHILE
\end{algorithmic}
\end{algorithm}

\textbf{Computational complexity.} The overall computational complexity of the symmetric Primal-GJE algorithm evaluates to $\mathcal{O}(N d^2 + d^3)$. In Step 2, computing the $d \times d$ empirical feature covariance matrix $C_{joint}$ via the inner product $Z^T Z$ requires $\mathcal{O}(N d^2)$ operations. Subsequently, evaluating the log-determinant geometric regularizer ($\log |C_{joint}|$) necessitates a Cholesky decomposition of this covariance matrix, imposing an $\mathcal{O}(d^3)$ computational cost. In standard self-supervised representation learning paradigms, the batch size $N$ (e.g. $4096$) is typically much larger than the embedding dimension $d$ (e.g. $512$). Under the condition $d \ll N$, the $\mathcal{O}(d^3)$ term becomes negligible, and the algorithm scales strictly linearly with respect to the batch size ($\mathcal{O}(N)$). This fundamentally resolves the intractable $\mathcal{O}(N^3)$ matrix inversion bottleneck inherent to the exact Dual-GJE (Sample Space) formulation (Algo.\ref{alg:exact_dual_gje} in Appendix.\ref{app:training_dual_gje}
), allowing the primal network to scale gracefully to massive batch sizes without requiring randomized kernel approximations.

\section{Derivations of Scalable RFF Matrix Identities} \label{app:rff_identities}

In Appendix.\ref{app:rff_theory}, we utilized \textit{the Woodbury Matrix Identity} and \textit{the Weinstein-Aronszajn identity} to project the intractable $\mathcal{O}(N^3)$ operations of the dual sample space into a computationally tractable $\mathcal{O}(D^3)$ feature space. This section provides the step-by-step algebraic derivations for these specific transformations.

\subsection{Derivation of the RFF Woodbury Identity} \label{app:woodbury_derivation}

In the GJE-RFF approximation, the exact $N \times N$ kernel matrix is approximated using Random Fourier Features as $\Sigma \approx \Psi \Psi^T + \epsilon I_N$, where $\Psi \in \mathbb{R}^{N \times D}$ and $\epsilon I_N$ is the jitter term added for numerical stability.

Directly inverting this $N \times N$ matrix is computationally prohibitive for large batch sizes. To resolve this, we apply the \textit{Woodbury Matrix Identity} \cite{woodbury1950inverting}, which is generally defined as:
\begin{equation}
    (A + U C_{wood} V)^{-1} = A^{-1} - A^{-1} U (C_{wood}^{-1} + V A^{-1} U)^{-1} V A^{-1}
\end{equation}

To apply this to our approximated covariance matrix, we make the following substitutions:
\begin{itemize}
    \item $A = \epsilon I_N \implies A^{-1} = \epsilon^{-1} I_N$
    \item $U = \Psi$
    \item $C_{wood} = I_D \implies C_{wood}^{-1} = I_D$
    \item $V = \Psi^T$
\end{itemize}

Substituting these into the Woodbury identity gives:
\begin{equation}
    (\epsilon I_N + \Psi \Psi^T)^{-1} = (\epsilon^{-1} I_N) - (\epsilon^{-1} I_N) \Psi \left( I_D + \Psi^T (\epsilon^{-1} I_N) \Psi \right)^{-1} \Psi^T (\epsilon^{-1} I_N)
\end{equation}

Since multiplying by the identity matrix leaves the matrices unchanged, we can simplify the outer terms and pull the scalar $\epsilon^{-1}$ multipliers to the front of the second term:
\begin{equation} \label{eq:woodbury_intermediate}
    \Sigma^{-1} \approx \epsilon^{-1} I_N - \epsilon^{-2} \Psi \left( I_D + \epsilon^{-1} \Psi^T \Psi \right)^{-1} \Psi^T
\end{equation}

Next, we look at the inner inverted matrix. We want to relate this to our smaller $D \times D$ feature covariance matrix, defined in Algo.\ref{alg:GJE_RFF} as $C = \Psi^T \Psi + \epsilon I_D$. We can algebraically factor out $\epsilon^{-1}$ from the inner term:
\begin{equation}
    I_D + \epsilon^{-1} \Psi^T \Psi = \epsilon^{-1} (\epsilon I_D + \Psi^T \Psi) = \epsilon^{-1} C
\end{equation}

Taking the inverse of this inner term flips the scalar:
\begin{equation}
    (\epsilon^{-1} C)^{-1} = \epsilon C^{-1}
\end{equation}

Finally, we substitute this back into Eq.\ref{eq:woodbury_intermediate}. The $\epsilon$ from the inner inverse cancels out one of the $\epsilon^{-1}$ scalars from the $\epsilon^{-2}$ multiplier in the front:
\begin{align}
    \Sigma^{-1} &\approx \epsilon^{-1} I_N - \epsilon^{-2} \Psi (\epsilon C^{-1}) \Psi^T \nonumber \\
                &= \epsilon^{-1} I_N - \epsilon^{-1} \Psi C^{-1} \Psi^T
\end{align}

This final elegant equation proves that to invert the large $N \times N$ matrix $\Sigma$, we only need to invert the much smaller $D \times D$ feature covariance matrix $C$, effectively reducing the computational bottleneck from $\mathcal{O}(N^3)$ to $\mathcal{O}(ND^2)$.

\subsection{Derivation of the Log-Determinant via Weinstein-Aronszajn identity} \label{app:Weinstein_Aronszajn_derivation}

Alongside the inverse, evaluating the joint likelihood requires computing the log-determinant $\log |\Sigma| = \log |\Psi \Psi^T + \epsilon I_N|$. Calculating the determinant of an $N \times N$ matrix is also an $\mathcal{O}(N^3)$ operation. 

To reduce this dimensionality, we first algebraically factor out the scalar $\epsilon$ from the matrix. Note that factoring a scalar out of an $N \times N$ determinant raises the scalar to the power of $N$:
\begin{equation}
    |\Sigma| = |\epsilon (I_N + \epsilon^{-1} \Psi \Psi^T)| = \epsilon^N |I_N + \epsilon^{-1} \Psi \Psi^T|
\end{equation}

We then apply the \textit{Weinstein-Aronszajn identity}\footnote{The Weinstein–Aronszajn identity is the determinant analogue of the Woodbury matrix identity for matrix inverses.} \cite{pozrikidis2014grids}, which states that for any matrices $A \in \mathbb{R}^{N \times D}$ and $B \in \mathbb{R}^{D \times N}$, the identity $|I_N + AB| = |I_D + BA|$ holds true. Letting $A = \epsilon^{-1} \Psi$ and $B = \Psi^T$, we shift the inner dimensions from $N$ to $D$:
\begin{equation}
    |I_N + \epsilon^{-1} \Psi \Psi^T| = |I_D + \epsilon^{-1} \Psi^T \Psi|
\end{equation}

Substituting this back into our determinant expression yields:
\begin{equation}
    |\Sigma| = \epsilon^N |I_D + \epsilon^{-1} \Psi^T \Psi|
\end{equation}

To reconstruct our $D \times D$ feature covariance matrix $C = \Psi^T \Psi + \epsilon I_D$, we factor the scalar $\epsilon^{-1}$ out of the $D$-dimensional determinant. This introduces an $\epsilon^{-D}$ term:
\begin{equation}
    |\Sigma| = \epsilon^N |\epsilon^{-1}(\epsilon I_D + \Psi^T \Psi)| = \epsilon^N \epsilon^{-D} |\epsilon I_D + \Psi^T \Psi| = \epsilon^{N - D} |C|
\end{equation}

Taking the natural logarithm of both sides cleanly transforms the exponentiation into linear scalar addition:
\begin{align}
    \log |\Sigma| &= \log (\epsilon^{N - D} |C|) \nonumber \\
    \log |\Sigma| &= \log |C| + (N - D) \log \epsilon
\end{align}

This derivation shows that the large $N$-dimensional log-determinant can be evaluated exactly by computing the $D$-dimensional log-determinant of $C$ and applying a simple scalar correction based on the dimensionality difference and the jitter term $\epsilon$.

\section{Primal-GJE: Alternative Optimization via the Hilbert-Schmidt Independence Criterion (HSIC) Objective} \label{app:hsic_theory}

As illustrated in in Section.\ref{alg:primal_gje}, when the Primal-GJE joint objective Eq.\ref{eq:primal_gje_loss} reduces to the differential entropy $\frac{1}{2} \log |C_{joint}|$ (“\textit{Mahalanobis Trace Trap}”), we can use an alternative, gradient-based optimization routine utilizing the Hilbert-Schmidt Independence Criterion (HSIC) \cite{Gretton2005HSIC,Mooij2009HSIC,Greenfeld2020HSIC}. Originally proposed as a non-parametric, kernel-based measure of statistical dependence, HSIC evaluates the squared Hilbert-Schmidt norm of the cross-covariance operator between two Reproducing Kernel Hilbert Spaces (RKHSs). In this section, we review the theoretical formulation of HSIC and demonstrate how it acts as a proxy for mutual information to optimize the Primal-GJE encoders.

\subsection{The Cross-Covariance Operator}
Consider two random variables $X$ and $Y$ residing in metric spaces $\mathcal{X}$ and $\mathcal{Y}$, and let $\mathcal{F}$ and $\mathcal{G}$ be separable Reproducing Kernel Hilbert Spaces on these spaces with universal kernels $k$ and $l$, respectively. Let $\phi(x) \in \mathcal{F}$ and $\psi(y) \in \mathcal{G}$ denote their corresponding feature maps. 

The cross-covariance operator $C_{xy} : \mathcal{G} \rightarrow \mathcal{F}$ associated with the joint measure $p_{xy}$ is defined as a linear operator generalizing the concept of a covariance matrix to infinite-dimensional function spaces:
\begin{equation}
    C_{xy} = \mathbb{E}_{xy} [(\phi(x) - \mu_x) \otimes (\psi(y) - \mu_y)]
\end{equation}
where $\otimes$ denotes the tensor product, and $\mu_x, \mu_y$ are the mean elements in their respective RKHSs. 

\subsection{Population and Empirical HSIC}
The population HSIC is formally defined as the squared Hilbert-Schmidt norm of the cross-covariance operator:
\begin{equation}
    \text{HSIC}(p_{xy}, \mathcal{F}, \mathcal{G}) := ||C_{xy}||_{HS}^2
\end{equation}
A fundamental theorem of HSIC establishes that if the kernels $k$ and $l$ are universal on compact domains, then $||C_{xy}||_{HS}^2 = 0$ if and only if $X$ and $Y$ are strictly statistically independent. Thus, the magnitude of the HSIC serves as a powerful non-parametric measure of dependence.

To compute this criterion practically from a finite sample without explicitly mapping to the potentially infinite-dimensional RKHS, it can be expressed entirely in terms of inner products (kernel functions). Given a mini-batch of $N$ embedding pairs generated by our dual encoders, $Z := \{(z_{c,1}, z_{t,1}), \dots, (z_{c,N}, z_{t,N})\}$, the empirical estimator of HSIC is defined as \cite{Gretton2005HSIC}:
\begin{equation} \label{eq:empirical_HSIC}
    \widehat{\text{HSIC}}(Z_c, Z_t) = \frac{1}{(N-1)^2} \text{Tr}(K_c H K_t H)
\end{equation}
In this formulation, $K_c$ and $K_t \in \mathbb{R}^{N \times N}$ are the kernel (Gram) matrices computed over the context embeddings $Z_c$ and target embeddings $Z_t$, respectively. The matrix $H \in \mathbb{R}^{N \times N}$ is the standard centering matrix defined as $H_{ij} = \delta_{ij} - \frac{1}{N}$. 

\subsection{Convergence and Approximation Rates}
A major advantage of HSIC over other kernel-based dependence tests (such as Kernel Canonical Correlation or Kernel Mutual Information) is its statistical sample efficiency and computational simplicity, as it requires no user-defined regularization terms to invert matrices. Theoretical analysis guarantees that the empirical estimate converges to the population quantity at a rapid rate of $\mathcal{O}(N^{-1/2})$. Further, the finite sample bias of the empirical estimator is bounded by $\mathcal{O}(N^{-1})$, rendering the bias negligible compared to finite sample fluctuations for sufficiently large batch sizes. 

\subsection{HSIC as a Differentiable Self-Supervised Objective}
Because the empirical HSIC estimator (Eq.\ref{eq:empirical_HSIC}) requires only basic matrix multiplications and the evaluation of differentiable kernel functions, it is entirely end-to-end differentiable. Recent machine learning literature has successfully adapted HSIC as a loss function for deep neural networks \cite{Greenfeld2020HSIC}. 

While prior machine learning literature typically \textit{minimizes} the HSIC loss to enforce strict statistical independence between model inputs and regression residuals for robust generalization or causal discovery, our self-supervised objective strictly \textit{maximizes} HSIC. Rather than modeling the joint distribution directly via NLL, we optimize the two encoder networks by forcefully maximizing the statistical dependence between the context latent representations $Z_c$ and target representations $Z_t$. 

Maximizing this trace operator directly enforces maximal non-linear correlation across the cross-covariance block, driving the network to extract shared, highly dependent features from the augmented views $x_c$ and $x_t$. Because the kernels $k(\cdot, \cdot)$ are differentiable, the gradients of this HSIC objective can be computed easily via standard automatic differentiation and backpropagated symmetrically through both encoders.

\paragraph{Combating Representation Collapse.} 
Importantly, maximizing HSIC natively penalizes representation collapse. If the context or target encoders lazily map all inputs to a trivial constant vector, their respective Gram matrix ($K_c$ or $K_t$) degenerates into a matrix of ones. Multiplying a matrix of ones by the centering matrix $H$ yields a zero matrix, driving the HSIC objective strictly to zero. Because our self-supervised objective seeks to \textit{maximize} this value, the network actively avoids collapsed states. By maximizing the dependence between the context and target spaces via this kernel geometry, GJE can maintain structurally diverse representations without necessarily relying on standard JEPA's asymmetric EMA target networks or stop-gradient heuristics. Further, standard auto-covariance regularizers (such as the marginal log-determinant penalty) can be seamlessly applied alongside the HSIC objective to ensure the individual embedding spaces maintain uniform, spherical marginal distributions.

\section{Exact Dual GJE with a Linear Kernel: The Primal-Dual Equivalence} \label{app:linear_kernel_equivalence}

As mentioned in Section \ref{sec:GJE}, with a linear kernel, Dual-GJE is exactly equal to Primal-GJE. In this appendix, we mathematically prove that, if this exact dual GJE formulation is implemented using a standard \textit{Linear Kernel}, it algebraically collapses into a single-component Gaussian Mixture Model ($K=1$ GMM) operating in the \textit{primal feature space}. 

This primal-dual equivalence proves that computing the joint likelihood using an $N \times N$ linear Gram matrix over the samples is mathematically identical to computing the likelihood using the $d \times d$ feature covariance matrix, thereby validating the unimodal ($K=1$) visualization baseline utilized in our synthetic experiments (Section \ref{subsec:exp_synthetic}).

\subsection{Setup and the Empirical Loss}
Consider a batch of joint embeddings represented by the matrix $Z \in \mathbb{R}^{N \times d}$, where $N$ is the batch size and $d = d_c + d_t$ is the combined dimensionality of the context and target features. To ensure positive-definiteness and invertibility, we introduce a standard jitter term $\epsilon > 0$.

We define the corresponding covariance matrices in the two perpendicular spaces:
\begin{itemize}
    \item \textbf{The Dual (Sample) Gram Matrix:} $K = Z Z^T + \epsilon I_N \quad \in \mathbb{R}^{N \times N}$
    \item \textbf{The Primal (Feature) Covariance Matrix:} $C = Z^T Z + \epsilon I_d \quad \in \mathbb{R}^{d \times d}$
\end{itemize}

When evaluating the exact Negative Log-Likelihood (NLL) over a multi-dimensional batch $Z$, the Mahalanobis distance sum across all $d$ independent feature channels analytically translates into the matrix trace operator. The empirical dual loss function to be minimized is (Eq.\ref{eq:primal_gje_loss}):
\begin{equation} \label{eq:app_dual_loss}
    \mathcal{L}_{\text{Dual}} = \underbrace{\frac{1}{2} \text{Tr}(Z^T K^{-1} Z)}_{\text{Data-Fit (Trace)}} + \underbrace{\frac{1}{2} \log |K|}_{\text{Regularizer (Log-Det)}}
\end{equation}

We will now mathematically project both the \textit{Data-Fit term} and the \textit{Regularizer term} from the $\mathbb{R}^{N \times N}$ sample space into the $\mathbb{R}^{d \times d}$ feature space.

\subsection{Part 1: The Data-Fit Term via the Push-Through Identity}
To project the $N \times N$ inverse Gram matrix $K^{-1}$ into the primal space, we leverage the \textit{Push-Through Identity}, a direct corollary of the Woodbury Matrix Identity, which dictates how matrices "push through" inverses: $A(I + BA)^{-1} = (I + AB)^{-1}A$.

Applying this structural property to our matrices:
\begin{equation*}
    Z^T (\epsilon I_N + Z Z^T)^{-1} = (\epsilon I_d + Z^T Z)^{-1} Z^T
\end{equation*}

We multiply both sides of this equation by $Z$ on the right:
\begin{equation}
    Z^T (\epsilon I_N + Z Z^T)^{-1} Z = (\epsilon I_d + Z^T Z)^{-1} Z^T Z
\end{equation}

Notice that the left side is exactly our dual inner term $Z^T K^{-1} Z$, and the right side is formulated entirely in terms of the primal feature matrix $C = Z^T Z + \epsilon I_d$. We can substitute $C$ into the right-hand side. Recognizing that $Z^T Z = C - \epsilon I_d$, we evaluate the matrix product:
\begin{align}
    Z^T K^{-1} Z &= C^{-1} (Z^T Z) \nonumber \\
    &= C^{-1} (C - \epsilon I_d) \nonumber \\
    &= C^{-1} C - \epsilon C^{-1} \nonumber \\
    &= I_d - \epsilon C^{-1}
\end{align}

Finally, taking the trace of both sides yields the exact primal equivalence for the data-fit term:
\begin{empheq}[box=\fbox]{equation}
    \text{Tr}(Z^T K^{-1} Z) = \text{Tr}(I_d) - \epsilon \text{Tr}(C^{-1}) = d - \epsilon \text{Tr}(C^{-1})
\end{empheq}

\subsection{Part 2: The Geometric Regularizer via Weinstein-Aronszajn}
Next, we project the $N \times N$ log-determinant volume penalty, $\log |K|$, into the primal space. We apply the \textit{Weinstein-Aronszajn Identity} (also known as Sylvester's Determinant Theorem), which states that for any matrices $A \in \mathbb{R}^{N \times d}$ and $B \in \mathbb{R}^{d \times N}$, the determinant identity $|I_N + AB| = |I_d + BA|$ holds true.

We begin by factoring the scalar jitter $\epsilon$ out of the dual Gram matrix determinant. Because $K$ is an $N \times N$ matrix, factoring out $\epsilon$ raises it to the power of $N$:
\begin{equation}
    |K| = |Z Z^T + \epsilon I_N| = \left| \epsilon \left( I_N + \frac{1}{\epsilon} Z Z^T \right) \right| = \epsilon^N \left| I_N + \left(\epsilon^{-1} Z\right) Z^T \right|
\end{equation}

We now apply the Weinstein-Aronszajn identity, setting $A = \epsilon^{-1} Z$ and $B = Z^T$, which swaps the multiplication order and cleanly reduces the identity matrix dimension from $N$ to $d$:
\begin{equation}
    |K| = \epsilon^N \left| I_d + Z^T \left(\epsilon^{-1} Z\right) \right| = \epsilon^N \left| I_d + \epsilon^{-1} Z^T Z \right|
\end{equation}

To reconstruct the primal feature covariance matrix $C = Z^T Z + \epsilon I_d$ inside the determinant, we must absorb a factor of $\epsilon$ back into the $d \times d$ matrix. Absorbing $\epsilon$ into a $d$-dimensional determinant scales the outside multiplier by $\epsilon^{-d}$:
\begin{equation}
    |K| = \epsilon^N \epsilon^{-d} \left| \epsilon (I_d + \epsilon^{-1} Z^T Z) \right| = \epsilon^{N-d} |Z^T Z + \epsilon I_d| = \epsilon^{N-d} |C|
\end{equation}

Taking the natural logarithm of both sides transforms the exponentiation into linear scalar addition:
\begin{empheq}[box=\fbox]{equation}
    \log |K| = \log |C| + (N - d) \log \epsilon
\end{empheq}

\subsection{Conclusion: The Unimodal Constraint}
By substituting our two derived primal equivalencies back into the original dual loss function (Eq.\ref{eq:app_dual_loss}), we obtain:
\begin{align} \label{eq:app_primal_loss_final}
    \mathcal{L}_{\text{Dual}} &= \frac{1}{2} \left[ d - \epsilon \text{Tr}(C^{-1}) \right] + \frac{1}{2} \left[ \log |C| + (N - d) \log \epsilon \right] \nonumber \\
    &= \frac{1}{2} \left( \log |C| - \epsilon \text{Tr}(C^{-1}) \right) + \text{Constant}
\end{align}

Equation \ref{eq:app_primal_loss_final} reveals a profound structural truth: \textit{if one utilizes a linear kernel, the exact $\mathcal{O}(N^3)$ Gaussian Process NLL is mathematically driven entirely by the $d \times d$ feature covariance matrix $C$}. It possesses absolutely no non-linear flexibility. 

Because the loss gradients are algebraically locked to the exact same $d \times d$ global covariance structure, optimizing the dual sample-space GJE yields the exact same weights as fitting a single, global multivariate Gaussian ($K=1$ GMM) in the primal feature space. This fundamentally restricts the learned representations to a strict linear geometry. To wrap around non-linear, multi-modal semantic branches (as illustrated in our synthetic datasets), one is forced to either utilize highly expensive non-linear approximations (e.g. Random Fourier Features for an RBF kernel) to bend this single density in an infinite-dimensional RKHS, or more naturally, pivot directly into the primal feature space and deploy multiple discrete prototypes via Gaussian Mixture Joint Embeddings (GMJE).

\section{Properties of GMM} \label{app:GMM_properties}

In the GMJE framework (Section.\ref{sec:GMJE}), we model the concatenated context and target embeddings $Z = [z_c^T, z_t^T]^T$ as a joint Gaussian Mixture Model (GMM) with $K$ components\footnote{$K$ in this context denotes the total number of GMM components, not a kernel generated covariance matrix.}. In this section, we mathematically derive the exact marginal distribution $p(z_t)$ and the conditional distribution $p(z_t \mid z_c)$ produced by this joint mixture, and explore how this distribution natively yields multi-modal uncertainty quantification.

Let the joint distribution be defined as (Eq.\ref{eq:GMJE_joint}):
\begin{equation} \label{eq:GMM_joint_dist}
p(z_c, z_t) = \sum_{k=1}^K \pi_k p(z_c, z_t \mid k)
\end{equation}
where $\pi_k$ is the prior probability of the $k$-th mixture component, and the component-wise joint distribution is a multivariate Gaussian:
\begin{equation} \label{eq:GMM_joint_dist_component}
    p(z_c, z_t \mid k) = \mathcal{N} \left( \begin{bmatrix} z_c \\ z_t \end{bmatrix} \middle| \begin{bmatrix} \mu_{c,k} \\ \mu_{t,k} \end{bmatrix}, \begin{bmatrix} \Sigma_{cc,k} & \Sigma_{ct,k} \\ \Sigma_{tc,k} & \Sigma_{tt,k} \end{bmatrix} \right)
\end{equation}

\subsection{Derivation of the Marginal Distribution of a Joint GMM} \label{app:GMJE_marginal_derivation}
For generative downstream tasks (such as unconditional image synthesis, Section.\ref{subsec:GMJE_generative}), we frequently need to sample directly from the representation space of a single view (e.g. the target space $z_t$). To obtain this, we must evaluate the marginal distribution $p(z_t)$.

By the sum rule of probability, the marginal distribution is obtained by integrating the joint density over the continuous context variable $z_c$:
\begin{align}
p(z_t) &= \int p(z_c, z_t) dz_c \nonumber \\
&= \int \sum_{k=1}^K \pi_k p(z_c, z_t \mid k) dz_c \nonumber \\
&= \sum_{k=1}^K \pi_k \int p(z_c, z_t \mid k) dz_c
\end{align}
By the marginalization properties of joint Gaussians (as detailed in Appendix.\ref{app:Gaussian_dist}), the integral of a joint Gaussian over one of its component blocks is simply the marginal Gaussian of the remaining block. Thus, the integral elegantly resolves to:
\begin{equation}
    \int p(z_c, z_t \mid k) dz_c = p(z_t \mid k) = \mathcal{N}(z_t \mid \mu_{t,k}, \Sigma_{tt,k})
\end{equation}
Substituting this back into the summation yields the final marginal distribution:
\begin{equation} \label{eq:GMJE_marginal_target}
p(z_t) = \sum_{k=1}^K \pi_k \mathcal{N}(z_t \mid \mu_{t,k}, \Sigma_{tt,k})
\end{equation}
This establishes a profound and convenient property: \textbf{the marginal distribution of a joint GMM is simply a new GMM} constructed using the marginalized means ($\mu_{t,k}$) and covariances ($\Sigma_{tt,k}$) of the respective block, weighted by the exact same prior mixing coefficients ($\pi_k$). 

By symmetry, the marginal distribution of the context embedding $p(z_c)$, which acts as the geometric regularizer during training, evaluates exactly to:
\begin{equation} \label{eq:GMJE_marginal_context}
p(z_c) = \sum_{k=1}^K \pi_k \mathcal{N}(z_c \mid \mu_{c,k}, \Sigma_{cc,k})
\end{equation}

\subsection{Derivation of the Conditional Distribution of a Joint GMM} \label{app:GMJE_conditional_derivation}

Here we prove that \textbf{the conditional distribution $p(z_t \mid z_c)$ derived from a joint Gaussian mixture is also exactly a new GMM}, and we identify its closed-form parameters.

By the definition of conditional probability, we divide the joint distribution by the marginal context distribution (Eq.\ref{eq:GMJE_marginal_context}):
\begin{equation} \label{eq:GMJE_cond_step1}
p(z_t \mid z_c) = \frac{p(z_c, z_t)}{p(z_c)} = \frac{\sum_{k=1}^K \pi_k p(z_c, z_t \mid k)}{p(z_c)}
\end{equation}
Using the product rule of probability \textit{within} each mixture component, we factorize the joint component distribution as $p(z_c, z_t \mid k) = p(z_t \mid z_c, k) p(z_c \mid k)$. Substituting this factorization into the numerator:
\begin{equation}
p(z_t \mid z_c) = \frac{\sum_{k=1}^K \pi_k p(z_t \mid z_c, k) p(z_c \mid k)}{p(z_c)}
\end{equation}
We can rearrange the terms inside the summation to isolate the component conditional distribution $p(z_t \mid z_c, k)$:
\begin{empheq}[box=\dashedbox]{equation}
p(z_t \mid z_c)
= \sum_{k=1}^K
\left[
\frac{\pi_k\, p(z_c \mid k)}{p(z_c)}
\right]
\, p(z_t \mid z_c, k)
= \sum_{k=1}^K
\gamma_k(z_c)
\, p(z_t \mid z_c, k)
\label{eq:GMJE_cond_step2}
\end{empheq}

Eq.\ref{eq:GMJE_cond_step2} is structurally a mixture model. Let us examine its two distinct parts.

\textbf{1. The mixing weights ($\gamma_k$):}
We define the bracketed term as our new data-dependent mixing weight, $\gamma_k(z_c)$. By substituting $p(z_c)$ from Eq.\ref{eq:GMJE_marginal_context}, we get:
\begin{empheq}[box=\dashedbox]{equation}
\gamma_k(z_c) 
= \frac{\pi_k p(z_c \mid k)}{\sum_{j=1}^K \pi_j p(z_c \mid j)}
= \frac{\pi_k \mathcal{N}(z_c \mid \mu_{c,k}, \Sigma_{cc,k})}
{\sum_{j=1}^K \pi_j \mathcal{N}(z_c \mid \mu_{c,j}, \Sigma_{cc,j})}
\label{eq:GMJE_cond_step3}
\end{empheq}
By Bayes' theorem, $\gamma_k(z_c)$ represents the exact posterior probability $p(k \mid z_c)$, that is, the probability that the $k$-th mixture component generated the observation, given the context embedding $z_c$. Because they are proper probabilities, they strictly sum to one ($\sum_{k=1}^K \gamma_k(z_c) = 1$).

\textbf{2. The component conditional distributions:}
The term $p(z_t \mid z_c, k)$ is the conditional distribution derived from a single joint Gaussian component. As proven in Appendix.\ref{app:derivation_of_conditional_distribution}, conditioning a joint Gaussian yields exactly another Gaussian. Unlike the zero-mean formulation in the base GJE, GMJE components possess learnable means ($\mu_{c,k}, \mu_{t,k}$). By applying the standard Gaussian conditioning formulas, the closed-form parameters for each component predictor $k$ are (see Eq.\ref{eq:gen_cond_mean_cov}):
\begin{empheq}[box=\dashedbox]{equation}
\begin{aligned}
\mu_{t|c, k} &= \mu_{t,k} + \Sigma_{tc,k}\Sigma_{cc,k}^{-1}(z_c - \mu_{c,k}) \\
\Sigma_{t|c, k} &= \Sigma_{tt,k} - \Sigma_{tc,k}\Sigma_{cc,k}^{-1}\Sigma_{ct,k}
\end{aligned}
\label{eq:GMJE_component_mean_cov}
\end{empheq}

Combining Eq.\ref{eq:GMJE_cond_step3} and Eq.\ref{eq:GMJE_component_mean_cov}, the full conditional distribution Eq.\ref{eq:GMJE_cond_step2} beautifully collapses into a new, \textit{dynamically weighted Gaussian Mixture Model}:
\begin{empheq}[box=\fbox]{equation} \label{eq:GMJE_conditional}
p(z_t \mid z_c) = \sum_{k=1}^K \gamma_k(z_c)\,
\mathcal{N}\!\left(z_t \mid \mu_{t|c, k}, \Sigma_{t|c, k}\right)
\end{empheq}
This mathematically demonstrates that the conditional distribution used for predictive inference is exactly a weighted average of all individual Gaussian conditional components, with weights $\gamma_k(z_c)$ acting as a soft, probabilistic routing mechanism.

\subsection{Uncertainty Quantification: The Law of Total Variance} \label{app:GMM_uncertainty}

A profound mathematical advantage of Gaussian Mixtures is their ability to model complex inverse problems. When a fixed input $z_c$ is provided, a conditional GMM does not merely output a single scalar value and a single variance; it outputs a full, continuous probability density function $p(z_t \mid z_c)$. Because it parameterizes a full distribution via $\{\gamma_k, \mu_k, \Sigma_k\}$, it quantifies uncertainty in two completely different ways simultaneously.

\textbf{Type A: Local ``Noise'' Uncertainty (Aleatoric):}
this is captured directly by the component variances ($\Sigma_k$). It represents the inherent, irreducible noise of the data manifold around a specific, valid prediction. For example, in our synthetic experiments (Section \ref{subsec:exp_synthetic}), when the model predicts $\sigma_k \approx 0.05$, it indicates that if the target belongs to this specific functional branch, it is expected to fluctuate by $\pm 0.05$ strictly due to random observation noise.

\textbf{Type B: Global ``Ambiguity'' Uncertainty (Epistemic / Multi-modal):}
this is captured by the interaction between the mixing weights ($\gamma_k$) and the spatial spread of the component means ($\mu_k$). It represents the model's structural uncertainty regarding which macroscopic outcome is actually occurring. For example, if the mixture model assigns equal mixing weights ($\gamma_1 \approx \gamma_2 \approx \gamma_3 \approx 1/3$) while placing the component means at wildly different coordinates ($+0.8, 0, -0.8$), it is explicitly stating that the given context $z_c$ is highly ambiguous, and any of the three divergent branches are equally valid possibilities. 

\textbf{The Total Variance Decomposition:}
in many engineering applications, it is desirable to reduce a full mixture distribution down to a single "average prediction" and a single "total uncertainty" metric. This is achieved via the \textit{Law of Total Variance}. 

If we force the GMM to collapse its prediction into a single overall expected value (the conditional mean), it simply evaluates to the weighted sum of the components:
\begin{equation} \label{eq:total_mean}
\mu_{\text{total}} = \mathbb{E}[z_t \mid z_c] = \sum_{k=1}^K \gamma_k \mu_k
\end{equation}
As proven in Appendix.\ref{app:mse_derivation}, this is exactly the deterministic point that classic JEPA converges to, frequently landing in empty space.

If we force the GMM to calculate a single overall variance around that mean, the mathematics elegantly splits the uncertainty into two distinct, interpretable terms:
\begin{equation} \label{eq:total_variance}
\Sigma_{\text{total}} = \underbrace{\sum_{k=1}^K \gamma_k \Sigma_k}_{\text{Average Local Noise}} + \underbrace{\sum_{k=1}^K \gamma_k (\mu_k - \mu_{\text{total}})(\mu_k - \mu_{\text{total}})^T}_{\text{Variance of the Means (Ambiguity)}}
\end{equation}

This universal equation is the ultimate key to understanding the specific architectural comparisons in our synthetic experiments:
\begin{itemize}
    \item \textit{GMJE-MDN} models this equation perfectly by keeping the components separate, allowing the neural network to dynamically distinguish between inherently noisy data (high $\Sigma_k$) and ambiguous mappings (spread out $\mu_k$ with balanced $\gamma_k$).
    \item \textit{Classic JEPA} outputs only $\mu_{\text{total}}$ (Eq.\ref{eq:total_mean}) via MSE optimization, completely ignoring the total variance and thus possessing zero notion of uncertainty.
    \item \textit{Unimodal GJE} mathematically attempts to model $\Sigma_{\text{total}}$ natively. However, because it possesses only a single covariance matrix, it is forced to absorb the massive ``Variance of the Means'' (the ambiguity term) directly into its shape. This forces the Unimodal GJE to stretch into a massive, over-smoothed ellipse that bleeds into empty space in a desperate attempt to cover the multi-modal ambiguity with a single Gaussian.
\end{itemize}

\section{MSE Minimization Yields the Conditional Mean} \label{app:mse_derivation}

In this section, we first present the general mathematical proof that minimizing the Mean Squared Error (MSE) risk for any two random variables strictly yields the conditional mean. We then explicitly apply this theorem to the classic JEPA architecture to expose its fundamental theoretical limitations, as mentioned in Section.\ref{sec:intro} and Section.\ref{sec:GJE}, which motivated our use of MDN in Section.\ref{subsec:GMJE_MDN}.

\subsection{The General Case: Two Random Variables}
Let $X$ represent our input variable and $Y$ represent our target variable. We want to find a predictor function $g(X)$ (e.g. a neural network) that minimizes the expected squared error risk:
\begin{equation}
    \mathcal{R}(g) = \mathbb{E}_{X,Y} \left[ \|Y - g(X)\|^2 \right]
\end{equation}

% \textit{Step 1: Introduce the conditional mean.} \\
Let us define the true conditional mean of the target given the context as $m(X) = \mathbb{E}[Y \mid X]$. We can introduce this term into our objective function by adding and subtracting it inside the squared norm:
\begin{equation}
    \|Y - g(X)\|^2 = \|Y - m(X) + m(X) - g(X)\|^2
\end{equation}

% \textbf{Step 2: Expand the quadratic.} \\
Expanding this gives us three distinct terms:
\begin{equation}
    \|Y - g(X)\|^2 = \|Y - m(X)\|^2 + \|m(X) - g(X)\|^2 + 2 \langle Y - m(X), m(X) - g(X) \rangle
\end{equation}

% \textbf{Step 3: Apply the Law of Iterated Expectations.} \\
Now, we take the overall expectation $\mathbb{E}_{X,Y}[\cdot]$ of this expanded equation. By the \textit{Law of Iterated Expectations}, $\mathbb{E}_{X,Y}[\cdot] = \mathbb{E}_X [\mathbb{E}_{Y \mid X}[\cdot \mid X]]$. Let's look at the expectation of the third term (the cross-term) conditioned on $X$:
\begin{equation*}
    \mathbb{E}_{Y \mid X} \left[ \langle Y - m(X), m(X) - g(X) \rangle \mid X \right]
\end{equation*}

Because we are conditioning on $X$, any function of strictly $X$ acts as a constant and can be pulled out of the expectation. Since $m(X)$ and $g(X)$ depend only on $X$, we can pull them out of the left side of the inner product by linearity:
\begin{equation*}
    = \langle \mathbb{E}_{Y \mid X}[Y - m(X) \mid X], m(X) - g(X) \rangle
\end{equation*}

% \textbf{Step 4: The cross-term vanishes.} \\
Look closely at the first part of that inner product: $\mathbb{E}_{Y \mid X}[Y - m(X) \mid X]$. Because the expectation is a linear operator, we can split it:
\begin{equation*}
    \mathbb{E}_{Y \mid X}[Y \mid X] - \mathbb{E}_{Y \mid X}[m(X) \mid X]
\end{equation*}

By our very definition in Step 1, $\mathbb{E}[Y \mid X] = m(X)$. And since $m(X)$ is a constant given $X$, its expectation is just itself. Therefore:
\begin{equation*}
    m(X) - m(X) = \mathbf{0}
\end{equation*}

Therefore, the entire cross-product term completely vanishes.

% \textbf{Step 5: The Final Conclusion.} \\
We are left with a beautifully simplified risk function:
\begin{equation}
    \mathcal{R}(g) = \mathbb{E}_{X,Y} \left[ \|Y - m(X)\|^2 \right] + \mathbb{E}_X \left[ \|m(X) - g(X)\|^2 \right]
\end{equation}

Look at these two remaining terms:
\begin{itemize}
    \item \textbf{The Irreducible Error:} $\mathbb{E}[\|Y - m(X)\|^2]$ is the inherent variance of the data (the conditional variance). It has absolutely no dependence on our neural network $g(X)$. We cannot optimize it; it is a fixed property of the universe/dataset.
    \item \textbf{The Reducible Error:} $\mathbb{E}[\|m(X) - g(X)\|^2]$ is a squared distance, which means it is strictly non-negative ($\geq 0$).
\end{itemize}

To minimize the total risk $\mathcal{R}(g)$, we can only control the second term. The absolute minimum possible value for this term is exactly $0$, which occurs if and only if:
\begin{equation}
    g(X) = m(X)
\end{equation}

That is, the best possible minimizer, using the MSE loss, is the conditional mean.

\subsection{Application to Classic JEPA}

In classic JEPA, the objective is to learn a deterministic predictor $g(z_c)$ that maps a context embedding $z_c$ to a target embedding $z_t$ by minimizing the MSE. We demonstrate here that, for any measurable function $g$, the optimal predictor under the MSE risk natively collapses to the conditional average of the target space, and fails to capture multi-modal distributions.

Let the expected risk be defined as:
\begin{equation}
    \mathcal{R}(g) = \mathbb{E}_{z_c, z_t} \left[ \| z_t - g(z_c) \|^2 \right]
\end{equation}

We define the true conditional mean of the target given the context as $m(z_c) = \mathbb{E}[z_t \mid z_c]$. Adding and subtracting this conditional mean inside the norm yields:
\begin{align}
    \mathcal{R}(g) &= \mathbb{E}_{z_c, z_t} \left[ \| z_t - m(z_c) + m(z_c) - g(z_c) \|^2 \right] \nonumber \\
    &= \mathbb{E}_{z_c, z_t} \left[ \| z_t - m(z_c) \|^2 \right] + \mathbb{E}_{z_c, z_t} \left[ \| m(z_c) - g(z_c) \|^2 \right] \nonumber \\
    &\quad + 2 \mathbb{E}_{z_c, z_t} \left[ \langle z_t - m(z_c), m(z_c) - g(z_c) \rangle \right]
\end{align}

Similarly, by the \textit{Law of Iterated Expectations}, the expectation of the cross-term can be conditioned on $z_c$. Because $m(z_c)$ and $g(z_c)$ are deterministic constants given $z_c$, they can be factored out of the inner conditional expectation:
\begin{align}
    \mathbb{E}_{z_c, z_t} \left[ \langle z_t - m(z_c), m(z_c) - g(z_c) \rangle \right] 
    &= \mathbb{E}_{z_c} \left[ \mathbb{E}_{z_t \mid z_c} \left[ \langle z_t - m(z_c), m(z_c) - g(z_c) \rangle \mid z_c \right] \right] \nonumber \\
    &= \mathbb{E}_{z_c} \left[ \langle \mathbb{E}_{z_t \mid z_c}[z_t \mid z_c] - m(z_c), m(z_c) - g(z_c) \rangle \right]
\end{align}

By our definition, $\mathbb{E}_{z_t \mid z_c}[z_t \mid z_c] = m(z_c)$. Consequently, the left side of the inner product is exactly zero ($m(z_c) - m(z_c) = \mathbf{0}$), causing the entire cross-term to vanish. 

The risk function thus simplifies to two strictly non-negative terms:
\begin{equation}
    \mathcal{R}(g) = \underbrace{\mathbb{E}_{z_c, z_t} \left[ \| z_t - m(z_c) \|^2 \right]}_{\text{Irreducible Variance}} + \underbrace{\mathbb{E}_{z_c} \left[ \| m(z_c) - g(z_c) \|^2 \right]}_{\text{Reducible Error}}
\end{equation}

The first term is the inherent conditional variance of the data distribution, which is independent of our predictor $g$. The second term is a squared distance, which is minimized to its absolute lower bound of $0$ if and only if $g(z_c) = m(z_c)$. Therefore, the optimal deterministic predictor is exactly the conditional mean:
\begin{equation}
    g^*(z_c) = \mathbb{E}[z_t \mid z_c]
\end{equation}
When the true distribution $p(z_t \mid z_c)$ is multi-modal, this conditional average frequently lies in the empty space between valid modes, yielding erroneous predictions, as observed in our first experiment shape Fig.\ref{fig:synthetic_exp_sep}.

\section{The Identity Collapse Trap in Dynamic Parameterization (MDN)} \label{app:identity_collapse}

During our first experiments, we noted that, when extending the Gaussian Mixture Joint Embeddings (GMJE) framework to utilize dynamic, instance-specific parameters via a neural network (as in the GMJE-MDN architecture), one might intuitively attempt to parameterize the joint distribution directly from the joint embedding itself, i.e. dropping all info $Z_k = [z_c^T, z_t^T]^T_k$ of instance $k$ into the neural network to obtain an amortized inference of $[\mu,\Sigma,\pi]_k$ which can be used to parameterize the subsequent GMM. In this appendix, we mathematically demonstrate why feeding the full joint vector $Z = [z_c^T, z_t^T]^T$ into the parameter network leads to a pathological failure state known as \textit{Identity Collapse} (or \textit{Information Leakage}), and why introducing a strict \textit{conditional information bottleneck} is theoretically required.

\subsection{The Naive Joint Parameterization}
Suppose we construct a parameter network $f_\phi$ that takes the entire joint embedding $Z$ as input to predict the mixture parameters for that exact same joint space: $f_\phi(Z) \to \{\pi_k(Z), \mu_k(Z), \Sigma_k(Z)\}_{k=1}^K$. 
The network is optimized to minimize the Negative Log-Likelihood (NLL) of the joint mixture:
\begin{equation*}
    \mathcal{L} = -\log \sum_{k=1}^K \pi_k(Z) \mathcal{N}\Big(Z \;\Big|\; \mu_k(Z), \Sigma_k(Z)\Big)
\end{equation*}

\subsection{The Pathological Optimum (Identity Collapse)}
Because the neural network takes $Z$ as its input, it has full deterministic access to the exact target variable it is attempting to evaluate in the Gaussian density. To maximize the Gaussian likelihood (and thus minimize the NLL loss), the network discovers a trivial "cheat code": it simply learns the \textit{Identity Function}.

By setting the predicted mean of every component exactly equal to the input, $\mu_k(Z) = Z$ for all $k \in \{1 \dots K\}$, the Mahalanobis distance term inside the Gaussian exponent evaluates to exactly zero:
\begin{equation*}
    -\frac{1}{2} (Z - \mu_k(Z))^T \Sigma_k(Z)^{-1} (Z - \mu_k(Z)) = 0
\end{equation*}
Consequently, the exponential term reaches its absolute maximum, $\exp(0) = 1$. The network can then trivially drive the loss toward $-\infty$ by shrinking the determinant of the predicted covariance matrices, $|\Sigma_k(Z)| \to 0$, causing the probability density to approach infinity.

\subsection{Empirical Symptoms}
When this identity collapse occurs, the network fails to learn the underlying semantic branches or multi-modal structure of the data manifold. Instead, it merely memorizes the exact noise profile of the specific input data points. 

Further, because every component $k$ perfectly cheats by outputting exactly $\mu_k(Z) = Z$, the network has no mathematical incentive to prefer one component over another. As a result, the dynamic mixing weights degenerate into a flat, uniform distribution ($\pi_k(Z) \approx 1/K$ everywhere). The network completely loses its ability to perform intelligent, semantic density routing based on ambiguity.

\subsection{The Resolution: The Conditional Information Bottleneck}
To successfully achieve dynamic multi-modal routing, we must enforce a strict \textit{Information Bottleneck}. As derived in the GMJE-MDN formulation (Section.\ref{subsec:GMJE_MDN}), the parameter network must operate conditionally: it is only permitted to observe the context view $z_c$, and must predict the mixture parameters for the target view $z_t$:
\begin{equation}
    \mathcal{L} = -\log \sum_{k=1}^K \gamma_k(z_c) \mathcal{N}\Big(z_t \;\Big|\; \mu_{t|c,k}(z_c), \Sigma_{t|c,k}(z_c)\Big)
\end{equation}
Under this strict conditional formulation, the network does not have access to $z_t$. When presented with an ambiguous context embedding $z_c$ (e.g. the intersection point of multiple valid branches, as in Fig.\ref{fig:synthetic_exp_int}), it cannot rely on the identity function. To survive the loss penalty, it is mathematically forced to distribute its $K$ components to cover the distinct, valid possibilities of $z_t$ that could arise from that specific $z_c$. This physically forces the network to learn the true underlying manifold and dynamically route its probability mass ($\gamma_k$) to the appropriate branches, ensuring a robust and mathematically sound generative alignment space.

\section{Generalizing SMC to the Full GMJE Objective} \label{app:general_smc_gmje}

In Section \ref{sec:SMC_Memory_Bank}, we introduced SMC-GMJE (Algo.\ref{alg:smc_gmje}) as a direct upgrade to standard contrastive learning frameworks (e.g. MoCo). However, by substituting the InfoNCE loss with the exact parametric GMJE Negative Log-Likelihood (NLL) objective (Eq.\ref{eq:GMJE_prototypical_loss}), the Sequential Monte Carlo particle filter natively generalizes to manage full-covariance Gaussian components.

In this generalized formulation, the memory bank does not merely store target embeddings $z_t$; it stores the concatenated joint embeddings $Z = [z_c^T, z_t^T]^T$, which act as the $M$ component means ($\mu^{(m)}$) of a dynamically weighted Gaussian Mixture Model. Simultaneously, the neural network learns a shared, full-rank covariance matrix $\Sigma$ via gradient descent to capture the non-linear dependencies and multi-modal shape of the manifold.

\subsection{The Generalized Weight Update}
Instead of relying on the isotropic dot-product similarity ($\exp(z_c^T z_t^{(m)} / \tau)$), the particle likelihoods are evaluated using the true Mahalanobis distance governed by $\Sigma$. For a batch of $B$ incoming joint queries $Z_{query}$, the unnormalized importance weight for particle $m$ recursively updates as:
\begin{equation}
    \tilde{W}_t^{(m)} = W_{prior}^{(m)} \times \frac{1}{B} \sum_{Z \in Z_{query}} \exp \left( - \frac{1}{2} (Z - \mu^{(m)})^T \Sigma^{-1} (Z - \mu^{(m)}) \right)
\end{equation}
Note that the constant scaling factors and the log-determinant of $\Sigma$ from the standard Gaussian density function elegantly cancel out when normalizing the weights to sum to one, rendering the weight update mathematically clean and strictly dependent on the learned metric geometry of $\Sigma^{-1}$.

To understand why this update is computationally efficient, consider the exact probability density function of the multivariate Gaussian. The true unnormalized importance weight, utilizing the exact likelihood, is:
\begin{equation}
    \hat{W}_t^{(m)} = W_{prior}^{(m)} \times \frac{1}{B} \sum_{Z \in Z_{query}} \underbrace{\frac{1}{\sqrt{(2\pi)^D |\Sigma|}}}_{C(\Sigma)} \exp \left( - \frac{1}{2} (Z - \mu^{(m)})^T \Sigma^{-1} (Z - \mu^{(m)}) \right)
\end{equation}
where $D = d_c + d_t$ is the total dimensionality of the joint embedding, and $C(\Sigma)$ is the standard Gaussian normalization constant. 

Because the GMJE formulation utilizes a \textit{shared} global covariance matrix $\Sigma$ across all components, $C(\Sigma)$ is identical for every particle $m$. Consequently, it acts as a global scalar that can be factored out of the summation:
\begin{equation}
    \hat{W}_t^{(m)} = C(\Sigma) \times \underbrace{\left[ W_{prior}^{(m)} \times \frac{1}{B} \sum_{Z \in Z_{query}} \exp \left( - \frac{1}{2} (Z - \mu^{(m)})^T \Sigma^{-1} (Z - \mu^{(m)}) \right) \right]}_{\tilde{W}_t^{(m)}}
\end{equation}

When normalizing the posterior weights across the pool of particles to ensure they sum to one, this shared constant $C(\Sigma)$ elegantly factors out of the denominator and cancels with the numerator:
\begin{equation}
    W_{pool}^{(m)} = \frac{\hat{W}_t^{(m)}}{\sum_{j=1}^{M+B} \hat{W}_t^{(j)}} = \frac{C(\Sigma) \tilde{W}_t^{(m)}}{C(\Sigma) \sum_{j=1}^{M+B} \tilde{W}_t^{(j)}} = \frac{\tilde{W}_t^{(m)}}{\sum_{j=1}^{M+B} \tilde{W}_t^{(j)}}
\end{equation}
This cancellation mathematically proves that evaluating the computationally expensive log-determinant ($|\Sigma|$) and scaling factors is strictly unnecessary for the SMC routing step. The normalized posterior weights are determined entirely by the unnormalized exponent $\tilde{W}_t^{(m)}$, rendering the dynamic particle update mathematically clean and strictly dependent on the learned metric geometry of the precision matrix $\Sigma^{-1}$.

\subsection{Algorithm: General SMC-GMJE}
The complete procedure for the generalized full-covariance SMC-GMJE is detailed in Algo.\ref{alg:general_smc_gmje}. During the forward pass, the objective acts exactly as the relaxed EM optimization formulated in Section \ref{subsec:GMJE_reduced_EM}, but it utilizes the dynamic posterior particle weights $W^{(m)}$ as the mixing priors $\pi_k$.

\begin{algorithm}[H]
\caption{General SMC-GMJE: Full-Covariance Particle Filter}
\label{alg:general_smc_gmje}
\begin{algorithmic}[1]
\REQUIRE Stream of query batches $Z_{query} = \{Z_b\}_{b=1}^B$ generated by encoders (where $Z_b = [z_c^T, z_t^T]^T_b$), memory bank size $M$. Learnable shared covariance $\Sigma$.
\STATE Initialize memory bank $\mathcal{M}$ with $M$ random joint embeddings acting as means $\mu^{(m)}$, and uniform weights $W \leftarrow 1/M$.
\FOR{each training step $t$ with incoming batch $Z_{query}$}
    
    \STATE \textbf{// Loop A: Parametric Optimization (Requires Gradients)}
    \STATE \textbf{Forward Pass:} Compute the Log-Sum-Exp objective (Eq.\ref{eq:GMJE_prototypical_loss}) for the batch $Z_{query}$ against the memory bank means $\mu^{(m)}$, utilizing the current particle weights $W^{(m)}$ as the fixed mixing priors $\pi_m$, and the learnable covariance $\Sigma$.
    \STATE \textbf{Backward Pass:} Compute gradients and update the dual encoder networks' parameters \textit{and} the shared covariance matrix $\Sigma$ via gradient descent.
    
    \STATE
    \STATE \textbf{// Loop B: General SMC Optimization (Gradient-Free)}
    \STATE \textbf{Combine:} Pool $\mathcal{M}_{pool} \leftarrow \mathcal{M} \cup Z_{query}$, resulting in size $M+B$.
    \STATE \textbf{Prior Weights:} Set $W_{prior}^{(m)} \leftarrow \frac{1}{M+B}$ for the new batch components, and scale old memory bank weights by $\frac{M}{M+B}$.
    
    \STATE \textbf{Importance Update:} For each particle $m \in \{1 \dots M+B\}$:
    \STATE \quad $\tilde{W}_t^{(m)} \leftarrow W_{prior}^{(m)} \times \frac{1}{B} \sum_{Z \in Z_{query}} \exp \left( - \frac{1}{2} (Z - \mu^{(m)})^T \Sigma^{-1} (Z - \mu^{(m)}) \right)$
    \STATE \textbf{Normalize:} $W_{pool}^{(m)} \leftarrow \tilde{W}_t^{(m)} / \sum_{j=1}^{M+B} \tilde{W}_t^{(j)}$
    
    \STATE \textit{// Execute Sequential Importance Resampling (SIR)}
    \STATE Sample $M$ particles from $\mathcal{M}_{pool}$ with replacement, proportional to their new posterior weights $W_{pool}^{(m)}$.
    \STATE Replace the memory bank $\mathcal{M}$ with these sampled particles.
    \STATE Reset all memory bank weights to uniform: $W \leftarrow 1/M$.
\ENDFOR
\end{algorithmic}
\end{algorithm}

\section{The ``Mahalanobis Trace'' Trap in Empirical Batch Optimization} \label{app:mahalanobis_trap}

The ``Mahalanobis Trace'' Trap is a mathematically elegant but dangerous linear algebra quirk that frequently catches researchers off guard when implementing probabilistic loss functions or geometric regularizers in deep learning frameworks. In our context, this trap is not limited to unimodal distributions (e.g. Primal-GJE); it infects multi-modal Gaussian Mixtures (GMJE) with equal severity. To understand this trap, one must carefully distinguish between theoretical population mathematics and empirical mini-batch operations.

In generative theory, the Mahalanobis distance $z^T \Sigma^{-1} z$ measures how far a point $z$ is from the center of a distribution, scaled by a \textit{fixed population covariance} $\Sigma$. If the representation $z$ moves during optimization, its distance relative to the fixed $\Sigma$ changes, providing a valid non-zero gradient to the neural network encoders. 

However, in empirical self-supervised learning, one typically does not have access to a fixed $\Sigma$. Instead, one computes an empirical covariance matrix $C$ dynamically from the current mini-batch $Z$, and then immediately evaluates the exact same mini-batch $Z$ against $C^{-1}$. The following mathematical proofs demonstrate why doing this collapses the entire calculation into a static constant for both unimodal and multi-modal settings, catastrophically killing the data-fit gradients.

\subsection{Case 1: The Trap in Unimodal Gaussian (Primal-GJE)}
Let $Z \in \mathbb{R}^{N \times d}$ be a batch of $N$ centered embedding vectors, where each row $z_i \in \mathbb{R}^d$ is a single image's latent representation. The empirical batch covariance matrix $C$ is defined as $C = \frac{1}{N} \sum_{i=1}^N z_i z_i^T$.

The average Mahalanobis distance across the batch, which fundamentally provides both the \textit{expansive barrier} (preventing volume compression) and the \textit{cross-view attractive force}, is formulated as:
\begin{equation} \label{eq:trap_loss}
    \mathcal{L}_{\text{Mahalanobis}} = \frac{1}{N} \sum_{i=1}^N z_i^T C^{-1} z_i
\end{equation}
Because the quadratic form $z_i^T C^{-1} z_i$ results in a single scalar value, wrapping it in a trace operator leaves its value unchanged. We apply the trace cyclic property, $\text{Tr}(ABC) = \text{Tr}(CAB)$, and push the linear summation inside the operator:
\begin{align}
    \mathcal{L}_{\text{Mahalanobis}} &= \frac{1}{N} \sum_{i=1}^N \text{Tr}(C^{-1} z_i z_i^T) \nonumber \\
    &= \text{Tr}\left( C^{-1} \left[ \frac{1}{N} \sum_{i=1}^N z_i z_i^T \right] \right)
\end{align}
The bracketed term is exactly the mathematical definition of the empirical covariance matrix $C$. By directly substituting $C$ into the equation, the variables gracefully cancel to yield the Identity matrix ($I_d$), the trace of which is strictly the constant intrinsic dimensionality $d$:
\begin{equation}
    \mathcal{L}_{\text{Mahalanobis}} = \text{Tr}(C^{-1} C) = \text{Tr}(I_d) = d
\end{equation}

\subsection{Case 2: The Trap in GMM (GMJE)}
One might assume that partitioning the space into $K$ discrete components via a Gaussian Mixture Model resolves this issue. However, standard Expectation-Maximization (EM) batch updates trigger the exact same cancellation.

For a specific mixture component $k$, let $\gamma_{ik}$ be the posterior responsibility (soft assignment) of data point $z_i$ to cluster $k$. The effective number of points in the cluster is $N_k = \sum_{i=1}^N \gamma_{ik}$, and the empirical component covariance is evaluated as:
\begin{equation}
    \Sigma_k = \frac{1}{N_k} \sum_{i=1}^N \gamma_{ik} (z_i - \mu_k)(z_i - \mu_k)^T
\end{equation}
When evaluating the Mahalanobis data-fit term for this specific cluster using the same batch, weighted by the responsibilities, we apply the exact same trace trick and cyclic permutation:
\begin{align}
    \mathcal{L}_{\text{data-fit}, k} &= \sum_{i=1}^N \gamma_{ik} (z_i - \mu_k)^T \Sigma_k^{-1} (z_i - \mu_k) \nonumber \\
    &= \text{Tr} \left( \Sigma_k^{-1} \left[ \sum_{i=1}^N \gamma_{ik} (z_i - \mu_k)(z_i - \mu_k)^T \right] \right)
\end{align}
Observe the bracketed term. It is exactly equal to $N_k \Sigma_k$. Substituting this yields:
\begin{equation}
    \mathcal{L}_{\text{data-fit}, k} = \text{Tr} \left( \Sigma_k^{-1} [N_k \Sigma_k] \right) = N_k \text{Tr}(I_d) = N_k \cdot d
\end{equation}
Thus, the data-fit gradient for every single cluster $k$ algebraically collapses to a constant. The network feels absolutely zero gradient pulling the representations toward the cluster center $\mu_k$.

\vspace{1em}

To intuitively grasp why evaluating this term as a constant is fatal, one must first examine the geometric dynamics of the factorized joint distribution objective. By the laws of probability, the joint negative log-likelihood factors elegantly into two opposing forces: $-\log p(z_c, z_t) = -\log p(z_t \mid z_c) - \log p(z_c)$. The conditional term, $-\log p(z_t \mid z_c)$, serves as the \textit{Attractive (Pulling) Force}. It acts as a mathematical rubber band connecting two augmented views, pulling the target embedding $z_t$ toward the specific semantic coordinate or cluster generated by the context $z_c$. Conversely, the marginal term, $-\log p(z_c)$, governs the \textit{Repulsive (Pushing) Force}. When structurally inverted to explicitly maximize the differential entropy (via the log-determinant, which measures the geometric volume), it acts as a universal repulsive force that pushes all independent samples in the batch apart.

When these forces fall out of balance, the latent space suffers from distinct geometric failure modes. If the repulsive force is entirely absent, the attractive force wins completely, resulting in \textit{Complete Collapse} (or \textit{Instance Collapse}), where the network lazily maps every single image in the dataset to the exact same coordinate, reducing the variance to zero. More insidiously, if the geometric repulsion is misaligned, weak, or improperly evaluated, the network suffers from \textit{Dimensional Collapse} (or \textit{Feature Collapse}). In this state, the embeddings spread out enough to distinguish instances, but they collapse onto a strictly lower-dimensional subspace (e.g. squashing a 128-dimensional hypersphere onto a 1D line or 2D plane). Mathematically, the feature covariance matrix $\Sigma$ becomes highly ill-conditioned, dominated by zero eigenvalues. Geometrically, the network wastes its representational capacity; by failing to utilize orthogonal dimensions, independent semantic concepts (e.g. "color" and "shape") become highly entangled into a single metric, destroying the expressiveness of the joint embedding.

\subsection{The Dual Consequences of the Trace Trap}
In deep learning frameworks, one might intuitively implement the exact NLL penalty as a complex batched tensor operation (e.g. evaluating the sum of dot products between the embeddings and their inverse covariance). However, because $C$ was derived directly from $Z$, linear algebra perfectly cancels the moving variables, evaluating the entire data-fit term as the constant scalar $d$. Because $\nabla_Z (d) = \mathbf{0}$, \textbf{no gradients ever flow backward through this term}. 

Because the exact data-fit term mathematically governs both the expansive barrier and the cross-view attraction, its algebraic death triggers a simultaneous, dual geometric failure:

\paragraph{Consequence 1: Loss of Repulsion (Dimensional Collapse).} 
Without the data-fit term acting as an \textit{expansive barrier}, the complementary half of the exact joint loss function, i.e. the geometric penalty $\frac{1}{2} \log|C|$, is left completely unopposed. Driven solely to minimize this unconstrained log-determinant, the optimizer aggressively shrinks the volume of the space to zero, causing the catastrophic \textit{dimensional collapse} frequently observed in naive implementations of joint covariance regularization.

\paragraph{Remedy 1: Explicit Entropy Maximization.}
To rescue the optimization dynamics, the regularizer must be inverted. By discarding the dead Mahalanobis term entirely and flipping the sign of the log-determinant (minimizing $-\frac{1}{2} \log|C|$), we explicitly instruct the optimizer to \textit{maximize} the differential entropy. This actively expands the variance of the embeddings, creating a flawless repulsive force that covers the available spherical space. 

\paragraph{Consequence 2: Eradication of the Pulling Force.} 
While explicit entropy maximization successfully pushes representations apart to fill the hypersphere, \textit{it cannot restore view alignment}. The dead data-fit term containing the inverse cross-covariance block ($C_{ct}^{-1}$) is responsible for pulling $z_c$ and $z_t$ together. Without it, the network will scatter the embeddings across the hypersphere to maximize entropy, but matched pairs $z_c$ and $z_t$ will arbitrarily drift into completely random, disconnected coordinates.

\paragraph{Remedy 2: GMJE Architectural Interventions.}
GMJE survives this trap not because Gaussian Mixtures are magically immune to linear algebra, but because our specific architectural designs explicitly decouple the parameters from the empirical batch calculations. The three structural remedies introduced in this work act as the literal "escape hatches" from the GMM trap:
\begin{enumerate}
    \item \textit{Parametric Decoupling (Global Prototypes):} as utilized in Section \ref{subsec:GMJE_reduced_EM}, we do not calculate $\Sigma_k$ empirically from the batch. Instead, $\Sigma$ is initialized as a standalone learnable parameter (or treated as a fixed constant like temperature $\tau$). Because $\Sigma_{\text{param}}$ is updated via gradient descent over many epochs, it never equals the instantaneous batch covariance ($\Sigma_{\text{param}} \neq \Sigma_{\text{batch}}$). Consequently, the algebraic cancellation fails, and the network physically feels the gradients of the Mahalanobis distance, allowing the Log-Sum-Exp objective to pull $z_c$ and $z_t$ together.
    \item \textit{Conditional Factorization (MDN):} in Section \ref{subsec:GMJE_MDN}, the covariance $\Sigma_k(z_c)$ is predicted dynamically by a neural network based solely on the context view. It is a generated function of the input, not an empirical average of the batch. Therefore, the cancellation cannot occur, and the conditional formulation natively insulates the pulling force.
    \item \textit{Dynamic Non-Parametric Density (SMC):} in the non-parametric branch (Section \ref{sec:SMC_Memory_Bank}), the "components" are historical representations sitting in a memory bank, and the covariance is fixed ($\tau I$). The current batch evaluates its distance against the historical bank, not against itself. By isolating the target representations from the current batch gradient graph, the pulling force (the numerator of the InfoNCE objective) remains perfectly intact and highly informative.
\end{enumerate}

\subsection{Beyond the Trace Trap: The Log-Determinant Cheat}
Even though the architectural interventions of Parametric GMJE successfully bypass the Mahalanobis Trace Trap by decoupling the parameters, the optimization dynamics remain vulnerable to a secondary failure mode if the exact Negative Log-Likelihood (NLL) is used. 

Because the neural network controls both the geometric parameters and the representations ($Z$) simultaneously, optimizing the exact NLL exposes a trivial shortcut. The network realizes that mapping all representations $Z$ to the origin ($\mathbf{0}$) forces the empirical variance to zero, which in turn allows it to shrink $\Sigma \to 0$. As $\Sigma$ collapses, the log-determinant regularizer drives $\log |\Sigma| \to -\infty$, netting a massively negative loss. Consequently, optimizing the exact NLL directly induces catastrophic Dimensional Collapse.

To prevent this "cheat", one could apply a hard boundary constraint to the representations. For instance, VicReg \cite{bardes2021vicreg} utilizes a heuristic hinge loss ($\max(0, 1 - \sigma)$) to explicitly constrain the variance of each dimension to be at least $1$. However, in generative representation learning, such constraints are suboptimal because they create a "huddled" latent space: the network will shrink the representations until they hit the exact floor of the constraint, and then it will stop, doing the bare minimum.

This geometric reality fundamentally motivates \textbf{Explicit Entropy Maximization} (Remedy 1). By explicitly flipping the sign of the regularizer to minimize $-\frac{1}{2} \log |\Sigma_c|$, we replace a rigid floor constraint with a continuous, outward repulsive force. The network actively pushes the representations apart, forcing them to span the entire available volume of the hypersphere uniformly. This mathematical tension seamlessly guarantees a rich, full-rank generative embedding space without requiring manually tuned constraint hyperparameters.

\section{Experimental Details} \label{app:exp_details}

To support reproducibility, this appendix summarizes the shared computing environment, data processing steps, model architectures, and hyperparameters used in the three experimental groups reported in the paper: (i) synthetic ambiguous alignment, (ii) vision benchmarks on CIFAR-10, and (iii) unconditional latent sampling on MNIST.

\subsection{Shared Computing Environment}
All experiments were executed in Google Colab on a Linux-based virtualized instance (KVM) with the following hardware configuration:
\begin{itemize}[label=-]
    \item \textit{CPU:} Intel(R) Xeon(R) Platinum 8481C CPU at 2.70\,GHz, with 13 physical cores and 26 logical threads.
    \item \textit{Memory (RAM):} 247.01\,GB.
    \item \textit{Storage:} 253.06\,GB disk space.
    \item \textit{GPU:} a single NVIDIA H100 80GB HBM3 GPU (Driver Version 580.82.07, CUDA Version 13.0).
    \item \textit{Software stack:} Python 3.x, PyTorch \cite{pytorch2019}, Torchvision \cite{torchvision2010}, Scikit-Learn \cite{sklearn2011}, NumPy \cite{harris2020array}, and Matplotlib \cite{matplot2007}. The synthetic experiments additionally used Scikit-Learn's \textit{GaussianMixture} and \textit{GaussianProcessRegressor}.
\end{itemize}

\subsection{Experiment 1: Synthetic Ambiguous Alignment} \label{app:exp1_details}

\subsubsection{Random Seeds}
For both Dataset A and Dataset B, NumPy and PyTorch were seeded with \texttt{111}.

\subsubsection{Data Generation}
Both synthetic datasets used the same basic generation protocol:
\begin{itemize}[label=-]
    \item \textit{Sample size:} $N=3000$ training pairs $(x_c, x_t)$.
    \item \textit{Context distribution:} $x_c \sim \mathcal{U}(-1,1)$.
    \item \textit{Branch probabilities:} uniform over the three branches, i.e.\ $\pi_1=\pi_2=\pi_3=1/3$.
    \item \textit{Observation noise:} additive Gaussian noise $\epsilon \sim \mathcal{N}(0,0.05^2)$ applied to $x_t$ only.
    \item \textit{Evaluation grid:} 300 evenly spaced test points over $[-1,1]$.
\end{itemize}

The two target-generation rules were:
\begin{itemize}[label=-]
    \item \textbf{Dataset A (Separated Branches):}
    \[
    f(x_c)\in \left\{x_c^2+0.5,\; -x_c^2-0.5,\; x_c^3\right\}.
    \]
    \item \textbf{Dataset B (Intersecting Branches):}
    \[
    f(x_c)\in \left\{\sin(3x_c),\; -\sin(3x_c),\; 0\right\}.
    \]
\end{itemize}

\subsubsection{Models and Hyperparameters}

\paragraph{Classic JEPA (MSE).}
A 3-layer MLP was used:
\[
\texttt{Linear}(1,64)\rightarrow \texttt{ReLU}\rightarrow
\texttt{Linear}(64,64)\rightarrow \texttt{ReLU}\rightarrow
\texttt{Linear}(64,1).
\]
The model was optimized with Adam using learning rate $10^{-2}$ for 1200 full-batch epochs under MSE loss.

\paragraph{Dual-space kernel baseline (RBF GJE).}
The dual baseline was implemented with Scikit-Learn's \texttt{GaussianProcessRegressor} using
\[
\texttt{RBF(length\_scale=0.5)} + \texttt{WhiteKernel(noise\_level=0.1)}.
\]
The internal optimizer was disabled (\texttt{optimizer=None}) to keep the kernel fixed on the strongly multi-modal tasks.

\paragraph{GMJE-EM ($K=1$) and GMJE-EM ($K=3$).}
Both were implemented using Scikit-Learn's \texttt{GaussianMixture} with:
\begin{itemize}[label=$\circ$]
    \item \texttt{covariance\_type='full'},
    \item \texttt{init\_params='kmeans'} for the $K=3$ model,
    \item standard EM fitting on the 2D joint data $Z=[x_c, x_t]^T$.
\end{itemize}

\paragraph{GMJE-GNG.}
Growing Neural Gas was initialized with two nodes sampled from the training data and then updated online with:
\begin{itemize}[label=$\circ$]
    \item maximum nodes $K_{\max}=25$,
    \item maximum edge age $a_{\max}=50$,
    \item winner learning rate $\epsilon_b=0.2$,
    \item neighbor learning rate $\epsilon_n=0.01$,
    \item insertion interval $\lambda=100$,
    \item local error decay $\alpha=0.5$,
    \item global error decay $\beta=0.995$.
\end{itemize}

\paragraph{GMJE-MDN.}
The conditional MDN used $K=3$ mixture components. Its architecture consisted of a shared 2-layer MLP backbone,
\[
\texttt{Linear}(1,64)\rightarrow \texttt{ReLU}\rightarrow
\texttt{Linear}(64,64)\rightarrow \texttt{ReLU},
\]
followed by three linear heads for the mixture weights, conditional means, and conditional standard deviations. The weights were passed through a softmax, the means were unconstrained, and the standard deviations were parameterized as
\[
\sigma(x_c)=\exp(\cdot)+10^{-5}.
\]
Training used Adam with learning rate $10^{-2}$ for 1200 full-batch epochs, minimizing the exact conditional Gaussian-mixture negative log-likelihood.

\subsection{Experiment 2: Representation Learning on Vision Benchmarks (CIFAR-10)} \label{app:exp2_details}

\subsubsection{Random Seeds}
The CIFAR-10 script set NumPy and PyTorch seeds to \texttt{42}.

\subsubsection{Data Pipeline and Augmentations}
All CIFAR-10 experiments used the standard SSL augmentation pipeline:
\begin{itemize}[label=-]
    \item \texttt{RandomResizedCrop(32)},
    \item \texttt{RandomHorizontalFlip()},
    \item \texttt{ColorJitter(0.4, 0.4, 0.4, 0.1)} applied with probability $0.8$,
    \item \texttt{RandomGrayscale(p=0.2)},
    \item channel normalization with CIFAR-10 statistics.
\end{itemize}
For linear probing, the supervised augmentation pipeline used random crop with padding 4 and random horizontal flip.

\subsubsection{Shared Backbone}
All CIFAR-10 models used a modified ResNet-18 backbone adapted to $32\times 32$ images:
\begin{itemize}[label=-]
    \item the original $7\times 7$ stride-2 convolution was replaced by a $3\times 3$ stride-1 convolution with padding 1,
    \item the initial max-pooling layer was removed,
    \item global average pooling produced a 512-dimensional feature vector,
    \item a projection head mapped $512\rightarrow 512 \rightarrow 128$ via
    \[
    \texttt{Linear}(512,512)\rightarrow \texttt{BatchNorm1d}\rightarrow \texttt{ReLU}\rightarrow \texttt{Linear}(512,128).
    \]
\end{itemize}

\subsubsection{Experiment 2a: SMC Memory Banks vs.\ FIFO}
This long-horizon experiment compared MoCo v2 and SMC-GMJE under a severely constrained bank:
\begin{itemize}[label=-]
    \item \textit{Training duration:} 1000 epochs.
    \item \textit{Batch size:} 256.
    \item \textit{Memory-bank size:} $M=256$.
    \item \textit{EMA momentum:} $m=0.999$.
    \item \textit{Temperature:} $\tau=0.1$.
    \item \textit{Backbone and augmentations:} identical between MoCo v2 and SMC-GMJE.
\end{itemize}
The goal of this experiment was to isolate the effect of the bank update rule. SMC-GMJE replaced FIFO replacement with likelihood-based particle reweighting and multinomial resampling.

\subsubsection{Experiment 2b: 200-Epoch Benchmark Comparison}
The shorter benchmark comparison used a common 200-epoch pre-training budget, followed by 100 epochs of linear probing. Shared settings were:
\begin{itemize}[label=-]
    \item \textit{Batch size:} 256.
    \item \textit{Pre-training epochs:} 200.
    \item \textit{Linear probing epochs:} 100.
\end{itemize}

\paragraph{Baseline-specific pre-training hyperparameters.}
\begin{itemize}[label=$\circ$]
    \item \textbf{SimCLR:} SGD, learning rate $0.03$, momentum $0.9$, weight decay $10^{-4}$, temperature $\tau=0.1$.
    \item \textbf{MoCo v2:} queue size $M=4096$, EMA momentum $m=0.999$, temperature $\tau=0.1$, with the same SGD optimizer settings as above.
    \item \textbf{BYOL:} EMA momentum $m=0.99$. The predictor head was
    \[
    \texttt{Linear}(128,256)\rightarrow \texttt{BatchNorm1d}\rightarrow \texttt{ReLU}\rightarrow \texttt{Linear}(256,128).
    \]
    \item \textbf{SwAV:} $K=1000$ prototypes, temperature $\tau=0.1$, three Sinkhorn iterations.
    \item \textbf{VICReg:} variance, invariance, and covariance coefficients $(25.0, 25.0, 1.0)$. Unlike the contrastive-style models, VICReg operated on raw Euclidean features without $L^2$ normalization. The appendix draft indicates LARS with learning rate $0.3$ where available, or SGD with learning rate $0.1$ as fallback.
    \item \textbf{Parametric GMJE:} $K=1000$ global prototypes, temperature $\tau=0.1$, and an EMA-tracked covariance buffer with momentum $0.99$ for the explicit marginal regularizer.
    \item \textbf{GMJE-MDN:} $K=5$ dynamic components. The parameter network used
    \[
    \texttt{Linear}(128,256)\rightarrow \texttt{BatchNorm1d}\rightarrow \texttt{ReLU}\rightarrow \texttt{Linear}(256, K\times 128 + K + K).
    \]
    The covariance outputs were constrained with \texttt{Softplus} $+\,0.05$.
    \item \textbf{SMC-GMJE:} queue size $M=4096$, EMA momentum $m=0.999$, temperature $\tau=0.1$, with likelihood-based particle reweighting and multinomial resampling at each update.
\end{itemize}

\paragraph{Linear probing.}
After pre-training, the encoders were frozen and a linear classifier
\[
\texttt{Linear}(512,10)
\]
was trained for 100 epochs using SGD with learning rate $10.0$, momentum $0.9$, and a cosine annealing schedule.

\subsection{Experiment 3: Generative GMJE via Unconditional Latent Sampling on MNIST} \label{app:exp3_details}

\subsubsection{Random Seeds}
The combined vision/generative script explicitly sets global NumPy and PyTorch seeds to \texttt{42}. In addition, the repeated unconditional sampling phase resets both seeds to \texttt{1000+i} for the $i$-th generated sample grid, for $i=0,\dots,9$.

\subsubsection{Data Pipeline}
MNIST images were padded from $28\times 28$ to $32\times 32$. Two data pipelines were used:
\begin{itemize}[label=-]
    \item \textit{SSL manifold-learning pipeline:} \texttt{RandomResizedCrop} with scale range $[0.8,1.0]$, \texttt{RandomRotation} up to $15^\circ$, conversion to tensor, and normalization with mean $0.5$ and standard deviation $0.5$.
    \item \textit{Plain reconstruction pipeline:} deterministic padding, conversion to tensor, and the same normalization, used for decoder training and latent extraction.
\end{itemize}

\subsubsection{Encoder and Decoder Architectures}

\paragraph{CNN Encoder.}
All three encoder families (SimCLR, Unimodal GJE, Parametric GMJE) used the same CNN backbone:
\begin{itemize}[label=$\circ$]
    \item Conv2d$(1,32,3,\text{ stride }2,\text{ padding }1)$ + BatchNorm2d + ReLU,
    \item Conv2d$(32,64,3,\text{ stride }2,\text{ padding }1)$ + BatchNorm2d + ReLU,
    \item Conv2d$(64,128,3,\text{ stride }2,\text{ padding }1)$ + BatchNorm2d + ReLU,
    \item Conv2d$(128,256,3,\text{ stride }2,\text{ padding }1)$ + BatchNorm2d + ReLU,
    \item flatten,
    \item \texttt{Linear}$(256\times 2\times 2,512)$ + BatchNorm1d + ReLU,
    \item \texttt{Linear}$(512,128)$.
\end{itemize}

\paragraph{CNN Decoder.}
The decoder inverted the latent code back to image space:
\begin{itemize}[label=$\circ$]
    \item \texttt{Linear}$(128,512)$ + ReLU,
    \item \texttt{Linear}$(512,1024)$ + ReLU,
    \item reshape to $256\times 2\times 2$,
    \item four transposed convolutions:
    \[
    256\rightarrow 128\rightarrow 64\rightarrow 32\rightarrow 1
    \]
    each using kernel size $3$, stride $2$, padding $1$, output padding $1$,
    \item BatchNorm2d + ReLU after each intermediate layer,
    \item final \texttt{Tanh} activation.
\end{itemize}

\subsubsection{Training and Sampling Protocol}

\paragraph{Phase 1: encoder pre-training.}
SimCLR, Unimodal GJE, and Parametric GMJE ($K=50$) were each trained for 50 epochs with batch size 256 using Adam with learning rate $10^{-3}$.

\paragraph{Phase 2: decoder training.}
The encoders were frozen, and one decoder per encoder family was trained for 50 epochs using MSE reconstruction loss and Adam with learning rate $10^{-3}$. During decoder training, latent codes were $L^2$-normalized before decoding.

\paragraph{Phase 3: latent density extraction.}
\begin{itemize}[label=-]
    \item \textit{SimCLR:} a post-hoc full-covariance GMM with $K=50$ components was fit to 10{,}000 training embeddings using EM with a maximum of 20 iterations.
    \item \textit{Unimodal GJE:} the empirical mean and covariance were estimated from the full latent distribution.
    \item \textit{Parametric GMJE:} the mixture means were taken directly from the learned prototype matrix, normalized to the hypersphere, and the covariance was extracted from the EMA-tracked marginal-regularizer buffer. The component weights were set uniformly as $\pi_k = 1/50$ in the released script.
\end{itemize}

\paragraph{Phase 4: unconditional latent sampling.}
Synthetic embeddings were sampled directly from the corresponding latent densities, projected back onto the unit hypersphere by $L^2$ normalization, and decoded into image space. The script generated 10 repeated sample grids (Sample 2 was used in Fig.\ref{fig:generative_samples} for demonstration purpose), each using 16 synthetic latent points per method, as well as a separate t-SNE diagnostic using 500 synthetic embeddings and 2000 real test embeddings.

\paragraph{Code availability} All codes can be found at: \href{Github}{https://github.com/YongchaoHuang/GMJE}

\end{document}